\documentclass[11pt]{article}
\usepackage[top=1in, bottom=1in, left=1in, right=1in]{geometry}

\usepackage[utf8]{inputenc} %
\usepackage[T1]{fontenc}    %
\usepackage{hyperref}       %
\usepackage{url}            %
\usepackage{booktabs}       %
\usepackage{amsfonts}       %
\usepackage{nicefrac}       %
\usepackage{microtype}      %
\usepackage{xcolor}         %

\usepackage{savesym}
\savesymbol{Bbbk}

\usepackage{multirow}
\usepackage{algorithm}
\usepackage{algpseudocode}[1]
\usepackage{tcolorbox}
\usepackage{amsthm,amsmath,amssymb}
\usepackage{tikz}
\usepackage{mathrsfs}
\usepackage{cleveref}
\usepackage{mathtools,pifont,enumitem}
\usepackage[framemethod=TikZ]{mdframed}
\usepackage{fancybox,graphicx,subfigure}
\usepackage{array,dsfont}

\usepackage{sidecap}
\usepackage[symbol]{footmisc}

\newcommand{\yesnum}{\addtocounter{equation}{1}\tag{\theequation}}
\newcommand{\tagnum}[1]{\addtocounter{equation}{1}{\tag{#1)\ \ (\theequation}}}
\makeatletter
\newcommand{\customlabel}[2]{%
\protected@write \@auxout {}{\string \newlabel {#1}{{#2}{\thepage}{#2}{#1}{}} }%
\hypertarget{#1}{}
}
\makeatother

\mdfdefinestyle{FrameBox}{ linecolor=white, outerlinewidth=0pt, roundcorner=5pt,
innertopmargin=5pt, innerbottommargin=5pt,
innerrightmargin=10pt, innerleftmargin=10pt, backgroundcolor=white}

\mdfdefinestyle{FrameBox2}{ linecolor=black, outerlinewidth=0pt, roundcorner=5pt,
innertopmargin=5pt, innerbottommargin=5pt,
innerrightmargin=10pt, innerleftmargin=10pt, backgroundcolor=white}

\mdfdefinestyle{FrameQuestion2}{linecolor=white, outerlinewidth=0pt,
roundcorner=5pt,
innertopmargin=5pt, innerbottommargin=0pt,
innerrightmargin=16pt, innerleftmargin=16pt, backgroundcolor=white}

\newtheorem{assumption}{Assumption}
\newtheorem{theorem}{Theorem}[section]
\newtheorem{lemma}[theorem]{Lemma}
\newtheorem{definition}[theorem]{Definition}

\newtheorem{problem}[]{Problem}
\newtheorem*{theorem*}{Theorem}
\newtheorem{remark}[theorem]{Remark}

\crefname{equation}{Equation}{Equations}
\crefname{figure}{Figure}{Figures}
\crefname{table}{Table}{Tables}
\crefname{section}{Section}{Sections}
\crefname{appendix}{Section}{Sections}
\crefname{algorithm}{Algorithm}{Algorithms}
\crefname{assumption}{Assumption}{Assumptions}
\crefname{theorem}{Theorem}{Theorems}
\crefname{lemma}{Lemma}{Lemmas}
\crefname{definition}{Definition}{Definitions}
\crefname{conjecture}{Conjecture}{Conjectures}
\crefname{corollary}{Corollary}{Corollaries}
\crefname{construction}{Construction}{Constructions}
\crefname{claim}{Claim}{Claims}
\crefname{observation}{Observation}{Observations}
\crefname{proposition}{Proposition}{Propositions}
\crefname{fact}{Fact}{Facts}
\crefname{question}{Question}{Questions}
\crefname{problem}{Problem}{Problems}
\crefname{remark}{Remark}{Remarks}
\crefname{example}{Example}{Examples}

\crefname{appendix}{Section}{Sections}

\DeclareFontFamily{U}{stixscr}{}
\DeclareFontShape{U}{stixscr}{m}{n}{<-> s*[1.1] stix-mathscr}{}

\newcommand{\white}[1]{\textcolor{white}{#1}}

\colorlet{RED}{black}
\colorlet{BLUE}{black}

\newcommand{\N}{\mathbb{N}}

\newcommand{\R}{\mathbb{R}}

\newcommand{\cC}{\mathcal{C}}

\newcommand{\cS}{\mathcal{S}}

\newcommand{\cW}{\mathcal{W}}

\newcommand{\evE}{\ensuremath{\mathscr{E}}}
\newcommand{\evF}{\ensuremath{\mathscr{F}}}

\newcommand{\wt}{\widetilde}
\newcommand{\nfrac}{\nicefrac}
\newcommand{\sfrac}[2]{#1/#2}
\newcommand{\st}{\mathrm{s.t.}}

\newcommand{\eps}{\varepsilon}
\renewcommand{\epsilon}{\varepsilon}

\newcommand{\argmax}{\operatornamewithlimits{argmax}}
\newcommand{\Ex}{\operatornamewithlimits{\mathbb{E}}}

\newcommand{\poly}{\mathop{\mbox{\rm poly}}}

\def\abs#1{\left| #1 \right|}
\def\sabs#1{| #1 |}

\newcommand{\sinparen}[1]{(#1)}
\newcommand{\sinbrace}[1]{\{#1\}}

\newcommand{\inparen}[1]{\left(#1\right)}
\newcommand{\inbrace}[1]{\left\{#1\right\}}
\newcommand{\insquare}[1]{\left[#1\right]}

\newcommand{\zo}{\{0,1\}}
\newcommand{\np}{{\bf NP}}

\newcommand{\negsp}{\hspace{-0.5mm}}

\newcommand{\hH}{{\widehat{H}}}
\newcommand{\hF}{{\widehat{F}}}

\newcommand{\hS}{{\widehat{S}}}

\newcommand{\hT}{{\widehat{T}}}

\newcommand{\hw}{\widehat{w}}
\newcommand{\hW}{{\widehat{W}}}

\newcommand{\Stackrel}[2]{\stackrel{\mathmakebox[\widthof{\ensuremath{#2}}]{#1}}{#2}}

\newcommand{\folder}{../figures/}

\newif\ifconf

\newif\ifnofigure

\newcommand{\midsepremove}{\aboverulesep = 0mm \belowrulesep = 0mm}
\newcommand{\midsepadd}{\aboverulesep = 1mm \belowrulesep = 1mm}

\newcommand{\leftmarginCUSTOM}{18pt}
\ifconf
\renewcommand{\leftmarginCUSTOM}{18pt}
\else
\renewcommand{\leftmarginCUSTOM}{\leftmargin}
\fi

\ifconf

\else

\fi

\newcommand{\opt}{\ensuremath{\mathrm{\textsc{OPT}}}}
\newcommand{\sopt}{S_\opt{}}
\newcommand{\ovsopt}{\overline{\sopt}}
\newcommand{\scons}{S_j}

\title{\bf Maximizing Submodular Functions for\\ \bf  Recommendation in the Presence of Biases}

\author{Anay Mehrotra\\ Yale University\and Nisheeth K. Vishnoi \\ Yale University}

\begin{document}

\maketitle

\begin{abstract}
    Subset selection tasks, arise in recommendation systems and search engines and ask to select a subset of items that maximize the value for the user. The values of subsets often display diminishing returns, and hence, submodular functions have been used to model them. If the inputs defining the submodular function are known, then existing algorithms can be used.  In many applications, however, inputs have been observed to have social biases that reduce the utility of the output subset. Hence, interventions to improve the utility are desired. Prior works focus on maximizing linear functions---a special case of submodular functions---and show that fairness constraint-based interventions can not only ensure proportional representation but also achieve near-optimal utility in the presence of biases. We study the maximization of a family of submodular functions that capture functions arising in the aforementioned applications. Our first result is that, unlike linear functions, constraint-based interventions cannot guarantee any constant fraction of the optimal utility for this family of submodular functions.  Our second result is an algorithm for submodular maximization. The algorithm provably outputs subsets that have near-optimal utility for this family under mild assumptions and that proportionally represent items from each group. In empirical evaluation, with both synthetic and real-world data, we observe that this algorithm improves the utility of the output subset for this family of submodular functions over baselines.
\end{abstract}

\addtocontents{toc}{\protect\setcounter{tocdepth}{2}}

\renewcommand*{\thefootnote}{\fnsymbol{footnote}}
\renewcommand*{\thefootnote}{\arabic{footnote}}

\newpage
\tableofcontents
\newpage

\section{Introduction}\label{sec:intro} 
    Subset selection arises in many web-based applications including different types of content recommendation systems \cite{spotify_submod, Blogosphere2009Submod, amazon_submod, amazon_stream_submod} and search engines \cite{microsoft_diverse}. 
    Generally speaking, in all of these applications, given a set of $n$ items (e.g., posts, products, videos, or websites), the task is to select a subset $S$ of size $k$ that is the most valuable for the user. 
    Submodular functions are often used to capture the utility of {subsets of items} because of the diminishing returns property that arises in the above applications \cite{krause2014submodular,microsoft_diverse,spotify_diversity}.
    Diminishing returns arise because of the fact that an item's value to a user depends on the other items shown to the user \cite{Beyond_topicality_1982, manning2010introduction, microsoft_diverse, Blogosphere2009Submod, amazon_stream_submod}.
    For instance, in a recommendation system, where items belong to different categories (such as genres or product types), each additional item from the same category provides a diminishing value to the user \cite{Blogosphere2009Submod,amazon_submod,amazon_stream_submod}.
    Similarly, in web search, results belong to different categories (based on, e.g., relevance to technology, news, locations, etc.) and each additional result from the same category adds diminishing value \cite{microsoft_diverse}.
    In settings with multiple stakeholders (e.g., multiple users, content creators, and platform), adding item $i$ to a subset containing another item, relevant to the same stakeholder as $i$, has a lower increment in utility than adding $i$ to a subset not containing such an item \cite{spotify_diversity}.
    Submodular functions model such diminishing returns.
    Formally, a set function $F$ is said to be submodular if, for any two subsets $S\subseteq T$ and item $i\not\in T$, the increase in value on adding $i$ to $S$ is at least as large as the increase in value on adding $i$ to $T$, i.e., $$F(S\cup\inbrace{i})-F(S)\geq F(T\cup\inbrace{i})-F(T).$$
    \noindent In the above applications, given $n$ items and an upper-bound $1\leq k\leq n$ on the maximum number of selected items, at a high level, the goal is to solve the following maximization program for a suitable submodular function $F\colon 2^{[n]}\to \R$
    \begin{align*}
        \ifconf\textstyle\else\fi
        \max\nolimits_{S\subseteq [n]: \abs{S}\leq k}F(S).
        \yesnum\label{prog:const_max}
    \end{align*} 
    \paragraph{A family of submodular functions.}
      There is a vast literature on maximizing submodular functions  \cite{nemhauser1978analysis, fisher1978analysis, feige1998threshold, krause2014submodular, pmlr-v48-mirzasoleiman16, guptaConstrained2010}.
      This literature studies various types of submodular functions.
      We focus on a family of submodular functions that captures many functions used in recommendation and web search.
      In these applications, each item has $m$ {\em attributes} and an item {$i\in\inbrace{1,2,\dots,n}$} generates a value or {\em utility} $W_{ij} \geq 0$ for a user (or stakeholder) who is looking for items with the $j$-th attribute.
      For instance, on Amazon Music, \cite{amazon_submod} choose $W_{ij}$ to encode the utility of song $i$ for users interested in songs from genre $j$.
      They use the submodular function $F(S) = \sum\nolimits_{j=1}^m $ $\log\inparen{1+\sum\nolimits_{i\in S} W_{ij}}$ to capture the value of a set $S$ of song recommendations.
      Another example is \cite{spotify_submod} who, roughly speaking, set $W_{ij}$ to encode the utility of song $i$ for stakeholder $j$.
      They use $F(S)= \sum_{i\in S}W_{i1} + \sum_{j=2}^m \sqrt{\sum_{i\in S}W_{ij}}$ to capture the value of a playlist $S$ on Spotify.
    Generalizing these examples, we consider the following family of submodular functions:
    Given $m$ increasing concave functions $g_1,g_2,\dots,g_m\colon \R_{\geq  0}\negsp{}\to\negsp{}\R_{\geq 0}$, corresponding to each attribute, and utilities $W$, the submodular function $F$ is
    \begin{align*}
        \ifconf\textstyle\else\fi
        \forall S\subseteq[n],\quad
        \ifconf\textstyle\else\fi
        F(S)\coloneqq \sum\nolimits_{j=1}^m g_j\inparen{\sum\nolimits_{i\in S}W_{ij}}.
        \ifconf\textstyle\else\fi
        \yesnum\label{def:submod_family}
    \end{align*} 
    {Here, $W$ captures an item's utility in ``isolation'' and the concavity of $g_1,g_2,\dots,g_m$ ensures that the utilities of sets of items display diminishing returns.}\footnote{{To see this, imagine we have multiple items $i\in \inbrace{1,2,...}$ and the utility of each item $i$ on the first attribute is $W_{i1}=1$. 
    Let $g(x)$ be $\sqrt{x}$ for each $x$, which is increasing and concave. 
    Selecting the first item increases the utility of the selection $g(1)-g(0)=1$. 
    Further selecting the second item increases the utility by a smaller amount $g(2)-g(1)\approx 0.414$, and further yet, selecting the third item increases the utility by an even smaller amount $g(3)-g(2)\approx 0.318$.}}
    We denote the above family of functions by $\evF$.
    $\evF$ captures the function in \cite{amazon_submod} when $g_j(x)=\log(1+x)$ for all $x$ and $j$.
    It also captures the function considered by \cite{spotify_submod} when $g_1(x)=x$ and $g_2(x)=\dots=g_m(x)=\sqrt{x}$ for all $x$.
    Moreover, in 
    \cref{app:additional_examples}, we show that $\evF$ captures functions used by \cite{microsoft_diverse,amazon_stream_submod}.
    
    If the utilities $W$ are accurately known, then one can use standard algorithms to approximately solve Program~\eqref{prog:const_max} (see \cref{sec:related_work}).
    However, in the above applications, the utilities are often derived from users; either directly from users' feedback or indirectly from predictions of learning algorithms trained on user data.
    Hence, societal biases can creep into such {\em observed} utilities.
    Consequently, the subset maximizing the objective defined by these observed utilities can have a sub-optimal value with respect to the objective $F$ defined by the true or {\em latent} utilities.
    
    \smallskip
    \paragraph{Bias in inputs.}  There are various mechanisms through which biases can arise in observed utilities.
    If certain social groups are overrepresented in the data compared to their proportion in the user base, then utilities derived from this data will be skewed toward the opinions of these groups.
    For instance, the IMDB rating of the 2016 remake of Ghostbusters was ``sabotaged by a faction of fans who appeared to be upset by its all-female cast'' \cite{imdb_ghostbuster_ringer, imdb_ghostbuster_forbes}.
    Apart from such explicit biases, humans also have unconscious implicit biases \cite{wenneras2001nepotism, bertrand2004emily, lyness2006fit, greenwald2006implicit, Bohnet_2016}.
    Humans' implicit biases can manifest in data and, in turn, introduce skews in the utilities \cite{Bohnet_2016,ekstrand_gender_in_book_recommendation_2021}.
    Bias can also arise due to differences in user characteristics across socially-salient groups.
    For instance, \cite{keswani2021dialect} observe that SOTA text summarization algorithms output summaries that under-represent minority dialects of English.
    \cite{keswani2021dialect} postulate that this is due to ``structural differences'' across dialects (e.g., differences in lengths of sentences or of Tweets).
    When algorithms are used to estimate utilities, such algorithmic biases can introduce skews in estimated utilities.
    
    \smallskip
    \paragraph{A model of bias.}
    To capture such skews, we consider a model that extends the model in \cite{KleinbergR18}.
    In this model, items belong to one of $p$ disjoint groups $G_1, G_2,\dots, G_p$. 
    Each group $G_\ell$ has an \textit{unknown} and increasing bias function $\phi_\ell\colon \R_{\geq 0}\to \R_{\geq 0}$.
    The observed utilities of items $i$ in group $G_\ell$ are defined as follows
    \begin{equation*}
        \ifconf\textstyle\else\fi
        \text{$\forall\ 1\leq j\leq m$,}\quad 
        \hW_{ij} \coloneqq \phi_\ell\inparen{W_{ij}}.
        \ifconf\textstyle\else\fi
        \yesnum\label{eq:bias_model_intro}
    \end{equation*}
    In the above examples, the groups $G_1, G_2,\dots, G_p$ can correspond to the set of movies whose protagonists are male (respectively non-male) and the set of tweets written in Standard English (respectively African American English).
    To gain some intuition, consider the special case studied in \cite{KleinbergR18}: for each $\ell$ and $x$, $\phi_\ell(x)=\beta_\ell\cdot x$ for some parameter $0<\beta_\ell<1$.
    In this case, the observed utilities of items in group $G_\ell$ are $\beta_\ell$ times smaller than their latent utilities.
    
    {To see a concrete connection to the above examples, consider user-given ratings used for recommendation: suppose $x$ fraction of the users are explicitly biased and give female-led movies a score of 1 (lowest).
    This reduces the score of a female-led movie from its true value of, say $v$, to $v(1-x)+x$.
    One can model this skew using by setting $G_\ell$ to be the set of female-led movies and $\phi_\ell(z)=z\cdot(1-x)+x$ for each $z$.
    The above setup can also be used to model algorithmic biases.
    For instance, consider a (hypothetical) algorithm that scores Tweets proportional to the TF-IDF score of the text in the Tweet.
    Since the TF-IDF score is known to be lower for text with more ``common words'' \cite{manning2010introduction}, Tweets in dialects $D$ that use a higher fraction, say $x_D$, of common words, receive lower scores.
    This can be modeled with $\phi_\ell(z)=z\cdot h(x_D)$ for some decreasing function $h$, where group $G_\ell$ consists of all Tweets in dialect $D$.}
    
    Since a platform that does not observe the latent utilities $W$, it naturally outputs the subset $\hS$ that maximizes the function $\hF$ defined by observed utilities $\hW$:
    \begin{eqnarray*}
            \ifconf\textstyle\fi
            \hS \coloneqq \ifconf\textstyle\fi \argmax_{T\subseteq[n]\colon \abs{T}\leq k} \hF(T), 
            \quad\text{where,}\quad \text{for all}\   
            \ifconf\textstyle\fi
            T \subseteq[n],\quad  \hF(T)  \coloneqq \sum_{j=1}^m g_j\inparen{\sum\nolimits_{i\in S}\hW_{ij}}.
            \ifconf\textstyle\fi
    \end{eqnarray*} 
    Since $\hS$ optimizes a different objective than $F$, the latent utility of $S$, $F\sinparen{\hS}$, can be smaller than the optimal latent utility, $\opt{} \coloneqq \max\nolimits_{S\subseteq [n]: \abs{S}\leq k}F(S).$ 
    
    \begin{mdframed}[style=FrameBox2]
      \centering
      \textit{Given $\hW$, can we find a set $S$ such that $F(S)$ is close to $\opt{}$?}
    \end{mdframed}
     
  \paragraph{Related work.}
  Recent works \cite{KleinbergR18,celis2020interventions,mehrotra2022intersectional} have studied the above model in the special case where the objective $F$ is a linear function--a specific type of functions in $\evF$--and the bias functions, for each $\ell$, are $\phi_\ell(x)=\beta_\ell \cdot x$ for some parameters $0<\beta_1,\beta_2,\dots,\beta_p<1$.
  These works explore if requiring the output $S$ to satisfy fairness constraints can improve its latent utility $F(S)$.
  Various fairness constraints have been considered in practice \cite{VermaR18,fairmlbook}: Equal representation requires $S$ to have at most $\nfrac{k}{p}$ items from each group $G_\ell$ and proportional representation requires $S$ to have at most $k\cdot\nfrac{\abs{G_\ell}}{n}$ items from  each group $G_\ell$.
  Generalizations of these constraints have also been considered, given values $u_{\ell}$, generalizations require $\abs{S\cap G_\ell}\leq u_\ell \cdot k$ respectively for each $\ell$.
  There are many reasons to use fairness constraints, including, ethical and legal ones \cite{ranking_survey,criticalReviewFairRanking22,overviewFairRanking,fairmlbook}.
  \cite{KleinbergR18,celis2020interventions,mehrotra2022intersectional} demonstrate that another benefit of fairness constraints is that they can improve the latent utility of the output.
  Given $u=(u_1,u_2,\ldots, u_p)$, let $S_u$ be the subset maximizing the observed utility $\hF(S_u)$  subject to satisfying the constraint specified by $u$.
  In the special case, where $F$ is linear (e.g., $F(S)\coloneqq \sum\nolimits_{i\in S}W_{i1}$) and latent utilities $W$ are drawn i.i.d. from some distribution,
  \cite{celis2020interventions,mehrotra2022intersectional} show that
  if $u$ captures proportional representation, then $F(S_u)\geq (1-o_k(1))\cdot \opt{}$. 
  \textit{Thus, a natural question is if there is a $u$ such that $F(S_u)$ is close to $\opt{}$ for a submodular function $F$ in the family $\evF$.}
           
    \paragraph{Stochasticity in groups.}
    \cite{KleinbergR18,celis2020interventions,mehrotra2022intersectional} assume that entries of $W$ are drawn i.i.d. from some distribution.
    We weaken this assumption:
    we let $W$ be arbitrary and assume that the groups $G_1,G_2,\dots,G_p$ are generated stochastically--independent of $W$.
    This is done by uniformly sampling $\abs{G_1}$ distinct items and assigning them to $G_1$, then sampling $\abs{G_2}$ distinct items (from those remaining) and assigning them to $G_2$, and so on.
    This captures the belief that there are no systematic differences in the latent utilities of items in different groups.
    If, in addition, entries of $W$ are also sampled i.i.d., then this is equivalent to the model of
    \cite{KleinbergR18,celis2020interventions,mehrotra2022intersectional}. 
    Further discussion of the model and results of     \cite{KleinbergR18,celis2020interventions,mehrotra2022intersectional} appears in \cref{sec:related_work_app}.
    
    \paragraph{Our contributions.}
        {We begin by studying the effectiveness of fairness constraints to achieve a high latent utility.}
        Our first result shows that for any $\eps>0$ and upper bound parameters $u$, there is a submodular function $F$ in $\evF$, latent utilities $W$, and functions $\phi_1,\phi_2,\dots,\phi_p$, such that with high probability, the latent utility of $S_u$ is at most $\eps\cdot\opt{}$ (\cref{thm:main_negative_result}).
        Thus, no choice of $u$ can ensure that $F(S_u)$ is close to $\opt{}$.
        This result holds with the ``multiplicative'' bias functions $\phi_1,\phi_2,\dots,\phi_p$ studied in \cite{KleinbergR18,celis2020interventions,mehrotra2022intersectional}.
        Hence, it contrasts the result of \cite{KleinbergR18,celis2020interventions,mehrotra2022intersectional} {that for any linear $F$, there is always a $u$ specifying fairness constraints} such that $F(S_u)\geq (1-o_k(1))\cdot \opt{}$. 
        {This shows a difference in the effectiveness of fairness constraints with linear objective functions compared to submodular functions.}

        \smallskip

        On the positive side, we give an algorithm for submodular maximization in the presence of biases (\cref{alg:disj}).
        \cref{alg:disj} can be used with any submodular function $F$ in the family $\evF$.
        If each item $i$ has a non-zero utility for at most one attribute (\cref{asmp:disj}), then the algorithm provably outputs a subset with near-optimal latent utility (\cref{thm:disj}).
        \cref{asmp:disj} is natural in some settings: for instance, \cite{microsoft_diverse} used the assumption in the context of web search.
        {Moreover, the assumption is also satisfied by the construction in our first result.}
        Concretely, we show that under \cref{asmp:disj}, given observed utilities $\hW$, \cref{alg:disj} outputs a subset $S$ whose latent utility is at least $\inparen{1-O(\tau^{-1}m^2 k^{-\sfrac{1}{4}})}\cdot \opt$, where $\tau>0$ is the minimum value such that all {non-zero entries of $W$ are between $\tau$ and $\tau^{-1}$ (\cref{thm:disj}).}

        \smallskip
         
        \cref{alg:disj} differs from the standard greedy algorithms for constrained monotone submodular maximization \cite{nemhauser1978analysis, fisher1978analysis, LazierThanGreedy}:
        Roughly speaking, given constraints, greedy algorithms iteratively select the item with the highest marginal utility (or an approximation of marginal utility). 
        \cref{alg:disj}, on the other hand, first computes the ``right'' constraints for the given data and, then, performs submodular maximization subject to the computed constraints.
 
        \newcommand{\uncons}{\textsf{Uncons}} 

        \smallskip
        
        Empirically, we evaluate the performance of \cref{alg:disj} when \cref{asmp:disj} does not hold.
        We run simulations on the MovieLens 20M \cite{MovieLensDataset} and two synthetic datasets and compare against the baselines \uncons{} and \textsf{ProportionalRepr}, which output the subset maximizing the observed utility and the subset maximizing the observed utility subject to satisfying proportional representation respectively.
        We fix $\phi_1(x)=x$ and $\phi_\ell(x)=\beta\cdot x$ for each $\ell\neq 1$.
        In simulations with synthetic datasets, we observe that \cref{alg:disj} achieves a latent utility higher than $0.95$ times the latent utility achieved when $\beta=1$ (i.e., there is no bias), even for small values of $\beta$ ($\beta\leq 0.01$).
        Whereas, \uncons{} outputs subsets whose latent utility decreases with $\beta$ and for $\beta<0.1$ is up to 12\% smaller than $\opt{}$.
        On MovieLens 20M \cite{MovieLensDataset}, we observe that the predicted relevance scores in the data are disproportionately higher (by up to 3 times) for movies led by male actors compared to movies led by non-male actors in genres stereotypically associated with men.
        In contrast, user ratings for these sets of movies are within 6\% of each other in all genres.
        We use these (biased) relevance scores to recommend movies from sets of men-stereotypical genres and evaluate the performance of recommended movies with user ratings.
        \cref{alg:disj} outperforms \uncons{} and \textsf{ProportionalRepr} by 3\% or more on 14/31 genre sets and has a similar as (within 1\%) or better performance than \uncons{} and \textsf{ProportionalRepr} on more than 87\% sets of men-stereotypical genres.
        
\section{Other related work}\label{sec:related_work}

    \paragraph{Fairness in information retrieval and recommendation.} 
        Information retrieval and recommendation systems (such as personalized feed generators, news recommenders, and search engines) have a significant societal influence \cite{Noble2018}.
        They are one of the primary sources of information for individuals \cite{manning2010introduction, Noble2018, ekstrand2022fairness, overviewFairRanking}, who, in turn, impart significant trust to these systems--tending to agree with their outputs \cite{Seeing_believing_2003} and to follow their suggestions \cite{PersuasionofRecommenders2006}. 
        Without fairness considerations, the outputs of existing systems have been observed to encode various societal biases \cite{Olteanu00K19}--leading to underrepresentation of some social groups \cite{KayMM15}, the polarization of user opinions \cite{polarizationWSJ2020}, and denial of economic opportunities available to individuals \cite{hannak2017bias}.
        Consequently, a growing body of works design interventions to mitigate the adverse effects of biases \cite{manning2010introduction, ekstrand2022fairness, overviewFairRanking}.
        These works can be broadly divided into those mitigating adverse effects on the users \cite{yao2017beyond,kamishima2017considerations,abdollahpouri2020popularity}, those mitigating adverse effects on the items (denoting providers such as journalists in news recommendation, artists in song recommendation, and individuals on online hiring platforms) \cite{fairExposureAshudeep, ReducingDisparateExposureZehlike,causal2021yang, BalancedRankingYang2019,linkedin_ranking_paper, AmortizedFairnessBiega2018}, and those considering both \cite{mehrotra2017auditing,  MehrotraFairMarketplace2018}.
        Works in each of these categories take diverse approaches:
        from modifying the relevance estimation pipeline to satisfy fairness criteria \cite{ReducingDisparateExposureZehlike,causal2021yang,yao2017beyond}, to requiring the output to satisfy fairness constraints \cite{fairExposureAshudeep,BalancedRankingYang2019,linkedin_ranking_paper,AmortizedFairnessBiega2018}, to modifying the objective of system to capture fairness metrics \cite{MehrotraFairMarketplace2018,abdollahpouri2020popularity,kamishima2017considerations,mehrotra2017auditing}.
        We focus on harm for the items or providers due to biases in the input and examine the efficacy of fairness constraints to mitigate these harms when the output's utility is captured by a submodular function. 
        Unlike this work, most prior works assume that the input data is accurate and is absent of biases.

    \paragraph{Submodular maximization.}
    There is a vast literature on maximizing submodular functions subject to different types of constraints \cite{nemhauser1978analysis,fisher1978analysis,feige1998threshold,krause2014submodular,pmlr-v48-mirzasoleiman16,guptaConstrained2010}.
    Among these, cardinality constraints are of specific interest. 
    These are constraints of the form $\abs{S}\leq k$ for some fixed $k$.
    Given a cardinality constraint and an evaluation oracle for $F$, the standard greedy algorithm of \cite{nemhauser1978analysis}  selects a subset $S$ of size $k$ such that $F(S)\geq (1-e^{-1})\cdot \opt{}$ while making at most $nk$ evaluations of $F$ and doing at most $O(nk)$ additional arithmetic operations \cite{nemhauser1978analysis}, where $\opt{}=\max_{\abs{S}\leq k} F(S)$. 
    Several variants of this algorithm have also been designed.
    These variants extend the constant-factor approximation guarantee to other types of constraints (including upper bounds stated in \cref{sec:intro}), improve its running time, and design distributed variants of the algorithm \cite{lazyGreedy, LazierThanGreedy, GruiaMaximizingMatroid2011, mirzasoleiman2016distributed}.
    {Finally, motivated by context-specific fairness requirements, a number of recent works \cite{celis2018multiwinner,bandyapadhyay2021Covering,Asudeh2022coverage,boehmer2022matching,blanco2022Covering,Yanhao2022submodular} also study submodular maximization in the presence of constraints beyond the family of constraints introduced in \cref{sec:intro}.}
    However, unlike the present work, these works, assume that one can evaluate the true function $F$ which may not be possible in the presence of biases.

\section{Model} 
    Let there be $n$ items, indexed by the set of values $[n]\coloneqq \inbrace{1,2,\dots,n}$.
    A set function $F\colon 2^{[n]}\to \R$ is said to be submodular if for each pair of subsets $T\subseteq S\subseteq [n]$ and item $i\in [n]$,
    $F(S\cup\inbrace{i})-F(S)\geq F(T\cup\inbrace{i})-F(T)$.
    We consider the following family of submodular functions that are studied in the context of content recommendation \cite{spotify_submod,Blogosphere2009Submod,amazon_submod,amazon_stream_submod} and web search \cite{microsoft_diverse}. 
    \begin{definition}[\textbf{A family of submodular functions}]
        $\evF$ is the family of all submodular functions $F$ that are parameterized by a number $m$, increasing concave functions $g_1,g_2,\dots,g_m\colon \R_{\geq 0}\negsp{}\to\negsp{}\R_{\geq 0}$, and matrix $W\in \R_{\geq 0}^{n\times m}$ as follows: $F(S)\coloneqq \sum\nolimits_{j=1}^m g_j\inparen{\sum\nolimits_{i\in S}W_{ij}}.$
    \end{definition}
    \noindent Setting $g_j(x)=\log(1+x)$ for all $x$ and $j$ in the above definition, we get the submodular function tested by \cite{amazon_submod} on Amazon Music.
    With $g_1(x)=x$ and $g_j(x)=\sqrt{x}$ for all $x$ and $j\neq 1$, we get the submodular function used by \cite{spotify_submod} to measure the quality of song playlists.
    Further examples appear in \cref{app:additional_examples}.

    Given a suitable $F\in  \evF$ and a number $k$, the goal in our motivating applications is to solve the following program:
    $$\max\nolimits_{S\subseteq [n]: \abs{S}\leq k}F(S).$$ 
    If $W$ and, hence, $F$ is known, then one can hope to find a subset $S$ such that $F(S)$ is close to the optimal value, $\opt{}$, of this program.
    However, as discussed, in many contexts, the utilities observed by a platform $\hW$ can encode societal biases, and hence, be different from the true or {\em latent} utilities $W$.
    Here, we consider a model of bias that builds on \cite{KleinbergR18,celis2020interventions}.
    In this model, items belong to one of $p$ disjoint groups $G_1,G_2,\dots,G_p$.
    We weaken the assumption in \cite{KleinbergR18,celis2020interventions} (and related models in \cite{mehrotra2022intersectional})by allowing $W$ to be arbitrary and requiring $G_1,G_2,\dots,G_p$ to be generated stochastically.
    In particular, given sizes $\abs{G_1},\abs{G_2},\dots,\abs{G_p}$, $G_1$ is constructed by selecting $\abs{G_1}$ items uniformly without replacement, $G_2$ is constructed by selecting $\abs{G_2}$ items uniformly from those remaining, and so on.
    We define $\gamma>0$ to be a constant such that for each $1\leq \ell\leq p$
    \begin{align*}
         \ifconf\textstyle\else\fi\abs{G_\ell}\geq \gamma n.
    \end{align*}
    The model of bias is as follows.
    \begin{definition}[\textbf{Model of bias}]\label{def:bias_model}
        For each $\ell$, there is an unknown and increasing bias function $\phi_\ell\colon \R_{\geq 0}\to \R_{\geq 0}$ such that
        the observed utility of item $i\in G_\ell$ for the $j$-th attribute is 
        $\hW_{ij}\coloneqq \phi_\ell\inparen{W_{ij}}.$
    \end{definition} 
    \noindent \cref{sec:related_work_app} extends this to overlapping groups.
    The specific groups vary with application and could, for instance, be defined by socially salient attributes (e.g., gender, race, or age) associated with each item (see examples in \cref{sec:intro}).
    For any $F\in \evF$, parameterized by $W$, let $\hF$ be the corresponding function parameterized by $\hW$.
    In this work, we study the following problem.
    \begin{problem}\label{prob:1} 
        Given $m$ functions $g_1,g_2,\dots,g_m$ and an $n\times m$ matrix $\hW$, parameterizing a monotone submodular function $\hF$, without knowledge of the specific bias functions $\phi_1,\phi_2,\dots,\phi_p$, find a subset $S$ of size at most $k$ such that $F(S)\approx \opt$.
    \end{problem}  
    \noindent {Note that for linear functions, which also belong to the family $\evF$, if the bias functions are ``multiplicative'' then \cite{KleinbergR18,celis2020interventions} already show that the set $S_u$ that maximizes observed utility subject to satisfying the proportional representation constraints satisfies $F(S_u)\approx \opt$.
    Since $S_u$ can be efficiently found when $F$ is linear, this answers \cref{prob:1} for linear $F$ and multiplicative bias functions. 
    This work studies generalizations to other bias functions and submodular $F$.}

\section{Theoretical results}   \label{sec:thy_results} 

    \subsection{Fairness constraints do not guarantee high latent utility} 
        
        In this section, we consider a family of fairness constraints and show that no fairness constraints in this family can guarantee a constant fraction of the optimal latent utility.
        
        For each $1\leq \ell\leq p$, let $\gamma_\ell$ be the fraction of all items that are in $G_\ell$, i.e., $\gamma_\ell\coloneqq \frac{\abs{G_\ell}}{n}$.
        For any vectors $u$ and $v$, the constraint specified by $u$ and $v$ requires the output subset $S$ to satisfy $$1\leq \ell\leq p,\quad \abs{S\cap G_\ell}\leq (u_\ell + v_\ell \gamma_\ell) \cdot k.$$ %
        {This family of fairness constraints generalizes the family introduced in \cref{sec:intro} (which corresponds to the subset of the above family where $v=0$).
        In particular, the above family captures} the equal representation when $u_\ell=\frac{1}{p}$ and $v_\ell=0$ for each $\ell$ and proportional representation when $u_\ell=0$ and $v_\ell=1$ for each $\ell$.
        Given $(u,v)$, let $S_{uv}$ be the subset that maximizes the observed utility $\hF$ subject to satisfying the constraints specified by $u$ and $v$.
        {Our first result studies the utility of {$S_{uv}$ under multiplicative bias, as studied by \cite{KleinbergR18,celis2020interventions}.}}
        \begin{theorem}[\textbf{Fairness constraints do not guarantee any fraction of $\opt{}$}]\label{thm:main_negative_result}
            Define $\phi_1(x)=\beta_1\cdot x$ and $\phi_2(x)=\beta_2\cdot x$ for all $x$.
            For any $0<\eps<1$, $u=(u_1,u_2)$, and $v=(v_1,v_2)$
            there exists 
            \begin{itemize}[itemsep=0pt,leftmargin=10pt]
                \item a submodular function $F\in \evF$, %
                \item numbers $0\leq \gamma_1,\gamma_2\leq 1$ specifying group sizes,
                \item numbers $0<\beta_1,\beta_2\leq 1$ specifying $\phi_1$ and $\phi_2$ respectively, and
                \item family of $n\times m$ matrices $\cW$ parameterized by $n$,
            \end{itemize}
            such that,
            for any $k\geq \poly(\eps^{-1})$, $n\geq k\cdot \poly(\eps^{-1})$, and $W\in\cW(n)$,
            \begin{align*}
                \ifconf\textstyle\else\fi\Pr\insquare{F(S_{uv})\leq \eps\cdot  \opt{}} \geq 1-\eps,
            \end{align*}
            where the probability is over the randomness in $G_1$ and $G_2$.
        \end{theorem}
        
        \noindent Thus, for any fairness constraint in the above family and any $\eps>0$, there is a submodular function in the family $\evF$ and family of utilities, such that the subset maximizing the observed utility subject to satisfying the fairness constraint has a utility of at most $\eps$ times the optimal (with high probability).
        
        \cref{thm:main_negative_result} straightforwardly generalizes to $p>2$ groups by adding empty groups. 
        It also generalizes to $m>3$ attributes by fixing utilities so that $W_{ij}=0$ for each item $i$ and $j\in [m]\backslash[3]$.
        {Note that in the above result the bias functions, $\phi_1$ and $\phi_2$, are multiplicative with positive multiplicative constants $\beta_1,\beta_2>0$.
        This ensures that the order of utilities of items is preserved within groups (irrespective of how close $\beta_1$ and $\beta_2$ are to 0).
        Under this property, \cite{celis2020interventions} shows that the proportional representation constraint recovers near-optimal latent utility with a linear function $F$. 
        In contrast, \cref{thm:main_negative_result}  shows that this is not true when $F$ is sub-modular (even when the bias functions satisfy the same multiplicative property).
        In fact, \cref{thm:main_negative_result}   shows that if $F$ is submodular then the subset maximizing the observed utility subject to satisfying proportional representation (or any other fairness constraint in the above family for that matter) can have a latent utility significantly smaller than the optimal.}
        The proof of {\cref{thm:main_negative_result} appears in \cref{sec:proofof:thm:main_negative_result}.}
        
    \subsection{An algorithm that recovers near-optimal latent utility}
    
    In this section, we present our main algorithm, \cref{alg:disj}, and the theoretical guarantee of its performance (\cref{thm:disj}).
    \cref{alg:disj} can be used for any function in the family $\evF$.
    It outputs a subset that proportionally represents items from each group and has near-optimal latent utility for functions in $\evF$ that satisfy an {algorithmically verifiable ``disjointedness'' assumption (\cref{asmp:disj}).}

    \paragraph{Disjointedness assumption.}
    {For each $j$, let $C_j$ be the set of items $i$ which have positive utility for the $j$-th attribute, i.e.,}
    $$\ifconf\textstyle\else\fi C_j\coloneqq \inbrace{i\in [n]: W_{ij}>0}.$$
    Intuitively, $C_j$ is the set of items that are relevant to the ``$j$-th attribute.''
    More concretely, in many of our motivation contexts, attributes correspond to different categories of items such as genres, topics, or retail types \cite{amazon_submod,microsoft_diverse, KleinbergRaghu18}.
    Here, $C_j$ is the set of items in the $j$-th category.
    \begin{assumption}[\textbf{Disjoint categories}]\label{asmp:disj}
        The matrix $W\in \R^{n\times m}$ is such that $C_1,\dots,C_m$ are disjoint.
    \end{assumption} 
    \noindent \cref{asmp:disj} holds in any context where the relevant categories of items are disjoint.
    In the context, of web search, the above assumption is identified and studied by \cite{microsoft_diverse}.
    Deviations from \cref{asmp:disj} can deteriorate the performance of \cref{alg:disj}.
    In \cref{sec:empirical_results}, we evaluate the performance of \cref{alg:disj} on MovieLens 20M data \cite{MovieLensDataset} that does not satisfy \cref{asmp:disj} (\cref{fig:real_world_4}).\footnote{In the MovieLens 20M data, the sets $C_1,\dots,C_m$ denote the genders of movies and can be non-disjoint when a movie has more than one genre.}
    We also evaluate \cref{alg:disj} on synthetic data that violates \cref{asmp:disj} and is inspired by the deployment of submodular-maximization based algorithms by \cite{spotify_submod} (\cref{fig:syn_iid_disj}).

    \paragraph{Our algorithm.}
        {\cref{alg:disj} outputs a subset that satisfies the proportional representation constraint.
        However, crucially, its output may not maximize the observed utility subject to satisfying the proportional representation constraint.
        Instead, its output satisfies data-dependent fairness constraints (which, in particular, guarantee that the output proportionally represents each group $G_1,G_2, \dots,G_p$).
        In particular, the constraints of \cref{alg:disj} are specified by $m$ values $k_1,k_2,\dots,k_m$ such that $k_1+k_2+\dots+k_m=k$:
        for each $j$, the output subset has $k_j$ items from the $j$-th category $C_j$ that proportionally represent the protected groups $G_1,G_2,\dots,G_p$.
        Note that this, in particular, guarantees that for any $k_1,k_2,\dots,k_m$, the output subset has a proportional number of items from each group $G_1, G_2, \dots, G_p$.}

        {The algorithm is divided into two parts.
        The first part computes the ``right'' constraints from given data.
        This fixes the parameters $k_1,k_2,\dots,k_m$ which are used in the second part.
        The second part performs submodular maximization subject to the computed constraints:
        for each $j$, it selects a subset $S_j$ of size $k_j$ that maximizes the observed utility $g_j\inparen{\sum_{i\in S} W_{ij}}$ subject to selecting a proportional number of items from each group $G_1, G_2,\dots, G_p$.
        \cref{alg:disj} outputs the subset $S_1\cup S_2 \cup \dots \cup S_m$.}

    \paragraph{Main theoretical result.}
    {Our theoretical result on the performance of \cref{alg:disj} assumes that there is one group, say  $G_1$, such that the utilities of items in $G_1$ face no ``bias'', i.e., the bias function $\phi_1$ corresponding to $G_1$ is the identity function.
    When utilities capture the relative ``quality'' or ``relevance'' of items then this assumption says that the observed utilities of items in $G_1$ act as a ``reference'' with respect to which the bias functions for all other groups are defined. 
    For instance, consider the multiplicative bias functions where $\phi_\ell(x)=\beta_\ell\cdot x$ for each $x$ and $1\leq \ell\leq p$.
    The assumption that $\phi_1$ is the identity function corresponds to re-scaling the other bias functions by $\beta_1$, i.e., choosing $\phi_\ell(x)=\frac{\beta_\ell}{\beta_1}\cdot x$ for each $x$ and $\ell$.}

    The performance of \cref{alg:disj} depends on the range of the non-zero entries of $W$.
    {We define a parameter $\tau$ to capture this range.}
        \begin{definition}\label{asmp:bounded_utils}
            Let $\tau>0$ be the smallest constant such that
            for each item $i\in [n]$ and attribute $j\in[m]$,
            either $W_{ij}=0$ or $\tau<W_{ij} < 1/\tau$.
        \end{definition}
    \noindent The performance of \cref{alg:disj}  depends on $\tau$ due to a certain concentration inequality, involving sums of the form $\sum_{i\in S\cap G_\ell}W_{ij}$, which depends on $\tau$ (\cref{lem:invariance} in \cref{sec:proofof:thm:disj}). 
        \begin{theorem}[\textbf{Latent utility guarantee of \cref{alg:disj}}]\label{thm:disj}
            {Suppose $\phi_1$ is the identity function.} 
            There is an algorithm (\cref{alg:disj}) that, given observed utilities $\hW\in \R^{n\times m}$ and evaluation oracles for $g_1,g_2,\dots,g_m:\R_{\geq 0}\to \R_{\geq 0}$,
            outputs a set $S$ of size $k$ with the following property (under \cref{asmp:disj}):
            For any $\eps,\tau,\gamma>0$ and any {increasing functions $\phi_2,\phi_3,\dots,\phi_p:\R_{\geq 0}\to \R_{\geq 0}$}, there is a large enough $k_0$ (dependent on $\eps,\gamma,m$, and $\tau$) such that for any $k\geq k_0$, 
            with probability at least $1-\eps$
            \begin{align*}
                \ifconf\textstyle\else\fi F(S)\geq \opt{} \cdot \inparen{1-\eps}.
            \end{align*}
            \noindent  Where the probability is over the randomness in the choice of $G_1,\dots G_p$.
            The algorithm makes $O(nk)$ evaluations of each $g_1,g_2,\dots,g_m$ and does $O(nk\log{n})$ additional arithmetic operations.
        \end{theorem} 
        \noindent {If the latent utilities are known, then under the disjointedness assumption (\cref{asmp:disj}) one can efficiently find a set $S$ that with the optimal latent utility using standard submodular maximization algorithms \cite{microsoft_diverse}.
        The challenge is to output a high latent utility subset when the latent utilities are unknown, and where existing submodular-maximization algorithms (either with or without fairness constraints) can have a low latent utility (\cref{thm:main_negative_result}).}
        \cref{thm:disj} shows that, for  $k$ is ``large,'' \cref{alg:disj} outputs a set with {near-optimal latent utility \textit{without knowing} the latent utilities.}
         
        \cref{thm:disj} can be extended (with the same proof) to a generalization of the bias model where the bias parameter corresponding to $W_{ij}$ not only depends on the protected group(s) which $i$ is in but also on the index $j$:
        {For each $1\leq \ell\leq p$ and  $j$-th attribute there is an increasing function $\phi_{\ell j}\colon \R\to \R$ such that the observed utilities of an item $i\in G_t$ are
        \begin{align*}
            \ifconf\textstyle\else\fi
            (\hW_{i1},\hW_{i2},\dots,\hW_{im}) = (\phi_{t1}(W_{i1}), \phi_{t2}(W_{i2}),\dots, \phi_{tm}(W_{im})).
            \yesnum\label{eq:extension_of_model}
        \end{align*}
        \noindent This reduces to the model of bias in \cref{def:bias_model} when $\phi_{t\ell}=\phi_\ell$ for each $1\leq t\leq m$.}
        The parameter $k_0$ in \cref{thm:disj} depends on $\eps,$ $\gamma$, $m$, and $\tau$ as $\wt{O}\inparen{\eps^{-4}\tau^{-4}\gamma^{-10}m^8}$.
        For a general submodular function, it is \np-hard to output a set $S$ with $F(S)\geq (1-e^{-1})\cdot \opt$ \cite{feige1998threshold}.
        \cref{alg:disj} gives a stronger guarantee (for large $k$) because this hardness result does not hold under \cref{asmp:disj}:
        if {$\phi_1,\phi_2,\dots,\phi_p$} are known and \cref{asmp:disj} holds, then there is an efficient algorithm to find a subset $S$ with optimal latent utility \cite{microsoft_diverse}.
        Finally, we remark that in the statements of our results, we have tried to capture the dependence on each parameter as cleanly as possible and not tried to optimize the constants.
        The proof of \cref{thm:disj} appears in \cref{sec:proofof:thm:disj}.

        \renewcommand{\algorithmicrequire}{\textbf{Input:}}
        \renewcommand{\algorithmicensure}{\textbf{Output:}}

        \begin{algorithm}[t]
        \caption{}\label{alg:disj}
        \begin{algorithmic}[1]
            \Require
            A matrix $\hW\in \R^{n\times m}$, a number $k\in \N$, groups $G_1,G_2,\dots,G_p$, and sets $C_1,\dots,C_m$
            \Ensure A subset $S$ of size $k$
            \vspace{2mm}
            \item[] $\triangleright$ \textit{Part 1: Compute data-dependent constraints}
            \State Initialize $\wt{S} = \emptyset$ and define following function
            \begin{align*}
                    \ifconf\textstyle\else\fi
                    \forall T\subseteq[n],
                    \quad 
                    \wt{F}(T) \coloneqq 
                    \ifconf\textstyle\else\fi
                    \sum\nolimits_{j=1}^m g_j\inparen{\frac{n}{\abs{G_1}}\cdot\sum\nolimits_{i\in T\cap G_1}\hW_{ij}}.
                    \ifconf\textstyle\else\fi
                    \yesnum\label{eq:rescaled_utility}
                \end{align*}
            \For{$1\leq j\leq m$}
                \State Sort items $i$ in $C_j\cap G_1$ in decreasing order of $\wt{F}(i)-\wt{F}(\emptyset)$,
                add the first $\min\sinbrace{\sqrt{k}, \abs{C_j\cap G_1}}$ \white{.} \qquad \white{.} \quad items in $\wt{S}$
            \EndFor{}
            \While{$\sabs{\wt{S}} < k\cdot \frac{\abs{G_1}}{n}$}
                \State Set $\wt{S}\coloneqq \wt{S}\cup \inbrace{i}$ where $i$ it the item in $G_1$ that maximizes the marginal increase in observed \white{...}\white{.} utility: $\wt{F}(\wt{S}\cup \inbrace{i})-\wt{F}(\wt{S})$
            \EndWhile 
            \vspace{1mm}
            \State Let $k_j\coloneqq \sabs{\wt{S}\cap C_j}$ for each $1\leq j \leq m$
            \vspace{3mm}
            \item[] $\triangleright$ \textit{Part 2: For each $j$, select a set $S_j\subseteq C_j$ of size $k_j$ that maximizes the observed utility while satisfying proportional representation}
            \State Initialize $S_j\coloneqq \emptyset$ for each $1\leq j\leq m$ %
            \vspace{1mm}
            \For{$1\leq j\leq m$ and $1\leq t\leq k_j$}
                    \State Define $\cC$ as the set of items in $C_j$ that can be added to $S_j$ without violating proportional \white{...}\quad \white{..}\quad  representation constraints
                    \State Set $S_j\coloneqq S_j\cup \inbrace{i}$, where $i$ is the item in $\cC$  with the highest marginal utility:
                    $$g_j\inparen{W_{ij}+\sum\nolimits_{t\in S_j} W_{tj}} - g_j\inparen{\sum\nolimits_{t\in S_j} W_{tj}}.$$ 
            \EndFor
            \vspace{1mm}
            \State \Return $S=S_1\cup S_2\cup\dots\cup S_m$
        \end{algorithmic}
    \end{algorithm}

    \subsection{Proof overviews of theoretical results}

    \subsubsection{Technical challenges in extending approaches for linear $F$} 
        Consider the linear function $F(S)=\sum\nolimits_{i\in S}W_{i1}$ and the subset $S_u$ that maximizes the observed utility for $F$ subject to satisfying proportional representation constraints.
        \cite{celis2020interventions,mehrotra2022intersectional} show that the latent utility of $S_u$, $F(S_u)$, is at least $(1-o_k(1))\cdot \opt{}$.
        This result relies on two properties of any linear function $F$:
        \begin{enumerate}[itemsep=0pt,leftmargin=\leftmarginCUSTOM]
            \item For any subset $S$, $F(S) =
            \sum\nolimits_{\ell\in [p]}  F(S\cap G_\ell)$
            \item For any increasing functions $\phi_1,\phi_2,\dots,\phi_p$, $\ell\negsp{}\in\negsp{} [p]$, and $R\negsp{}\subseteq\negsp{}\negsp{} [n]$,
            $$\ifconf\textstyle\else\fi\argmax\nolimits_{S\subseteq R\cap G_\ell:\abs{S}\leq k} \hF(S) \Stackrel{}{=} \argmax\nolimits_{S\subseteq R\cap G_\ell: \abs{S}\leq k} F(S).$$
        \end{enumerate} 
    
        \noindent Suppose $\sopt{}$ satisfies proportional representation.  
        Using property 1, one can show that $ S_u =  \bigcup\nolimits_{\ell=1}^p \hT_\ell$ and $\sopt{} =  \bigcup\nolimits_{\ell=1}^p T_\ell$, where 
        \[
            \hT_\ell \ \coloneqq \argmax_{\substack{T\subseteq G_\ell, \abs{T}\leq \sfrac{k\abs{G_\ell}}{n}}}  \hF(T)
            \ \ \text{and}\ \ 
            T_\ell\ \coloneqq \argmax_{\substack{T\subseteq G_\ell, \abs{T}\leq \sfrac{k\abs{G_\ell}}{n}}}  F(T).
        \]
        \noindent Further, because of property 2 (with $R=[n]$), it follows that for any $\ell\in [p]$, $\hT_\ell=T_\ell$ and, hence, $S_u=\sopt{}$.
        This relies on the assumption that $\sopt{}$ satisfies proportional representation.
        This may not be always true, but one can show that, with high probability, $\sopt{}$ \textit{nearly}-satisfies proportional representation.
        Using this and accounting for approximations in the above argument one can show that $F(S_u)\approx F(\sopt{}).$
        However, unfortunately, straightforward examples show that neither of the above properties holds for submodular functions in $\evF$. 
         
    \subsubsection{{Proof overview of \cref{thm:disj}}}
        Under \cref{asmp:disj}, we prove the following variants of properties 1 and 2 stated above:
        \begin{itemize}[itemsep=-0pt,leftmargin=14pt]
            \item For any subset $S$, $F(S) =
            \sum\nolimits_{j\in [m]}\sum\nolimits_{\ell\in [p]}  F(S\cap G_\ell\cap C_j)$.
            \item For any increasing functions $\phi_1,\phi_2,\dots,\phi_p$, $\ell\in [p]$, and $j\in[m]$
            $$\ifconf\textstyle\else\fi \argmax\nolimits_{S\subseteq C_j\cap G_\ell:\abs{S}\leq k} \hF(S) \Stackrel{}{=} \argmax\nolimits_{S\subseteq C_j\cap G_\ell: \abs{S}\leq k} F(S).$$
        \end{itemize}
        
        \paragraph{Observation.}
        Let $\sopt{}$ be any subset with the optimal latent utility.
        Suppose we know $k_{j}\coloneqq \abs{\sopt{}\cap C_j}$ for all $1\leq j\leq m$.
        If for all $j$, $k_{j}\geq \sqrt{k}$, then using the above properties (and adapting the analysis of \cite{mehrotra2022intersectional}), we can show that, with high probability, $S\coloneqq \bigcup_{j=1}^m S_j(k_j)$ has latent utility at least $(1-o_k(1))\cdot F(\sopt{})$.
        Where, for each $1\leq j\leq m$, $S_j(k_j)\subseteq C_j$ is the subset that maximizes the observed utility subject to selecting at most $k_j$ items and satisfying proportional representation constraint.

        Part 1 of \cref{alg:disj} computes estimates, $\tilde{k}_1, \tilde{k}_2,\dots,\tilde{k}_j$, of $k_1,$ $k_2,\dots,k_j$.
        For this, it relies on the fact that $\phi_\ell$ is the identity function for some (known) $\ell$ and that the groups $G_1, G_2,\dots, G_p$ are constructed stochastically.
        The estimated values $\tilde{k}_1,\dots,\tilde{k}_j$ specify the constraints for Part 2.
        Part 2 of \cref{alg:disj} outputs the set $S\coloneqq \bigcup_{j=1}^m S_j(\tilde{k}_j)$.
        It is possible that for some $j$, $\abs{\sopt{}\cap C_j} < \sqrt{k}$ and, hence, the above argument does not apply.
        \cref{alg:disj} avoids this by overestimating $k_j$ to ensure that $\tilde{k}_j\geq \sqrt{k}$ for all $1\leq j\leq m$.
        At a high level, this guarantees that the output set $S$ contains any item in $\abs{\sopt{}\cap C_j}$ with a ``high-utility'' with high probability.

\section{{Empirical results}}\label{sec:empirical_results}
    We evaluate \cref{alg:disj}'s performance on both synthetic and real-world data.\footnote{The code for our simulations is available at \url{https://github.com/AnayMehrotra/Submodular-Maximization-in-the-Presence-of-Biases}}

    \newcommand{\normlu}{\textsc{\rm NLU}}
      \subsection{Baselines and setup}
        We compare \cref{alg:disj}'s performance against two baselines: \uncons{} and \textsf{ProportionalRepr}.
        \uncons{}, given $k$ and observed utilities $\hW$, runs the standard greedy algorithm of \cite{nemhauser1978analysis} to find a subset of size $k$ that approximately maximizes the observed utility (\cref{alg:model_greedy} in \cref{sec:standard_algs}).
        \textsf{ProportionalRepr}, given $k$ and observed utilities $\hW$, uses a variant of the greedy algorithm that satisfies proportional representation constraints (\cref{alg:greedy_with_ub} in \cref{sec:standard_algs}). 
        \textsf{ProportionalRepr} outputs the subset of size $k$ that approximately maximizes the observed utility subject to selecting a proportional number of items from each group.

        In all simulations, we generate latent and observed utilities $W$ and $\hW$ (as explained in subsequent sections) and run algorithms with the following inputs:
        \cref{alg:disj} are given $\hW$, $k$, and the protected groups and \uncons{} is given $\hW$ and $k$.

    \subsection{Simulations with synthetic datasets}
        In this simulation, we show that \cref{alg:disj} can achieve high latent utility even in some cases where \cref{asmp:disj} does not hold.
        We consider two synthetic datasets inspired by the recommendation algorithms used on Spotify \cite{spotify_submod} and tested on Amazon music \cite{amazon_submod}.
        Among these, the first dataset and the corresponding function do not satisfy \cref{asmp:disj}.
        
        \paragraph{Setup.}
        In both simulations, the task is to recommend a set of $k \coloneqq 50$ songs to the current user.
        We set $n \coloneqq 250$, $m \coloneqq 3$, and consider two disjoint groups.\footnote{{Our parameter choices draw from \cite{spotify_submod}'s context (recommendation on Spotify): Spotify's algorithmically-curated playlists have $50$ songs, hence, $k=50$. 
        $m$ is small and unspecified in \cite{spotify_submod}. 
        We set $m=3$ with synthetic data to keep the corresponding scenario simple. 
        With real-world data, we try all values of $m$ (for chosen genres). 
        \cite{spotify_submod} do not specify $n$. 
        We varied $n$ over $\inbrace{100,250,500}$, observed similar results, and chose $n=250$ arbitrarily.}}
        
        In the first simulation, we fix the objective as $$F(S)\coloneqq \sum\nolimits_{i\in S}W_{i1} + \lambda \sum\nolimits_{j=2}^3 \sqrt{\sum\nolimits_{i\in S}W_{ij}}$$ with $\lambda=\frac{1}{20}$.\footnote{This is the same as the objective function used by the recommendation algorithm, Mostra, on Spotify \cite{spotify_submod}. Except that Mostra considers more than $m=3$ attributes and instead $W_{i1}$ encoding the number of  song-plays, it encodes a ``relevance score'' predicted by a learning algorithm.}
        Here, $W_{i1}\geq 0$, $W_{i2}\in \zo$, and $W_{i3}\in \zo$ denote some measure of the song's popularity $i$, whether $i$ is from an ``emerging artist,'' and whether the song has not been heard by the current user respectively.
        Intuitively, among songs $i$ with similar popularity (i.e., similar $W_{i1}$), songs from emerging artists and those not heard by the current user have a higher marginal utility.
        We draw the entries of $W$ i.i.d. from natural distributions that can arise in these contexts (see \cref{sec:implementation_details} for details).
        We divide the items into two groups $G_1$ and $G_2$; and vary the size of $G_1$ among $\inbrace{\frac{n}{4}, \frac{n}{2},\frac{3n}{4}}$.
        Here, $G_2$ can denote songs from artists that users want to hear, but which are nevertheless under-recommended due to biases in the recommendation pipeline, e.g., as have been recently observed for regional music on Spotify in India \cite{spotify_ken_india}.
        Given $G_1,G_2$ and a parameter $\beta\in [0,1]$, we generate observed utilities $\hW$ as follows:
        \begin{align*}
            \ifconf\textstyle\else\fi \hW_{i1} = \begin{cases}
                \ifconf\textstyle\else\fi W_{i1}, &\text{if } i\in G_1\ifconf\textstyle\else\fi\\
                \beta \cdot W_{i1}\ifconf\textstyle\else\fi &\text{if } i\in G_2.\ifconf\textstyle\else\fi
            \end{cases}
            ,\qquad
            \hW_{i2}=W_{i2},\ifconf\textstyle\else\fi
            \quad\text{and}\quad
            \hW_{i3}=W_{i3}.\ifconf\textstyle\else\fi
        \end{align*}
        \noindent For the first attribute, this corresponds to using $\phi_1(x)=x$ and $\phi_2(x)=\beta\cdot x$ for all $x$.
        For the last two attributes, we do not apply bias because they encode values that the platform may know.
        This violates the model in \cref{def:bias_model} where the same bias function acts on all attributes, hence, is a hard case for \cref{alg:disj}.
        (It falls into the extension of this model discussed in \cref{eq:extension_of_model}.)

            In the second simulation, we fix the objective:\footnote{This is the same as the function used by \cite{amazon_submod} except that they allow $m\geq 3$.} $$F(S)\coloneqq \sum\nolimits_{j=1}^3 \log\inparen{1+\sum\nolimits_{i\in S}W_{ij}}.$$
            Here, attributes denote genre, and the sets $C_1,C_2,C_3$ are disjoint.
            For any item $i$ in genre $h(i)$, $W_{ih(i)}$ is number of times users played song $i$ and $W_{ij}=0$ for $j\neq h(i)$.
            At a high level, this promotes the content to be diverse across genres as between two items $i$ and $j$ with a similar number of user plays, $i$ has a higher marginal utility if $\sum\nolimits_{r\in S}W_{rh(i)} < \sum\nolimits_{r\in S}W_{rh(j)}$.
            Like the previous simulation, we draw entries of $W$ i.i.d. from natural distributions that can arise in these contexts.
            We divide items into two protected groups {$G_1$ and $G_2$; and vary the size of $G_1$ among $\inbrace{\frac{n}{4}, \frac{n}{2},\frac{3n}{4}}$.
            The complete implementation {details of these simulations appear in \cref{sec:implementation_details}.}
            
        \paragraph{Results and discussion.}
            We vary $\beta\in [0,1]$ and $\abs{G_1} \in \inbrace{\frac{n}{4}, \frac{n}{2},\frac{3n}{4}}$, and report the normalized latent utilities of different algorithms.
            \cref{fig:syn_data_combined} presents the results for $\frac{\abs{G_1}}{n}=0.5$.
            The results with $\abs{G_1}\in \inbrace{\frac{n}{4},\frac{3n}{4}}$ appear in \cref{fig:syn_iid_non_disj_full,fig:syn_iid_disj_full} in \cref{sec:additional_simulations}.
            Across all figures and both synthetic datasets, \cref{alg:disj} outputs subsets with \normlu{} $>0.95$ even for values of $\beta$ close to $0$.
            In contrast, when $\beta$ approaches 0 and $\abs{G_1}= \frac{n}{2}$, \uncons{} achieves \normlu{}~ $\leq 0.85$.
            {As $\beta$ approaches 1 (i.e., as the bias decreases), the utility achieved by \uncons{} increases and approaches \normlu{}~ $\approx 1$ (which is the {optimal value and matches the performance of \cref{alg:disj}.}}
            
            Thus, we observe that \cref{alg:disj} can achieve high latent utility. 
            This observation also holds on the first synthetic dataset, where \cref{asmp:disj} does not hold.
            Hence, \cref{alg:disj} can also achieve high latent utility in some cases where \cref{asmp:disj} is violated.

    \begin{figure}[t!]
        \centering
        \par
        \hspace{8mm}
        \renewcommand{\folder}{./figures/synthetic-data}
        \hspace{-7mm}\subfigure[$\phi_1(x)=x$, $\phi_2(x)=\beta\cdot x$, $\frac{\abs{G_1}}{n}=0.5$, and $\delta=2.0$\label{fig:syn_iid_non_disj}]{
            \begin{tikzpicture}
            \node (image) at (0,-0.17)
                {\includegraphics[width=0.5\linewidth, trim={0cm 0cm 1.5cm 2cm},clip]{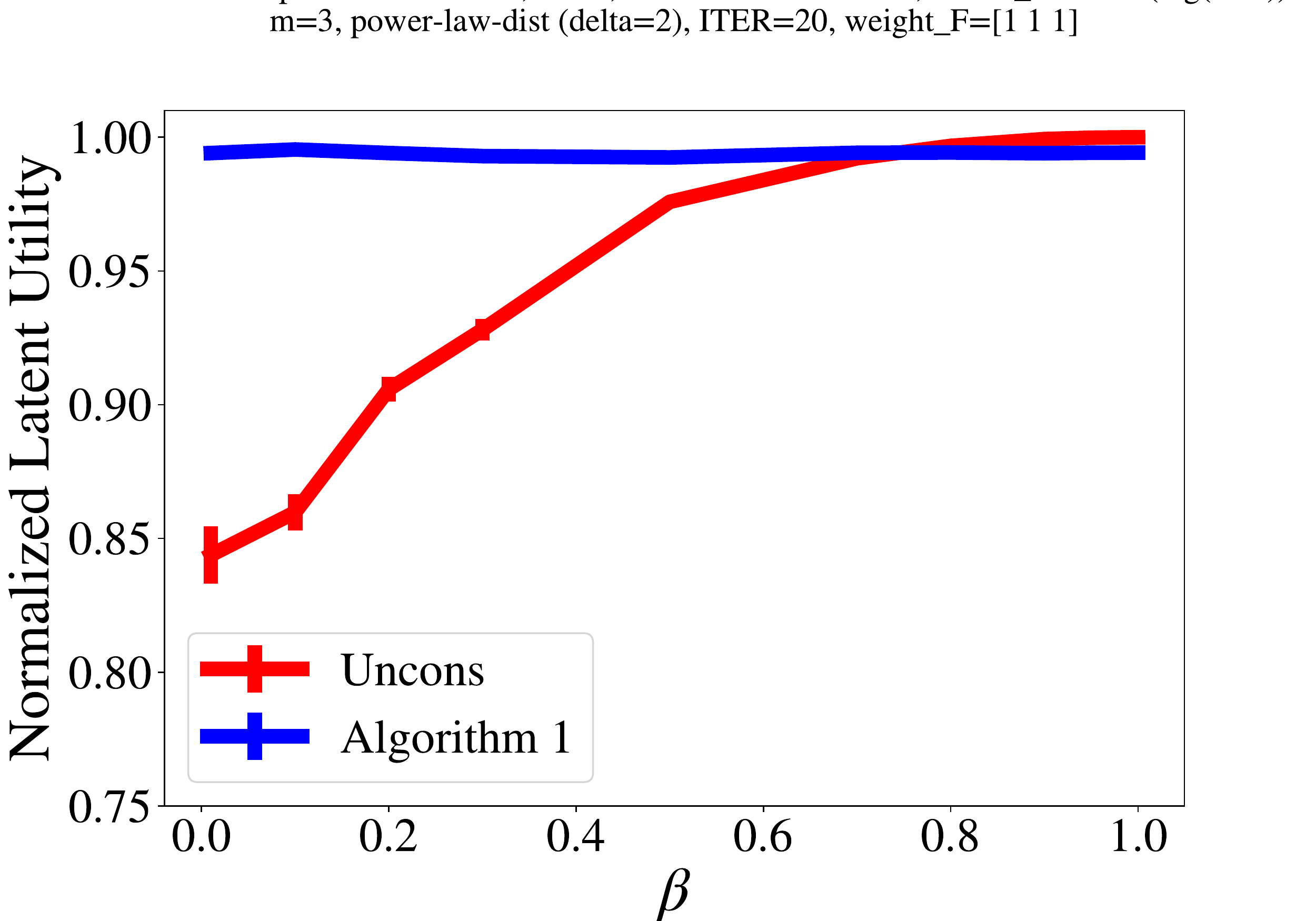}}; 
          \node[rotate=0] at (2.80,-2.75) {\textit{\small(less bias)}};
          \node[rotate=0] at (-2.14375, -2.75) {\textit{\small(more bias)}};
        \end{tikzpicture}
        }
        \hspace{8mm}
        \renewcommand{\folder}{./figures/synthetic-data}
        \hspace{-12mm}\subfigure[$\phi_1(x)=x$, $\phi_2(x)=\beta\cdot x$, $\frac{\abs{G_1}}{n}=0.5$, and $\delta=2.0$\label{fig:syn_iid_disj}]{
            \begin{tikzpicture}
            \node (image) at (0,-0.17)
                {\includegraphics[width=0.5\linewidth, trim={0cm 0cm 1.5cm 2cm},clip]{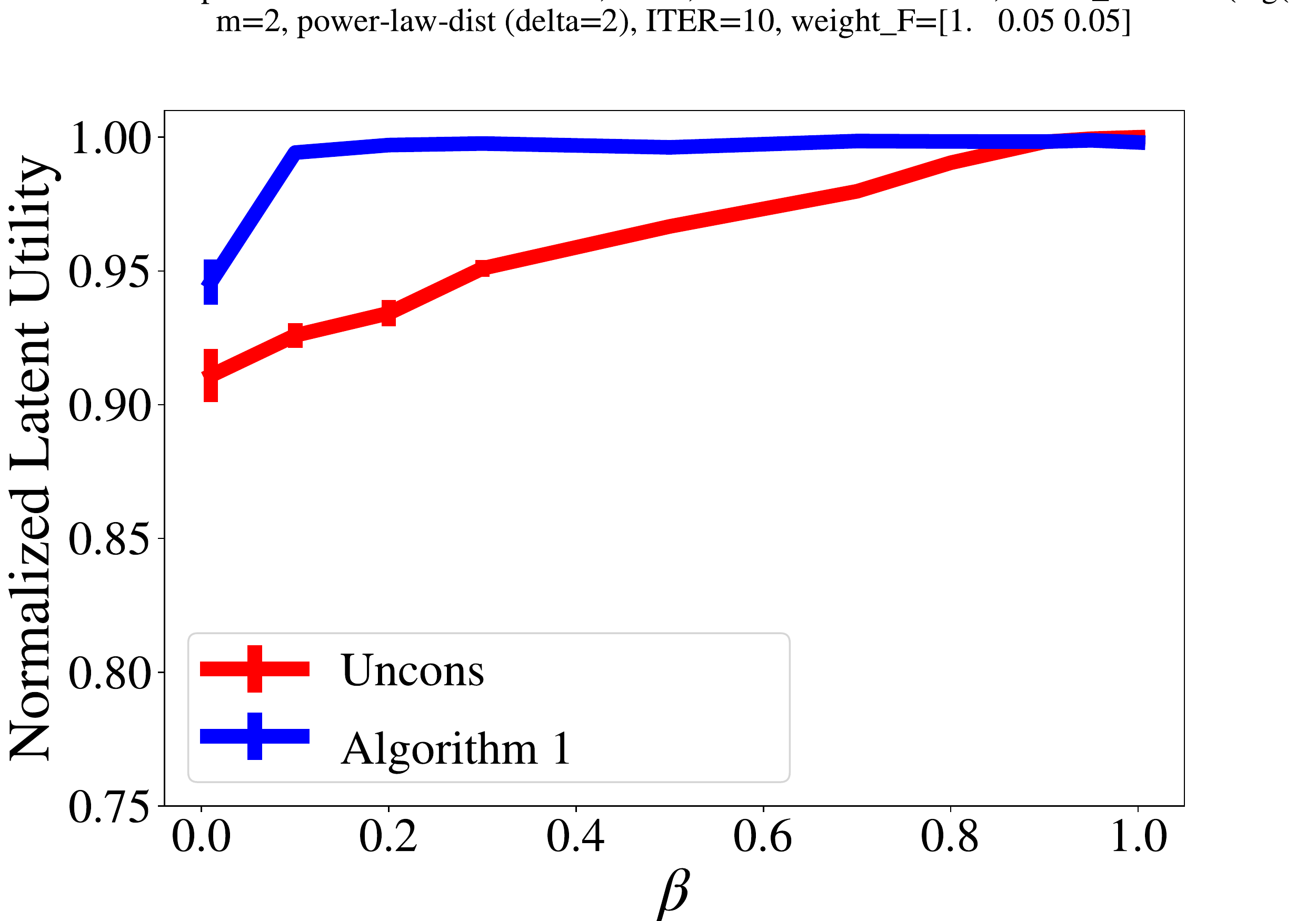}};
          \node[rotate=0] at (2.8,-2.75) {\textit{\small(less bias)}};
          \node[rotate=0] at (-2.14375, -2.75) {\textit{\small(more bias)}};
        \end{tikzpicture}
        }
        \hspace{-5mm}
        \par
        \caption{
        {\em Simulation on synthetic data:}
        We run \cref{alg:disj}  and \uncons{} (\cref{fig:syn_iid_disj,fig:syn_iid_non_disj}) on synthetic data and report their normalized latent  utility (error bars denote standard error of the mean).
        The results show that \uncons{} can lose a significant fraction of the optimal latent utility in the presence of bias (up to 15\% for $\beta<0.1$)
        while \cref{alg:disj}  have a normalized latent utility higher than 0.99.
        }
        \label{fig:syn_data_combined}
    \end{figure}
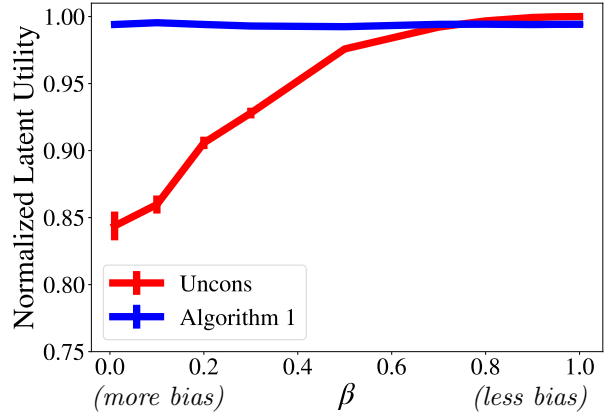
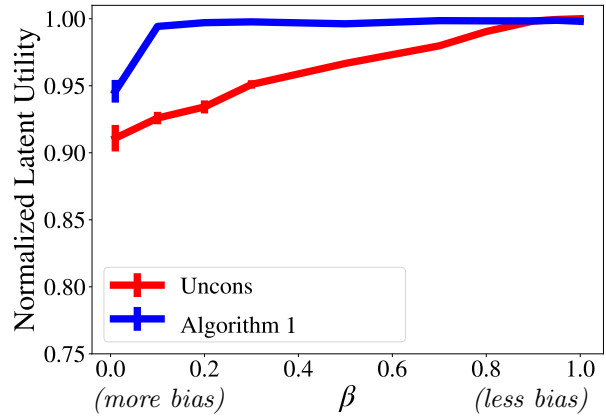

    \subsection{Simulations with real-world data}
        In this simulation, we evaluate \cref{alg:disj} on real-world data with pre-existing bias (as discussed below)
        and show that it can outperform \textsf{ProportionalRepr} and \uncons{}, even when the data does not satisfy \cref{asmp:disj}.
 
        \paragraph{Data.}
            MovieLens 20M \cite{MovieLensDataset} contains 20 million user ratings (on a scale of 0 to 5) for 27,000 movies submitted by 138,000 users of the \url{movielens.org}.
            For each movie $i$, apart from user-rating, the data has a set of genres of $i$ and relevance scores $r_{ig}\in [0,1]$ of each genre $g$ encoding ``how strongly [movie $i$] exhibits particular properties represented by [genre $g$].''
            In addition, we gathered information about the lead actor (the first-listed cast member) of each movie from \url{movielens.org}.

        \paragraph{Preprocessing.}
            For each movie $i$, we predict the (probable) gender of its lead actor using the Genderize API (\url{gender-api.com})
            and remove all movies where this prediction has confidence less than 0.9; this leaves 6612 and 1990 movies led by male and non-male actors respectively.\footnote{We choose a high threshold of 0.9 to ensure that the gender predictions have a low error rate among the remaining movies. We repeated the simulation with thresholds 0.7 and 0.8 and observed similar results.}
            Among the remaining movies, movies led by male actors have disproportionately higher relevance scores on genres that are stereotypically associated with men (e.g., ``Action'' or ``War'') compared to movies led by non-male actors (differing by up to 300\%; see Table~\ref{table:rel_score_movie} in \cref{sec:additional_simulations}).
        In contrast to relevance scores, user ratings are relatively balanced across movies led by male and non-male actors across all genres (the difference is at most 6\% across all genres, see \cref{table:user_rating_per_genre} in \cref{sec:additional_simulations}).
        (This observation also holds if we consider all 27,000 movies; see \cref{table:rel_score_movie_all} in \cref{sec:additional_simulations}.)
        Hence, compared to the user ratings, the relevance scores are systematically lower for movies led by an actor of a non-stereotypical gender in many genres.
 
        \paragraph{Setup.}
            Given a subset $T$ of genres, the task is to recommend $k$ movies from genres $T$ to the users.
            For each genre $g$, let $R_g$ be the ratio of the average relevance score of movies in this genre led by male actors and non-male actors.
            We select all genres $g$ where $R_g\geq 2$, these are $B=\{$\texttt{action}, \texttt{adventure}, \texttt{crime}, \texttt{western}, and \texttt{war}$\}$.
            Given a set $M$ of movies and a subset of selected genres $T\subseteq B$, we recommend a subset of movies $S$ that maximizes the following ``observed utility:''
            $$\hF(S)\coloneqq \sum\nolimits_{g\in T} \sqrt{\sum\nolimits_{i\in S} r_{ig}}.$$
            This function captures the benefit of recommending movies that are relevant to the selected genres $T$ and the $\sqrt{\cdot}$ captures the diminishing return of recommending multiple movies from the same genre $g\in T$.
            We use the user ratings to evaluate the quality of the recommended movies (or their ``latent utility''): Given a set of movies $S$, we say its latent utility is $$F(S)\coloneqq \frac{1}{\abs{S}}\sum\nolimits_{i\in S}{\rm rat}_{i},$$ where ${\rm rat}_i$ is the average user rating of movie $i$.

            In the simulation, we vary $k\in \inbrace{50,100, 150, 200}$ and select different nonempty subsets $T\subseteq B$ (31 such subsets).
            For each $k$ and $T$, we repeat the simulation 100 times.
            In each iteration, we draw a user $u$ uniformly at random from the set of users who have rated at least $200$ movies (19\% of the total users).
            We set $M$ to be the set of movies rated by user $u$; intuitively, this ensures that $M$ contains movies that are watched by users on the platform.
            \cref{asmp:disj} requires that each movie has a unique genre.
            This does not hold, but we still use \cref{alg:disj} (without modifications).

        \paragraph{Results.}
            We report the normalized latent utilities of \uncons{}, \textsf{ProportionalRepr}, and \cref{alg:disj} in \cref{fig:real_world_4} (for all subsets $T$ of size 4).
            The results for the remaining choices of $T$ appear in \cref{fig:movie_lens_pairs_action,fig:movie_lens_pairs_adventure,fig:movie_lens_pairs_crime,fig:real_world_3,fig:real_world_5} in \cref{sec:additional_simulations}.
            Across all choices of $T$, we observe that \cref{alg:disj} achieves {3\%} higher quality than \uncons{} and \textsf{ProportionalRepr} for {14/31} subsets of stereotypical genres ({>45\%}), has similar quality (within {1\%}) as \uncons{} or \textsf{ProportionalRepr} in {13/31} subsets of stereotypical genres ({42\%}), and has up to {2\%} lower quality than \uncons{} in {4/31} choices of genres ({13\%}).

    \renewcommand{\folder}{./figures/real-world-data/4}
    \begin{figure}[t!]
        \centering
        \subfigure[Action, Adventure, Crime, War\label{6a}]{
            \begin{tikzpicture}
                  \node (image) at (0,-0.17) {\includegraphics[width=0.5\linewidth, trim={0.5cm 0cm 1.7cm 0.5cm},clip]{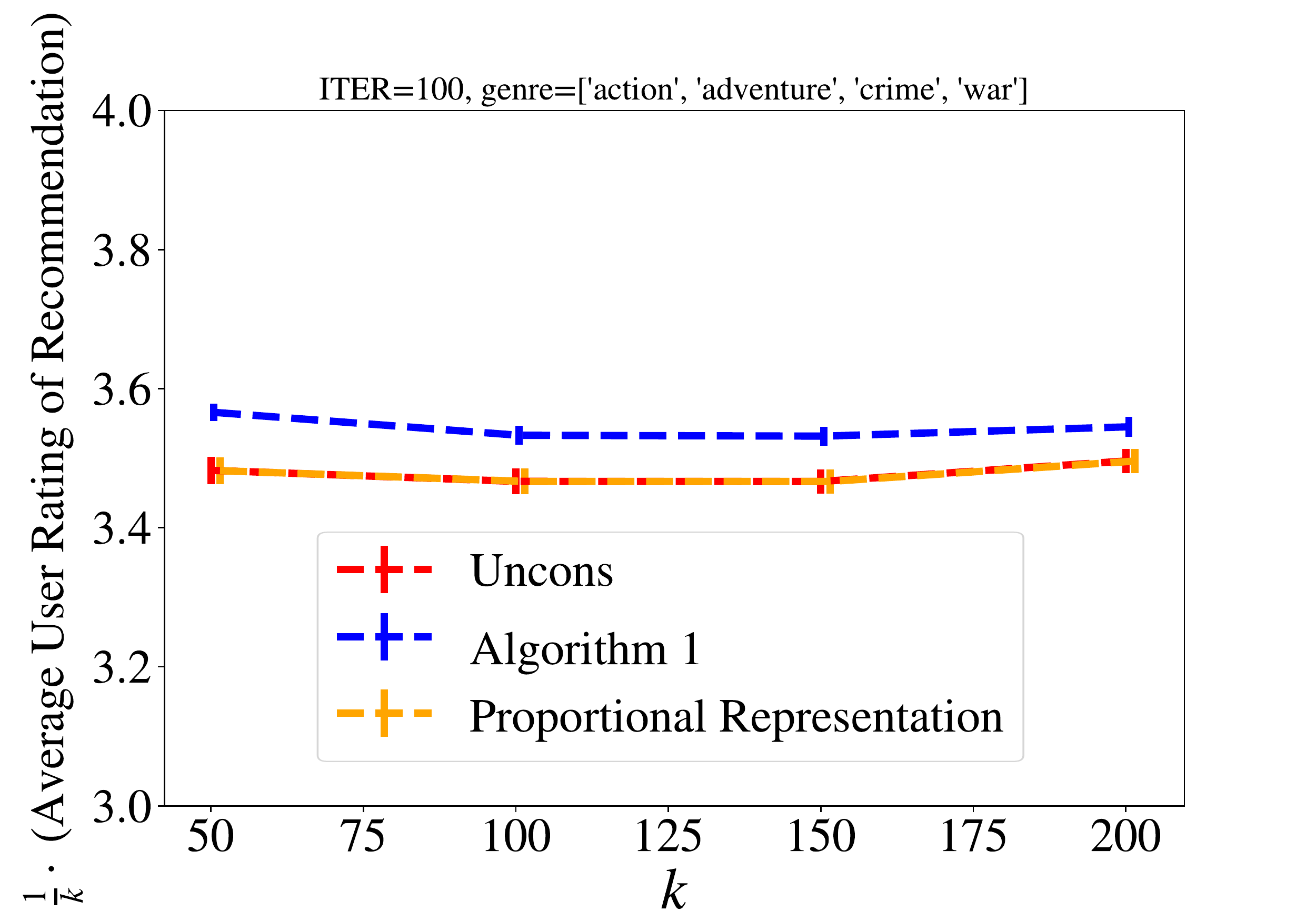}};
            \end{tikzpicture} 
        }
        \subfigure[Action, Adventure, Crime, Western\label{6c}]{
            {\begin{tikzpicture}
                  \node (image) at (0,-0.17) {\includegraphics[width=0.5\linewidth, trim={0cm 0cm 1.7cm 0.5cm},clip]{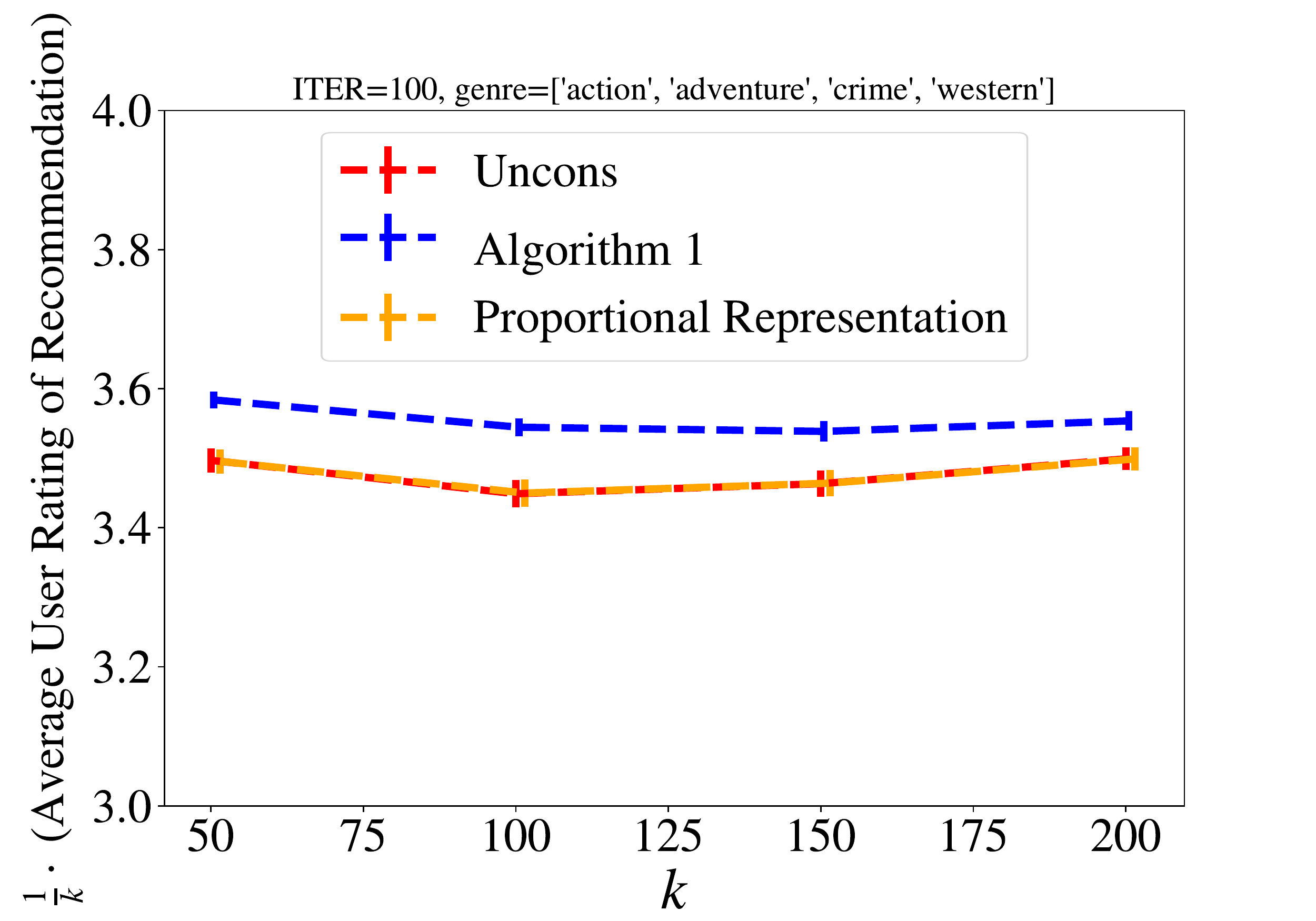}}; 
            \end{tikzpicture}} 
        }
        \caption{
        {\em Simulation on MovieLens 20M data:}
            The relevance scores in the data are disproportionately higher (up to 3 times) for movies led by male actors compared to movies led by non-male actors, in genres stereotypically associated with men.
            In contrast, user ratings for these sets of movies are within 6\% for each genre.
            We use relevance scores to recommend $k\in\negsp{} \inbrace{50,\negsp{} 100,\negsp{} 150,\negsp{} 200}$ movies from different subsets of men-stereotypical genres and use user ratings to estimate the latent utility of the recommended movies.
            {\cref{6a,6c} present results for two subsets of size 4}--we observe that \cref{alg:disj} achieves {1.40\% and 0.28\%} higher normalized latent utility than \uncons{} and \textsf{ProportionalRepr} for all $k$.
            (Results for other genre subsets of other sizes appear in \cref{fig:movie_lens_pairs_action,fig:movie_lens_pairs_adventure,fig:movie_lens_pairs_crime,fig:real_world_3,fig:real_world_5} in \cref{sec:additional_simulations}.)
        }
        \label{fig:real_world_4}
    \end{figure}
 
\addtocontents{toc}{\protect\setcounter{tocdepth}{2}}

\section{Proofs}\label{sec:proofs}

    \subsection{{Proof of \cref{thm:main_negative_result}}}
    \label{sec:proofof:thm:main_negative_result}
    \label{sec:proofof:thm:UB_any_const}
    \label{sec:proofof:thm:UB_any_const_app}
        In this section, we show that a family of fairness constraints is insufficient to guarantee any constant fraction of the optimal latent utility.
        This family of fairness constraints is parameterized by $2p$ parameters $u_1,u_2,\dots,u_p\geq 0$ and $v_1,v_2,\dots,v_p\geq 0$.
        Given $u_1,\dots,u_p$ and $v_1,v_2,\dots,v_p$, the constraints require any output subset $S$ to satisfy 
        \begin{align*}
            \forall \ell\in [p],\quad \abs{S\cap G_\ell} \leq (u_\ell+v_\ell \gamma_\ell) \cdot k.
            \yesnum\label{eq:const_satisfied}
        \end{align*}
        Where for each $\ell\in [p]$
        \begin{align*}
            \gamma_\ell \coloneqq \frac{\abs{G_\ell}}{n}.
        \end{align*}
        This family captures equal representation when $u_\ell=\frac{1}{p}$ and $v_\ell=0$ (for each $\ell\in [p]$), proportional representation when $u_\ell=0$ and $v_\ell=1$ (for each $\ell\in [p]$), and it ensures that each output subset satisfies the 4/5ths rule when $u_\ell=0$ and $\frac{\min_\ell v_\ell}{\max_\ell v_\ell}\geq \frac{4}{5}$.
        Given $U\coloneqq (u_1,\dots,u_p)$ and $V\coloneqq (v_1,\dots,v_p)$, let $S_{UV}\subseteq [n]$ be the set of size at most $k$ that maximizes observed utility subject to satisfying the constraint specified by $(U, V)$, i.e.,
        $$S_{UV} \coloneqq \argmax_{\abs{S}\leq k\colon \text{$S$ satisfies \cref{eq:const_satisfied}}}\hF(S).$$
        \noindent In this section, we prove the following theorem.
        \begin{theorem} %
            For any $0<\eps<1$, $U=(u_1,u_2)$, and $V=(v_1,v_2)$
            there exists 
            \begin{itemize}[itemsep=0pt,leftmargin=\leftmarginCUSTOM]
                \item a submodular function $F\in \evF$, %
                \item numbers $0\leq \gamma_1,\gamma_2\leq 1$ denoting group sizes,
                \item numbers $0<\beta_1,\beta_2\leq 1$ denoting the bias parameters, 
                \item  family of $n\times 3$ matrices $\cW$ parameterized by $n$, and 
                \item a constant $m_0$ dependent on the smallest non-zero value in $\inbrace{u_1,u_2}$, 
            \end{itemize}
            such that,
            for any 
                $k\geq m_0\cdot \poly\inparen{\eps^{-1}}$, 
                $n\geq m_0\cdot \poly\inparen{k, \eps^{-1}}$, 
            and $W=\cW(n)$, it holds that 
            \begin{align*}
                \Pr\insquare{F(S_{UV})\leq \eps\cdot  \opt{}} \geq 1-\eps,
            \end{align*}
            where the probability is over the randomness in $G_1$ and $G_2$.
        \end{theorem}
        \noindent The above theorem straightforwardly extends to $p>2$ groups by adding empty groups. It also extends to $m>2$ attributes by fixing utilities so that $W_{ij}=0$ for each item $i$ and $j\in [m]\backslash\inbrace{1,2}$.

        We claim that the following lower bounds on $k$ and $n$ are sufficient 
        \[
            k\geq m_0\cdot \frac{1}{\eps^{30}}\log\frac{e}{\eps}
            \quad\text{and}\quad 
            n\geq m_0\cdot \frac{4k}{\eps^4}\log\frac{e}{\eps}.
        \]
        We fix $m_0=1$ and suppose that $\eps$ is smaller than the smallest non-zero value in $\inbrace{u_1,u_2}$ (if any).
        Hence, we have the guarantee that: (1) if $u_1\neq 0$, then $u_1\geq \eps$ and, similarly, (2) if $u_2\neq 0$, then $u_2\geq \eps$.
        It is straightforward to see that one can allow for an arbitrary $0<\eps<1$ by choosing an appropriate $m_0$ that depends on the smallest non-zero value in $\inbrace{u_1,u_2}$ (if any).
        
        We divide the proof into cases depending on the value of $u_1$ and $v_1$.

        \subsubsection{Case A: ($v_1>\eps^{-1}$ and $u_1=0$)}
            Set the following parameters.
            \begin{itemize}[itemsep=0pt,leftmargin=\leftmarginCUSTOM]
                \item For any subset $S$ define $F(S) = \sum_{i\in S} W_{i1}$.
                \item Let $\gamma_1=\eps$ and $\gamma_2=1-\eps$.
                \item Let $\beta_1=1$ and $\beta_2=\eps^2$.
                \item Given $n$, let $W=\cW(n)$ be the following matrix: $W_{i1}=1$ for each $1\leq i \leq k$ and $W_{i1}=\eps$ otherwise. ($W_{ij}=0$ for any $j\neq 1$ and $i\in [n]$.)
            \end{itemize}  
            \noindent Let $\evE$ be the event that $\abs{G_1\cap [k]}\leq 2k\eps$.
            By the construction of $G_1$, $\abs{G_1\cap [k]}$ is a hypergeometric random variable.
            Using the standard concentration inequality (\cite[Theorem 1]{hush2005concentration}) for hypergeometric random variables, it follows that
            \begin{align*}
                \Pr\insquare{ \abs{ \abs{G_1\cap [k]}  - \gamma_1 k }\geq
                \sqrt{\frac{k}{2}\log{\frac{1}{\eps}}}}
                \leq \eps
                \qquad\Stackrel{k\eps \geq \sqrt{\frac{k}{2}\log{\nfrac{1}{\eps}}} }{\implies}\qquad 
                \Pr\insquare{ \abs{G_1\cap [k]} \geq  2k\eps} \geq 1-\eps.
            \end{align*}
            Thus, $\Pr[\evE]\geq 1-\eps.$
            By the construction of $W$, $\beta_1$, and $\beta_2$, the following holds
            \begin{align*}
                \forall i\in [G_1], \forall j \in [G_2],\quad 
                \hW_{i1} \geq \hW_{j1}.
            \end{align*}
            This implies that $\abs{S_{UV}\cap G_1}=k$ and $\abs{S_{UV}\cap G_2}=0$. 
            If not, then we can swap any item in $S_{UV}\cap G_2$ with an item from $G_1\backslash S_{UV}$ to increase $S_{UV}$'s observed utility.
            $G_1\backslash S_{UV}$ is non empty as $\abs{G_1}=n\eps > k \geq \abs{S_{UV}}$.
            The resulting set satisfies the constraints as in this case $u_1+v_1\gamma_1 \geq k$.
            Conditioned on $\evE$, we have 
            \begin{align*}
                F(S_{UV}) &=  \sum\nolimits_{i\in S_{UV}\cap G_1} W_{i1} \tag{Since $\abs{S_{UV}\cap G_2}=0$}\\
                &\leq  \sum\nolimits_{i\in G_1\cap [k]} W_{i1} + \sum\nolimits_{i\in (S_{UV}\cap G_1)\backslash [k]} W_{i1} \\
                &\leq  2k\eps + k\eps \tag{Using that conditioned on $\evE$, $\abs{G_1\cap [k]}\leq 2k\eps$, $\abs{S_{UV}}\leq k$, and the value of $W$}\\%
                &\leq 3k\eps. 
            \end{align*}
            From construction, $\opt=k$.
            Thus, it holds that 
            \begin{align*}
                \Pr\insquare{F(S_{UV})\leq 3\eps\cdot  \opt{}}
                &\geq \Pr\insquare{F(S_{UV})\leq 3\eps\cdot  \opt{}\mid \evE}\Pr[\evE]\\ 
                &\geq (1-\eps).   
            \end{align*}
            The claim follows by scaling down $\eps$ by a factor of 3 in the construction.

        \subsubsection{Case B: ($v_1<\eps$ and $u_1=0$)}
            Set the following parameters.
            \begin{itemize}[itemsep=0pt,leftmargin=\leftmarginCUSTOM]
                \item For any subset $S$ define $F(S) = \sum_{i\in S} W_{i1}$.
                \item Let $\gamma_1=1-\eps$ and $\gamma_2=\eps$.
                \item Let $\beta_1=1$ and $\beta_2=\eps^2$.
                \item Given $n$, let $W=\cW(n)$ be the following matrix: $W_{i1}=1$ for each $1\leq i \leq k$ and $W_{i1}=\eps$ otherwise. ($W_{ij}=0$ for any $j\neq 1$ and $i\in [n]$.)
            \end{itemize}  
            \noindent Let $\evE$ be the event that $\abs{G_2\cap [k]}\leq 2k\eps$.
            By replacing $G_1$ and $\gamma_1$ by $G_2$ and $\gamma_2$ in the calculation of $\Pr[\evE]$ in Case A, it follows that $\Pr[\evE]\geq 1-\eps$.
            Since $u_1+v_1\gamma_1 < \eps$ in this case, it holds that $\abs{S_{UV}\cap G_1}\leq k\eps$.
            Conditioned on $\evE$, we have 
            \begin{align*}
                F(S_{UV}) &=  \sum\nolimits_{i\in S_{UV}\cap G_1} W_{i1} + \sum\nolimits_{i\in S_{UV}\cap G_2\cap [k]} W_{i1} + \sum\nolimits_{i\in (S_{UV}\cap G_2)\backslash [k]} W_{i1}\\ 
                &\leq k\eps  + \sum\nolimits_{i\in S_{UV}\cap G_2\cap [k]} W_{i1} + \sum\nolimits_{i\in (S_{UV}\cap G_2)\backslash [k]} W_{i1}
                \tag{Using that $W_{ij}\leq 1$ and $\abs{S_{UV}\cap G_1}\leq k\eps$}\\
                &\leq  3k\eps + \sum\nolimits_{i\in (S_{UV}\cap G_2)\backslash [k]} W_{i1} \tag{Using that conditioned on $\evE$, $\abs{G_2\cap [k]}\leq 2k\eps$ and $W_{ij}\leq 1$}\\
                &\leq 4k\eps. 
                \tag{Using that for any $i\not \in [k]$, $W_{i1}=\eps$ and $\abs{S_{UV}}\leq k$}
            \end{align*}
            From construction, $\opt=k$.
            Thus, it holds that 
            \begin{align*}
                \Pr\insquare{F(S_{UV})\leq 4\eps\cdot  \opt{}}
                &\geq \Pr\insquare{F(S_{UV})\leq 4\eps\cdot  \opt{}\mid \evE}\cdot \Pr[\evE]\\ 
                &\geq (1-\eps).   
            \end{align*}
            The claim follows by scaling down $\eps$ by a factor of 4 in the construction.

        \newcommand{\ub}{{\eps^3}}
            
        \subsubsection{Case C: ($\eps <v_1<\eps^{-1}$ and $u_1=0$)}
            Set the following parameters.
            \begin{itemize}[itemsep=0pt,leftmargin=\leftmarginCUSTOM]
                \item For any subset $S$ define \begin{align*}
                    F(S) = \inparen{\sum\nolimits_{i\in S} W_{i1}}^{1/3} + \eps \sqrt{\sum\nolimits_{i\in S} W_{i2}}.
                    \yesnum\label{eq:const_of_f_case_3_new_neg_result}
                \end{align*}    
                \item Let $\gamma_1={\eps^4}$ and $\gamma_2=1-\eps^4$.
                \item Let $\beta_1=1$ and $\beta_2=\beta$ (for any $\beta>0$). 
                \item We divide items into two types. Let $A\coloneqq [n/2]$ be the set of type $A$ items and $B\coloneqq [n]\backslash A$ be the set of type $B$ items.
                Given $n$, let $W=\cW(n)$ be the following matrix: 
                \begin{itemize}
                    \item for each type $A$ item $i\in A$, $W_{i}=(1,0)$, and
                    \item for each type $B$ item $i\in B$, $W_{i}=(0,1)$.
                \end{itemize}
            \end{itemize}  
            Let $\evE$ be the event that (1) $A$ has at least $k$ items from each of $G_1$ and $G_2$ and (2) $B$ has at least $k$ items from each of $G_1$ and $G_2$.
            By the construction of $G_1$ and $G_2$, $\abs{A\cap G_\ell}$ and $\abs{B\cap G_\ell}$ are hypergeometric random variables (for each $\ell\in [2]$).
            Using the standard concentration inequality (\cite[Theorem 1]{hush2005concentration}) for hypergeometric random variables, it follows that for each $\ell\in [p]$
            \begin{align*}
                &\Pr\insquare{ \abs{ \abs{A\cap G_\ell}  - \gamma_1 \abs{A} }\geq
                  \sqrt{\frac{\abs{A}}{2}\log{\frac{1}{\eps}}}}
                \ \ \leq\quad \eps\\ 
                \Stackrel{\gamma_1,\gamma_2\geq \eps^4,\ \abs{A}=\frac{n}{2}}{\implies}\qquad\quad  
                &\Pr\insquare{ \abs{A\cap G_1} \geq  \frac{\eps^4 n}{2} - \sqrt{\frac{n}{4} \log\frac{1}{\eps}}} 
                \qquad \ \ \geq\quad 1-\eps.
                \yesnum\label{eq:case_c_new_neg_result}
            \end{align*}
            Using that $\frac{\eps^4 n}{2} - \sqrt{\frac{n}{4} \log\frac{1}{\eps}}$ is an increasing function of $n$ for $n\geq \frac{1}{4\eps^8}\log\frac{1}{\eps}$ and by construction $n\geq \frac{1}{\eps^8}\log\frac{1}{\eps}$, it follows that 
            \begin{align*}
                \frac{\eps^4 n}{2} - \sqrt{\frac{n}{4} \log\frac{1}{\eps}}
                \qquad \Stackrel{n\geq \frac{4k}{\eps^4}\log\frac{e}{\eps}}{\geq}\qquad \inparen{2k-\eps^{-2}\sqrt{k}}\cdot \log\frac{1}{\eps}
                \geq  k \cdot \log\frac{1}{\eps}. \tag{Using that $k\geq \eps^{-4}$ by construction}
            \end{align*}
            Substituting this in \cref{eq:case_c_new_neg_result}, it follows that for each $\ell\in [p]$,
            $\Pr\insquare{ \abs{A\cap G_\ell} \geq  k}\geq 1-\eps$.
            Replacing $A$ by $B$ in the above argument, $\Pr\insquare{ \abs{B\cap G_\ell} \geq  k}\geq 1-\eps$.
            Thus, by union bound $\Pr[\evE]\geq 1-4\eps$.
            
            We claim that conditioned on $\evE$
            \begin{align*}
                \abs{S_{UV}\cap G_2 \cap A} \leq \ub{} k.\yesnum\label{claim:case_c_new_neg_result}
            \end{align*}
            Suppose that the above claim is true.
            Conditioned on $\evE$ we have that 
            \begin{align*}
                F(S_{UV}) &= \inparen{\abs{S_{UV}\cap G_1\cap A} + \abs{S_{UV}\cap G_2\cap A}}^{1/3}
                + \eps \sqrt{\abs{S_{UV}\cap G_1\cap B} + \abs{S_{UV}\cap G_1\cap B}}
                \tag{Using \cref{eq:const_of_f_case_3_new_neg_result} and construction of $W$}\\
                &= \inparen{\abs{S_{UV}\cap G_1\cap A} + \eps^3 k }^{1/3}
                + \eps \sqrt{\abs{S_{UV}\cap G_1\cap B} + \abs{S_{UV}\cap G_1\cap B}}
                \tag{Using \cref{claim:case_c_new_neg_result}}\\
                &= \inparen{2\eps^3 k}^{1/3} + \eps \sqrt{k} \tag{Using that $S_{UV}$ satisfies the constraints, hence, $\abs{S_{UV}\cap G_1\cap A}\leq (u_1+\gamma_1 v_1) k=\eps^3 k$ and $\abs{S_{UV}}\leq k$}\\
                &\leq 3\eps k. \tag{Using \cref{eq:const_of_f_case_3_new_neg_result}}
            \end{align*}
            By construction, $\opt\geq \sqrt{k}$.
            Thus, 
            \begin{align*}
                \Pr\insquare{F(S_{UV})\leq 3\eps\cdot  \opt{}}
                &\geq \Pr\insquare{F(S_{UV})\leq 3\eps\cdot  \opt{}\mid \evE}\cdot \Pr[\evE]\\
                &\geq (1-4\eps).   
            \end{align*}
            The theorem's claim follows by scaling down $\eps$ by a factor of 4 in the construction.
            
            It remains to prove the claim above (\cref{claim:case_c_new_neg_result}).
            Towards a contradiction assume that
            \begin{align*}
                \abs{S_{UV}\cap G_2\cap A} > \ub{} k.
                \yesnum\label{eq:contradiction_case_c_new_negative_res}
            \end{align*}
             
                \noindent Consider a subset $S$ that is the same as $S_{UV}$ except that compared to $S_{UV}$ it selects one less item from $G_2\cap A$ and one more item from $G_2\cap B$.
                $S$ exists because, conditioned on $\evE$, $\abs{G_2\cap B} \geq k$.
                $S$ satisfies the constraints specified by $(U,V)$ as $\abs{S\cap G_\ell}=\abs{S_{UV}\cap G_\ell}$ for each $\ell\in [2]$ and $S_{UV}$ satisfies the constraints specified by $(U,V)$.
                It holds that 
                \begin{align*}
                    \hF(S)-\hF(S_{UV})
                    &= \inparen{\sabs{S_{UV}\cap A\cap G_1}+\beta{}\sabs{S_{UV}\cap A\cap G_2}-\beta{}}^{1/3}
                    +\eps \sqrt{\sabs{S_{UV}\cap B\cap G_1}+\beta{}\sabs{S_{UV}\cap B\cap G_2}+\beta{}}\\
                    &-\inparen{\inparen{\sabs{S_{UV}\cap A\cap G_1}+\beta{}\sabs{S_{UV}\cap A\cap G_2} }^{1/3}
                    +\eps\cdot \sqrt{\sabs{S_{UV}\cap B\cap G_1}+\beta{} \sabs{S_{UV}\cap B\cap G_2}}}.\\ 
                    \intertext{The RHS of the above equation is a decreasing function of $\sabs{S_{UV}\cap B\cap G_1}+\beta\sabs{S_{UV}\cap B\cap G_2}$.
                    Since $S_{UV}$ satisfies the constraints specified by $(U,V)$, $\sabs{S_{UV}\cap B\cap G_1}+\beta{}\sabs{S_{UV}\cap B\cap G_2}\leq (\eps^3+\beta{}) k = 2\beta{} k$. Consequently}
                    \hF(S)-\hF(S_{UV}) &\geq \inparen{\sabs{S_{UV}\cap A\cap G_1}+\beta{}\sabs{S_{UV}\cap A\cap G_2}-\beta{}}^{1/3}
                    +\eps\inparen{ \sqrt{2\beta{} k + \beta{}} - \sqrt{2\beta{} k}}\\
                    &-{\inparen{\sabs{S_{UV}\cap A\cap G_1}+\beta{}\sabs{S_{UV}\cap A\cap G_2} }^{1/3}.
                    }
                    \intertext{RHS of the above inequality is an increasing function of $\sabs{S_{UV}\cap A\cap G_2}$ and, from \cref{eq:contradiction_case_c_new_negative_res}, $\sabs{S_{UV}\cap A\cap G_2} > \ub{} k$. Hence, }
                    \hF(S)-\hF(S_{UV}) &\geq \inparen{\sabs{S_{UV}\cap A\cap G_1}+\beta{}\ub{} k-\beta{}}^{1/3}
                    +\eps\inparen{ \sqrt{2\beta{} k + \beta{}} - \sqrt{2\beta{} k}}\\ 
                    &-\inparen{\sabs{S_{UV}\cap A\cap G_1} + \beta{}\ub{} k }^{1/3}.
                \end{align*}
                RHS of the above inequality is an increasing function of $\sabs{S_{UV}\cap A\cap G_1}$.
                In this case, $(u_1+v_1\gamma_1)\geq \eps\gamma_1 k = \eps^5 k$.
                We claim this implies that $\abs{S_{UV}\cap A\cap G_1}\geq \eps^5 k$.
                If this is not true, then one can increase the observed utility of $S_{UV}$ by removing one item in $S_{UV}$ from $A\cap G_2$ and adding one item in $S_{UV}$ from $A\cap G_1$.
                This is possible as (1) in this case, $\abs{S_{UV} \cap A\cap G_2}>0$ and (2) conditioned on $\evE$, $\abs{A\cap G_1}\geq k$.
                Combining $\abs{S_{UV}\cap A\cap G_1}\geq \eps^5 k$ with previous observation, it follows that
                \begin{align*}
                    \hF(S)-\hF(S_{UV}) &= \inparen{(\eps^5+\beta{}\ub{}) k-\beta{}}^{1/3}-\inparen{(\eps^5+\beta{}\ub{}) k}^{1/3}
                    +\eps\inparen{ \sqrt{2\beta{} k + \beta{}} - \sqrt{2\beta{} k}} \\
                    &\geq -\frac{\beta{}}{\inparen{\eps^3\beta k}^{2/3}} + \eps\inparen{ \sqrt{2\beta{} k + \beta{}} - \sqrt{2\beta{} k}} \tag{Using that for all $0\leq x\leq 1$, $\inparen{1-x}^{1/3}\geq 1-x$ and $\frac{\beta}{\eps^5 k+\beta\eps^{3} k}\leq 1$}\\
                    &\geq -\frac{\beta{}}{\inparen{\eps^3\beta k}^{2/3}} +\eps{}\frac{\beta}{\sqrt{6\beta k}}\tag{Using that for all $0\leq x\leq 3$, $\sqrt{1+x}\geq 1+\frac{x}{3}$ and $k\geq 1$}\\
                    &= \frac{\beta}{\sqrt{k}}\inparen{\sqrt{\frac{\eps}{6}} - \frac{1}{\eps^{5/3}k^{1/6}}}\tag{Using \cref{eq:const_of_f_case_3_new_neg_result} and that $\beta=\eps$}\\
                    &>0. \tag{Using that $k > 216\eps^{-13}$, $\beta>0$, and $k\geq 1$}
                \end{align*}

        \subsubsection{Case D: ($u_1 \geq \eps$)}
            Set the following parameters.
            \begin{itemize}[itemsep=0pt,leftmargin=\leftmarginCUSTOM]
                \item For any subset $S$ define \begin{align*}
                    F(S) = \inparen{\sum\nolimits_{i\in S} W_{i1}}^{1/3} + \eps^3 \sqrt{\sum\nolimits_{i\in S} W_{i2}}.
                    \yesnum\label{eq:const_of_f_case_4_new_neg_result}
                \end{align*}  
                \item Let $\gamma_1={\eps^3}$ and $\gamma_2=1-\eps^3$.
                \item Let $\beta_1=1$ and $\beta_2=\beta$ (for any $\beta>0$).
                \item We divide items into two types. Let $A\coloneqq [k]$ be the set of type $A$ items and $B\coloneqq [n]\backslash A$ be the set of type $B$ items.
                Given $n$, let $W=\cW(n)$ be the following matrix: 
                \begin{itemize}
                    \item for each type $A$ item $i\in A$, $W_{i}=(1,0)$, and
                    \item for each type $B$ item $i\in B$, $W_{i}=(0,1)$.
                \end{itemize}
            \end{itemize}  
            Let $\evE$ be the event that (1) $A$ has at least $\frac{1}{2}\eps^3 k$ items from $G_1$, (2) $A$ has at most $2\eps^3 k$ items from $G_1$,
            and (3) $B$ has at least $k$ items from $G_2$.
            By the construction of $G_1$ and $G_2$, $\abs{A\cap G_\ell}$ and $\abs{B\cap G_\ell}$ are hypergeometric random variables (for each $\ell\in [2]$).
            Using the standard concentration inequality (\cite[Theorem 1]{hush2005concentration}) for hypergeometric random variables, it follows that for each $\ell\in [p]$
            \begin{align*}
                \Pr\insquare{ \abs{ \abs{A\cap G_1}  - \gamma_1 \abs{A} }\geq
                  \sqrt{\frac{\abs{A}}{2}\log{\frac{1}{\eps}}}}
                \leq \eps
                \qquad\Stackrel{\gamma_1,\gamma_2\geq \eps^3,\ \abs{A}=k}{\implies}\qquad 
                \Pr\insquare{ \abs{A\cap G_1} \geq  \eps^3 k - \sqrt{\frac{k}{2} \log\frac{1}{\eps}}} \geq 1-\eps.
                \yesnum\label{eq:case_d_new_neg_result}
            \end{align*}
            Since $k\geq \frac{2}{\eps^6} \log\frac{1}{\eps}$, it follows that $\sqrt{\frac{k}{2} \log\frac{1}{\eps}}\leq \frac{\eps^3 k}{2}$ and, hence, $\eps^3 k - \sqrt{\frac{k}{2} \log\frac{1}{\eps}} \geq \frac{\eps^3 k}{2}.$
            This implies that 
            \begin{align*}
                \Pr\insquare{ \abs{A\cap G_1} \geq  \frac{\eps^3 k}{2} \text{ and } \abs{A\cap G_1} \leq  \frac{3\eps^3 k}{2}} \geq 1-2\eps.
            \end{align*}
            A similar calculation shows that $\Pr\insquare{ \abs{B\cap G_2} \geq  k}\geq 1-\eps$.
            Thus, by union bound $\Pr[\evE]\geq 1-3\eps$.
            
            We claim that conditioned on $\evE$
            \begin{align*}
                \abs{S_{UV}\cap G_2 \cap A} =0.\yesnum\label{claim:case_d_new_neg_result}
            \end{align*}
            Suppose that the above claim is true.
            Conditioned on $\evE$ we have that 
            \begin{align*}
                F(S_{UV}) &= \inparen{\abs{S_{UV}\cap G_1\cap A} + \abs{S_{UV}\cap G_2\cap A}}^{1/3}
                +\eps^3  \sqrt{\abs{S_{UV}\cap G_1\cap B} + \abs{S_{UV}\cap G_1\cap B}}
                \tag{Using \cref{eq:const_of_f_case_4_new_neg_result} and construction of $W$}\\
                &= \inparen{0 + \abs{S_{UV}\cap G_2\cap A}}^{1/3}
                + \eps^3 \sqrt{\abs{S_{UV}\cap G_1\cap B} + \abs{S_{UV}\cap G_1\cap B}}
                \tag{Using \cref{claim:case_d_new_neg_result}}\\
                &\leq \inparen{\frac{3}{2}\eps^3 k}^{1/3} + \eps^3\sqrt{k} \tag{Using that conditioned on $\evE$, $\abs{A\cap G_2}\leq \frac{3}{2}\eps^3k$ and $\abs{S_{UV}}\leq k$}\\
                &\leq 3\eps k.
            \end{align*}
            By construction, $\opt\geq \sqrt{k}$.
            Thus, 
            \begin{align*}
                \Pr\insquare{F(S_{UV})\leq 3\eps\cdot  \opt{}}
                &\geq \Pr\insquare{F(S_{UV})\leq 3\eps\cdot  \opt{}\mid \evE}\cdot \Pr[\evE]\\ 
                &\geq (1-3\eps).   
            \end{align*}
            The theorem's claim follows by scaling down $\eps$ by a factor of 3 in the construction.
            
            It remains to prove the claim above (\cref{claim:case_d_new_neg_result}).
            Towards a contradiction assume that
            \begin{align*}
                \abs{S_{UV}\cap G_2\cap A} \geq 1.
                \yesnum\label{eq:contradiction_case_d_new_negative_res}
            \end{align*}
            \noindent In this case, it holds that $\abs{S_{UV}\cap G_1\cap A}\geq \frac{\eps^3k}{2}$.
            (This is true because if $\abs{S_{UV}\cap G_1\cap A} < \frac{\eps^3k}{2}$, then one can add an additional item from $G_1\cap A$ and remove one item from $G_2\cap A$, which improves the observed utility while satisfying the constraints, and contradicts the fact that $S_{UV}$ maximizes the observed utility subject to satisfying the constraints.)

            \newcommand{\smallmath}[1]{\text{\small $#1$}}
               
            Consider a subset $S$ that is the same as $S_{UV}$ except that compared to $S_{UV}$ it selects one less item from $G_2\cap A$ and one more item from $G_2\cap B$.
            One can verify that conditioned on $\evE$, $S$ exists and satisfies the fairness constraints.
                It holds that 
                \begin{align*}
                    \hF(S)-\hF(S_{UV})
                    &= \smallmath{\inparen{\sabs{S_{UV}\cap A\cap G_1}+\beta{}\sabs{S_{UV}\cap A\cap G_2}-\beta{}}^{1/3}
                    +\eps^3\sqrt{\sabs{S_{UV}\cap B\cap G_1}+\beta{}\sabs{S_{UV}\cap B\cap G_2}+\beta{}}}\\
                    &\quad \smallmath{-{\inparen{\sabs{S_{UV}\cap A\cap G_1}+\beta{}\sabs{S_{UV}\cap A\cap G_2} }^{1/3}
                    -\eps^3 \sqrt{\sabs{S_{UV}\cap B\cap G_1}+\beta{} \sabs{S_{UV}\cap B\cap G_2}}}.}\\ 
                    \intertext{The RHS of the above equation is an increasing function of $\sabs{S_{UV}\cap A\cap G_1}$ and $\sabs{S_{UV}\cap A\cap G_1}\geq \frac{\eps^3 k}{2}$ and, hence,}
                    \hF(S)-\hF(S_{UV}) &\geq 
                    \smallmath{\inparen{\frac{\eps^3 k}{2}+\beta{}\sabs{S_{UV}\cap A\cap G_2}-\beta{}}^{1/3}
                    +\eps^3 \sqrt{\sabs{S_{UV}\cap B\cap G_1}+\beta{}\sabs{S_{UV}\cap B\cap G_2}+\beta{}}}\\
                    &\quad \smallmath{-{\inparen{\frac{\eps^3 k}{2}+\beta{}\sabs{S_{UV}\cap A\cap G_2} }^{1/3}
                    +\eps^3 \sqrt{\sabs{S_{UV}\cap B\cap G_1}+\beta{} \sabs{S_{UV}\cap B\cap G_2}}}.}
                    \\
                    \intertext{RHS of the above inequality is a decreasing function of $\sabs{S_{UV}\cap B\cap G_1}+\beta\cdot \sabs{S_{UV}\cap B\cap G_2}$ and $\sabs{S_{UV}\cap B\cap G_1}+\beta\cdot \sabs{S_{UV}\cap B\cap G_2}\leq k$. Hence,}
                    \hF(S)-\hF(S_{UV}) &\geq \smallmath{\inparen{\frac{\eps^3 k}{2}+\beta{}\sabs{S_{UV}\cap A\cap G_2}-\beta{}}^{1/3}
                    +\eps^3\sqrt{k+\beta{}} -\inparen{\frac{\eps^3 k}{2}+\beta{}\sabs{S_{UV}\cap A\cap G_2} }^{1/3}
                    +\eps^3 \sqrt{k}}.
                \end{align*}
                RHS of the above inequality is an increasing function of $\sabs{S_{UV}\cap A\cap G_2}$ and $\abs{S_{UV}\cap A\cap G_1}\geq 1$.
                Hence, 
                \begin{align*}
                    \hF(S)-\hF(S_{UV}) &= \inparen{\frac{\eps^3 k}{2}}^{1/3}
                    +\eps^3\sqrt{k+\beta{}}
                    -\inparen{\frac{\eps^3 k}{2}+\beta{} }^{1/3}
                    +\eps^3 \sqrt{k}\\
                    &\geq -\eps \inparen{\frac{k}{2}}^{1/3} \frac{2\beta}{\eps^3 k} + \frac{\eps^3 \sqrt{k}\beta}{4k}\\
                    &\geq \frac{\beta}{k}\inparen{\eps^3 \sqrt{k} - \frac{k^{1/3}}{2^{1/3} \eps^2}}\\
                    &=  \frac{\beta k^{1/3}}{k} \cdot\inparen{\eps^3 k^{1/6} - \frac{1}{2^{1/3} \eps^2}}\\
                    &>0. \tag{Using that $k > \eps^{-30}$, $\beta>0$, and $k\geq 1$}
            \end{align*}

    \subsection{Proof of \cref{thm:disj}}\label{sec:proofof:thm:disj}
        In this section, we prove \cref{thm:disj}.
        For the reader's convenience, we restate \cref{thm:disj} and the assumption used in \cref{thm:disj}, below.
        \begin{assumption}[\textbf{Disjoint attributes}]\label{asmp:disj_app}
            The matrix $W\in \R^{n\times m}$ is such that $C_1,\dots,C_m$ are disjoint.
        \end{assumption}
        \begin{theorem}
            {Suppose $\phi_1(x)=x$ for each $x$.}
            There is an algorithm (\cref{alg:disj}) that, given observed utilities $\hW\in \R^{n\times m}$ and evaluation oracles for $g_1,g_2,\dots,g_m:\R_{\geq 0}\to \R_{\geq 0}$,
            outputs a set $S$ of size $k$ with the following property (under \cref{asmp:disj}):
            For any $\eps,\tau,\gamma>0$ and any {increasing functions $\phi_2,\phi_3,\dots,\phi_p:\R_{\geq 0}\to \R_{\geq 0}$}, there is a large enough $k_0$ such that for any $k\geq k_0$, 
            with probability at least $1-\eps$
            \begin{align*}
                F(S)\geq \opt{} \cdot \inparen{1-\eps}.
            \end{align*} 
            \noindent  Where the probability is over the randomness in the choice of $G_1,G_2,\dots,G_p$.
            The algorithm makes $O(nk)$ evaluations of each $g_1,g_2,\dots,g_m$ and does $O(nk\log{n})$ additional arithmetic operations.
        \end{theorem}

        \newcommand{\tS}{\wt{S}}
        \newcommand{\tF}{\wt{F}}

        \subsubsection{{Overview of the proof of \cref{thm:disj}}}
            We have to prove that the subset $\wt{S}$ output by \cref{alg:disj} has latent utility at least $\opt\cdot \inparen{1 - \wt{O}\inparen{k^{-\sfrac{1}{4}}} }$ with high probability, provided that \cref{asmp:disj} holds.
            One can show that because of \cref{asmp:disj}, $\sopt\coloneqq \argmax_{S\subseteq [n]: \abs{S}\leq k}F(S)$ is of the form
            $$\sopt\coloneqq \bigcup\nolimits_{j=1}^m S_j^\star,$$
            where for each $j$,
            $$S_j^\star\coloneqq \argmax_{S\subseteq C_j: \abs{S}\leq \abs{\sopt \cap C_j}}g_j\inparen{\sum\nolimits_{i\in S} W_{ij} }.$$
            Here, the function in the RHS $H(T)\coloneqq g_j\inparen{\sum\nolimits_{i\in T} W_{ij} }$ is a function in $\evF$ with one attribute.
            For such functions, there is an algorithm which, given {\em observed} utilities as input, outputs a subset $S$ with {\em latent} utility $F(S)$ at least $\opt{}\cdot \inparen{1- k^{-\sfrac{1}{2}}}$ with high probability (\cref{lem:pos_result_linear}).
            Hence, we can use \cref{lem:pos_result_linear} to compute (an approximation to) $S_j^\star$ and, hence, (an approximation to) $\sopt$.
            The catch is that the definition of $S_j^\star$, itself, depends upon $\sopt$ because of the $\abs{\sopt\cap C_j}$ term in its definition.
    
            {Our key idea is to estimate $\abs{\sopt{}\cap G_j}$ without actually computing $\sopt$.}
            This uses the stochasticity in the protected groups $G_1,G_2,\dots,G_p$ and that $\phi_1(x)=x$ for all $x$.
            In particular, given a function $F\in \evF$, $F(T)\coloneqq \sum\nolimits_{j=1}^m g_j\inparen{\sum\nolimits_{i\in S} W_{ij} }$, \mbox{we define the following alternate version of $F$:}
            \begin{align*}
                \forall T\subseteq [n],\quad \tF(T) \coloneqq \sum\nolimits_{j=1}^m g_j\inparen{\frac{n}{\abs{G_1}}\cdot \sum\nolimits_{i\in T\cap G_1} W_{ij} }.
                \yesnum\label{eq:def_of_fT_proof_overview}
            \end{align*}
            There are two differences between $\tF$ and $F$:
            First, given a set $T$ as input $\tF$ only sums the utilities over $T\cap G_1$.
            Second it scales the computed sums by a factor of $\frac{n}{\abs{G_1}}$.
            Intuitively, because $G_1$ is constructed by sampling elements uniformly at random without replacement, for any subset $T$ independent of $G_1$, the two changes cancel each other:
            For any subset $T$ independent of $G_1$, under some mild assumptions on $T$, we prove that $F(T)\in \tF(T)\cdot \inparen{1\pm k^{-\sfrac{1}{4}}}$ (\cref{lem:invariance}).
    
            \cref{alg:disj} computes a subset $\tS$ which selects $k\cdot \frac{\abs{G_1}}{n}$ items from $\abs{G_1}$ and maximizes $\tF$.
            \cref{alg:disj} is able to find such a subset because of \cref{asmp:disj} (see \cref{lem:optimality:dist}).
            We use $k_j\coloneqq \frac{n}{\abs{G_1}}\cdot \sabs{\tS\cap C_j}$ as an approximation of $\sabs{\sopt\cap C_j}$ for each $j\in [m]$.
            Finally, using \cref{lem:invariance} and some additional analysis (\cref{lem:disj_main}), we show that the subset $S$ defined as,
            \begin{align*}
                S\coloneqq \bigcup\nolimits_{j=1}^m S_j,
                \qquad\text{where}\qquad
                \forall j,\quad S_j\coloneqq \argmax_{T\subseteq C_j: \abs{T}\leq k_j} g_j\inparen{\sum\nolimits_{i\in S} W_{ij} },
                \yesnum\label{eq:def_of_S_proof_overview}
            \end{align*}
            has latent utility at least $\opt\cdot (1-\wt{O}(k^{-\sfrac{1}{8}}))$ (\cref{lem:high_utility_solution}).
            This analysis has to be careful to ensure that
            the assumptions needed by \cref{lem:invariance} are satisfied for all considered subsets.
            For example, $\sopt$ may not satisfy these assumptions and we construct a subset $\ovsopt$ which is ``similar'' to $\sopt$ and which satisfies these assumptions (\cref{lem:disj_main}).
            Moreover, $\tS$ is not independent of $G_1$ as required by \cref{lem:invariance}.
            To bypass this, we show that $\tS$ always belongs to a family with at most $k^m$ subsets; and show that \cref{lem:invariance} holds for all subsets in $(k+1)^m$ with high probability and, hence, in particular it holds for $\tS$.

        \subsubsection{Proof of \cref{thm:disj}}

        \begin{lemma}\label{lem:optimality:dist}
            Suppose \cref{asmp:disj} holds and $\phi_1(x)=x$ for all $x$.
            For any $W\in \R_{\geq 0}^{m\times n}$, $F\in \evF$ and the corresponding function $\tF$ (defined in \cref{eq:def_of_fT_proof_overview}),
            increasing functions $\phi_2,\phi_3,\dots,\phi_p$, $k\geq 1$, and protected groups $G_1,G_2,\dots,G_p$,
            the subset $\tS$ computed by \cref{alg:disj} is an optimal solution to
            \begin{align*}
                \max\limits_{S\subseteq G_1} &\quad  \tF(S),
                \yesnum\label{prog:prog_solved_by_s}\\
                \st, &\quad 
                    1\leq j\leq m,\quad \abs{S\cap C_j}\geq \min\sinbrace{\sqrt{k},\abs{C_j}},\\
                &\quad \abs{S}\leq k\cdot \frac{\abs{G_1}}{n}.
            \end{align*}
        \end{lemma}
        \begin{proof}
            In this proof, we only consider subsets of $G_1$.
            To simplify notation, we write $C_j$ to denote $C_j\cap G_1$ for each $1\leq j\leq m$.
            Under \cref{asmp:disj}, for each item $i$, there is a unique $1\leq j(i)\leq m$ such that $W_{ij}\neq 0$.
            Define $w_{i}\coloneqq W_{ij(i)}$ for each $1\leq i \leq n$.
            For each $j$, let $i(1,j), i(2,j), \dots, i(\abs{C_j}, j)$ be items in $C_j$ arranged in non-increasing order of $w_i$. %

            \paragraph{{Local optimality.}}
             A subset $S$ is said to be {\em locally-optimal} if it has the property that:
             For each $j$, if $S$ selects $r$ elements from $C_j$ then these are the $r$ elements with the highest utility in $C_j$, i.e., 
             $$S\cap C_j = \inbrace{i(r,j) \mid 1\leq r\leq \abs{S\cap C_j}}.$$

             \noindent Let $S^\star$ be an optimal solution of \cref{prog:prog_solved_by_s}. Without loss of generality, $S^\star$ is locally optimal.

             \paragraph{{Without loss of generality $S^\star$ is locally optimal.}}
                 If $S^\star$ is not locally-optimal, construct another set $S$ such that $\abs{S\cap C_j}=\abs{S^\star \cap C_j}$ and $S$ is locally-optimal.
                 Hence, $\sum\nolimits_{i\in S\cap C_j} w_i\geq \sum\nolimits_{i\in S^\star\cap C_j} w_i$  for each $1\leq j\leq m$.
                 Consequently, as $g_1,g_2,\dots,g_m$ are increasing, $F(S)\geq F(S^\star)$.
                 Because $\abs{S\cap C_j}=\abs{S^\star \cap C_j}$ for each $j$, it also follows that $S$ is feasible for \cref{prog:prog_solved_by_s}.
                 Since $S^\star$ is an optimal solution of \cref{prog:prog_solved_by_s}, it must hold that $F(S)=F(S^\star)$.
                 Hence, $S$ is optimal solution of \cref{prog:prog_solved_by_s} and locally optimal.
    
             One can verify that \cref{alg:disj} selects items in decreasing order of $w_i$ and, hence, $\tS$ is also locally-optimal.
             Suppose $F(\tS)< F(S^\star)$ and, hence, $\tS\neq S^\star$.
             Since $\tS$ and $S^\star$ are not equal and both $\tS$ and $S^\star$ are locally optimal hence, there exist two attributes $j(1)$ and $j(2)$ such that
             $\sabs{\tS\cap C_{j(1)}} > \sabs{S^\star\cap C_{j(1)}}$ and $\sabs{\tS\cap C_{j(2)}} < \sabs{S^\star\cap C_{j(2)}}$.
             Without loss of generality assume that $j(1)=1$ and $j(2)=2$.
             Let
             $$x\coloneqq \sabs{S^\star\cap C_1}\quad\text{and}\quad y\coloneqq \sabs{S^\star\cap C_2}.$$

             \noindent Since $S^\star$ is feasible for the \cref{prog:prog_solved_by_s}, $x,y\geq \sqrt{k}$.
             Consider the iteration (of the while loop in Step 7 of \cref{alg:disj}) where $\tS$ selects the $(x+1)$-th item from $C_1$, i.e., $i(x+1,1)$.
             (This iteration exists because $x>\sqrt{k}$ and, hence, $(x+1)$-th item is not selected in Step 5 of \cref{alg:disj} and $\sabs{\tS\cap C_1}>x$, so $(x+1)$-th item is selected in some iteration of Step 7 of \cref{alg:disj}.)
             Let $\bar{S}$ be the value of $\tS$ at this iteration and $z\coloneqq \sabs{\bar{S}\cap C_2}\leq \sabs{\tS\cap C_2}< y$.
             Since \cref{alg:disj}, selected $i^\star\coloneqq i(x+1,1)$ instead of $i(z+1,2)$, it must hold that
             \begin{align*}
                     \tF_{i(x+1,1)}(\bar{S})\geq F_{i(z,2)}(\bar{S}).
             \end{align*}
             Equivalently,
             \begin{align*}
                    &\smallmath{{g_1\inparen{w_{i(x+1,1)}+\sum_{h=1}^{x} w_{i(h,1)}}
                     - g_1\inparen{\sum_{h=1}^{x} w_{i(h,1)}}} \geq
                    g_2\inparen{w_{i(z+1,1)}+\sum_{h=1}^{z} w_{i(h,2)}}
                     - g_2\inparen{\sum_{h=1}^{z} w_{i(h,2)}}.}
                     \yesnum\label{eq:inequality_marg_util1}
             \end{align*}
             If the above holds with equality, then \cref{alg:disj} could have chosen $i(z+1,2)$ instead of $i(x+1,1)$ in this iteration.
             If \cref{eq:inequality_marg_util1} holds with equality, suppose \cref{alg:disj} selects $i(z+1,2)$ instead of $i(x+1,1)$ in this iteration, and find the next iteration where \cref{alg:disj} selects an item not in $S^\star$.
             (This will only repeat a finite number of times as part 2 of \cref{alg:disj} lasts for finite number of iterations.)
             If there is no such iteration, we have shown $F(\tS)=F(S^\star)$.

             It remains to consider the case where \cref{eq:inequality_marg_util1} holds with strict inequality, i.e.,
             \begin{align*}
                    &\smallmath{{g_1\inparen{w_{i(x+1,1)}+\sum_{h=1}^{x} w_{i(h,1)}}
                     - g_1\inparen{\sum_{h=1}^{x} w_{i(h,1)}}}
                     >
                    g_2\inparen{w_{i(z+1,1)}+\sum_{h=1}^{z} w_{i(h,2)}}
                     - g_2\inparen{\sum_{h=1}^{z} w_{i(h,2)}}.}
                     \yesnum\label{eq:inequality_marg_util2}
             \end{align*}
             Consider the subset $S'\coloneqq S^\star\cup  i(x+1,2)\backslash  i(y+1,1)$.
             We can lower bound $F(S')$ as follows
             \begin{align*}
                {F(S') - F(S^\star)}
                 &= \inparen{g_1\inparen{w_{i(x+1,1)}+\sum_{h=1}^{x} w_{i(h,1)}}
                     - g_1\inparen{\sum_{h=1}^{x} w_{i(h,1)}}}\\
                 &\quad -
                    \inparen{g_2\inparen{w_{i(y+1,1)}+\sum\nolimits_{h=1}^{y} w_{i(h,2)}}
                     - g_2\inparen{\sum_{h=1}^{y} w_{i(h,2)}}}\\
                &\geq \inparen{g_1\inparen{w_{i(x+1,1)}+\sum_{h=1}^{x} w_{i(h,1)}}
                     - g_1\inparen{\sum_{h=1}^{x} w_{i(h,1)}}}\\
                &\quad -
                    \inparen{g_2\inparen{w_{i(z+1,1)}+\sum_{h=1}^{z} w_{i(h,2)}}
                     - g_2\inparen{\sum_{h=1}^{z} w_{i(h,2)}}}
                     \tag{Using that $z\leq y$ and the $g_2$ is concave}\\
                &> 0.\tagnum{Using \cref{eq:inequality_marg_util2}}
                \customlabel{eq:dkljlk}{\theequation}
             \end{align*}
             Further, $S'$ is feasible for \cref{prog:prog_solved_by_s}.
             We only need to verify that $S'$ satisfies the constrains on the first two attributes.
             These hold because $\abs{S'\cap C_1}\geq \sabs{\tS\cap C_1}\geq L_1$ and $\abs{S'\cap C_2}\geq \abs{S^\star\cap C_2}\geq L_2$.
             Since $S^\star$ is optimal for \cref{prog:prog_solved_by_s}, $S'$ is feasible, and $F(S')\geq F(S^\star)$, it must be true that $F(S')\leq F(S^\star)$.
             But Equation~\eqref{eq:dkljlk} contradicts this.
             Hence, $F(\tS)=F(S^\star)$ and, since we already showed that $\tS$ is feasible for \cref{prog:prog_solved_by_s}, $\tS$ is optimal for \cref{prog:prog_solved_by_s} and the lemma follows.
        \end{proof}

        \noindent Next, \cref{lem:invariance} shows that, under some assumptions on a subset $T\subseteq [n]$ (\cref{eq:asmps_on_T}), $\tF(T)$ is a good approximation of $F(T)$. %

        \begin{lemma}\label{lem:invariance}
            For any subset $S$ of size $k$ (independent of $G_1$), numbers $x,\delta>0$, attribute $1\leq j\leq m$, and increasing and concave function $g_j\colon\R\to \R$, if
            \begin{align*}
                \sum\nolimits_{i \in S \cap C_j}W_{ij} \geq xk,
                \yesnum\label{eq:asmps_on_T}
            \end{align*}
            then for any $\delta \geq \exp\inparen{-\tau\gamma^6 x k }$, with probability at least $1-\delta$, it holds that
            \begin{align*}
               {\abs{
               g_j\inparen{ \frac{n}{\abs{G_1}}\cdot \sum\nolimits_{i \in S \cap G_1 \cap C_j}W_{ij}  } -
               g_j\inparen{  \sum\nolimits_{i\in S \cap C_j}W_{ij}  }
               }}
                & \leq
                \sqrt{\frac{1}{\tau\gamma^6 xk}\log{\frac{1}{\delta}}}\cdot g_j\inparen{  \sum\nolimits_{i\in S \cap C_j}W_{ij}  }.
            \end{align*}
            Where the probability is over the randomness in the choice of $G_1$.
        \end{lemma}
        \noindent Taking a union bound over $1\leq j\leq m$ and using the definition of $\tF$ (\cref{eq:def_of_fT_proof_overview}), it follows that:
        with probability at least $1-m\delta$ (for any $\delta \geq e^{-\tau\gamma^6 x k }$), it holds that
        \begin{align*}
            \abs{\tF(S)-F(S)}\leq \sqrt{\frac{1}{\tau\gamma^6 xk}\log{\frac{1}{\delta}}}\cdot F(S).
        \end{align*}

        \begin{proof}[Proof of \cref{lem:invariance}]
            Fix any $j\in [m]$.
            For each $i\in [n]$, let $Z_i$ be the discrete random variable that is $W_{ij}$ if $i\in G_1$ and $0$ otherwise.
            We have that $\Pr[Z_i] = \frac{\abs{G_1}}{n}$ and $\sum\nolimits_{i\in S\cap C_j} Z_i = \sum\nolimits_{i\in S\cap C_j\cap G_1} W_{ij}.$
            Using linearity of expectation, it follows that
            $$\mu\coloneqq \Ex\insquare{  \sum\nolimits_{i\in S\cap C_j} Z_i  } = \frac{\abs{G_1}}{n}\cdot \sum\nolimits_{i\in S\cap C_j} W_{ij}.$$
            Using a variant of the Hoeffding's inequality which holds for sampling without replacement \cite[Proposition 1.2]{bardenet2015concentration}, we get that for any $\delta>0$
            \begin{align*}
                \Pr\insquare{ \abs{\sum\nolimits_{i\in S\cap C_j} Z_i - \mu} >
                \sqrt{\frac{1}{2}\cdot \inparen{\sum\nolimits_{i\in S\cap C_j} W_{ij}^2 } \cdot\log{\frac{1}{\delta}}}
                } \leq \delta.
                \yesnum\label{eq:probElkjdlkj}
            \end{align*}
            Let $\evE$ be the event that
            \begin{align*}
                \abs{\sum\nolimits_{i\in S\cap C_j} Z_{i}- \mu}
                \leq  \sqrt{\frac{1}{2}\cdot \inparen{\sum\nolimits_{i\in S\cap C_j} W_{ij}^2 } \cdot\log{\frac{1}{\delta}}}.
            \end{align*}
            Since $0\leq W_{ij}\leq \tau^{-1}$, it follows that $W_{ij}^2\leq W_{ij}/\tau$.
            Substituting this in the above inequality, we get that, conditioned on $\evE$
            \begin{align*}
                \abs{\sum\nolimits_{i\in S\cap C_j} Z_{i}- \mu}
                \leq  \sqrt{\frac{1}{2\tau}\cdot \inparen{\sum\nolimits_{i\in S\cap C_j} W_{ij} } \cdot\log{\frac{1}{\delta}}}.
            \end{align*}
            Next, using that $\sum\nolimits_{i\in S\cap C_j}W_{ij}\geq xk$, we get that
            conditioned on $\evE$
            \begin{align*}
                \frac{\abs{\sum\nolimits_{i\in S\cap C_j} Z_{i}- \mu}}{ \mu }
                \leq \frac{n}{\abs{G_1}} \sqrt{\inparen{2\tau xk}^{-1}\cdot\log{\frac{1}{\delta}}}
                \leq \sqrt{\inparen{2\tau\gamma^2 xk}^{-1}\cdot\log{\frac{1}{\delta}}}.
                \tag{Using that $\abs{G_1}\geq \gamma n$}
            \end{align*}
            Since $g_j$ is an increasing function and $\frac{\abs{G_1}}{n}\geq\gamma$, the above equation implies: Conditioned on $\evE$
            \begin{align*}
                g_j\inparen{\frac{n\mu}{\abs{G_1}} - \sqrt{\frac{1}{2\tau \gamma^4 xk}\log{\frac{1}{\delta}}}} &\leq g_j\inparen{\frac{n}{\abs{G_1}} \sum\nolimits_{i\in S\cap C_j} Z_{i}}
                \leq g_j\inparen{\frac{n\mu}{\abs{G_1}} + \sqrt{\frac{1}{2\tau \gamma^4 xk}\log{\frac{1}{\delta}}}}.
                \yesnum\label{eq:aboveslklkldk}
            \end{align*}
            Let 
            \begin{align*}
                \alpha\coloneqq g_j'\inparen{\frac{n\mu}{\abs{G_1}}\inparen{1 - \sqrt{\frac{1}{2\tau \gamma^4 xk}\log{\frac{1}{\delta}}}} }.
                \yesnum\label{def:alpha}
            \end{align*}
            Since $g_j$ is concave $g'(y)\leq g_j'\inparen{z}$ for any $y\geq z$.
            Hence, from  \cref{eq:aboveslklkldk}, we get that conditioned on $\evE$
             \begin{align*}
                \abs{g_j\inparen{\frac{n}{\abs{G_1}} \sum\nolimits_{i\in S\cap C_j} Z_{i}} - g_j\inparen{\frac{n\mu}{\abs{G_1}}}}
                \leq \alpha \frac{n\mu}{\abs{G_1}} \sqrt{\frac{1}{2\tau \gamma^4 xk}\log{\frac{1}{\delta}}}.
                \yesnum\label{eq:kjlkjdf}
            \end{align*}
            Let
            \begin{align*}
                \zeta\coloneqq \sqrt{\frac{1}{2\tau \gamma^6 xk}\log{\frac{1}{\delta}}}.
                \yesnum\label{eq:def_zz}
            \end{align*}
            Observe that
            \begin{align*}
                g_j\inparen{\frac{n\mu}{\abs{G_1}}}
                &\geq g_j(\mu)\tag{Using that $n\geq \abs{G_1}$ and that $g_j$ is an increasing function}\\
                &= \int_{0}^{\mu} g'_j(z) dz\\
                &\geq \int_{0}^{\mu(1-\zeta)} g'_j(z) dz\tag{Using that $\zeta\geq 0$ and $g'(z)\geq 0$ for all $z\geq 0$}\\
                &\geq  g'_j(\mu(1-\zeta)) \cdot \mu(1-\zeta) \tag{Using that $g'_j$  is a decreasing function, as $g_j$ is concave}\\
                &= \alpha  \mu(1-\zeta) \tag{Using \cref{def:alpha,eq:def_zz}}\\
                &\geq \alpha  \gamma(1-\zeta)\frac{\mu n}{\abs{G_1}}.
            \end{align*}
            Substituting this in \cref{eq:kjlkjdf}, it follows that
            conditioned on $\evE$
            \begin{align*}
                \abs{g_j\inparen{\frac{n}{\abs{G_1}} \sum\nolimits_{i\in S\cap C_j} Z_{i}} - g_j\inparen{\frac{n\mu}{\abs{G_1}}}}
                \leq \frac{\zeta}{(1-\zeta)} g\inparen{\frac{n\mu }{\abs{G_1}}} 
                \leq 2\zeta g\inparen{\frac{n\mu }{\abs{G_1}}}. \tag{Using that $\zeta\leq \frac{1}{2}$ as $\delta\geq e^{-\tau\gamma^6 xk}$, $\mu\geq x$, and \cref{eq:def_zz}}
            \end{align*}
            Substituting the value of $\mu$ and $\zeta$, we get that conditioned on $\evE$
            \begin{align*}
                \abs{{g_j\inparen{\frac{n}{\abs{G_1}}\sum\nolimits_{i\in S\cap C_j} Z_{i}}} - {g_j\inparen{\sum\nolimits_{i\in S\cap C_j}W_{ij}   }}  }
                &\leq g_j\inparen{\sum\nolimits_{i\in S\cap C_j}W_{ij}}\cdot \sqrt{\frac{2}{\tau xk\gamma^6}\cdot\log{\frac{1}{\delta}}}.
            \end{align*}
            The lemma follows because $\Pr[\evE]\geq 1-\delta$ by \cref{eq:probElkjdlkj}.
        \end{proof}

        \noindent Next, \cref{lem:disj_main} lower bounds $\tF(\tS)$ by $\opt\cdot \inparen{1-O(k^{-\sfrac14})}$.

        \begin{lemma}\label{lem:disj_main}
            Suppose \cref{asmp:disj} holds and $\phi_1(x)=x$ for all $x$.
            Consider $\tS$ computed in Part 1 of \cref{alg:disj}.
            For any $\delta\geq m e^{-8\tau \gamma k}$, with probability $1-\delta$
            it holds that
            \begin{align*}
                \tF(\tS)\geq \opt{} \cdot \inparen{1-O\inparen{\sqrt{\frac{m^4}{\tau^2\gamma^5k^{1/2}} \log{\frac{m}{\delta}}}  }}.
            \end{align*}
            Where the probability is over the randomness in the choice of protected groups $G_1,G_2,\dots,G_p$.
        \end{lemma}
        \noindent Since $\tS$ is an optimal solution for \cref{prog:prog_solved_by_s}, it holds that for any other subset $T$ feasible for \cref{prog:prog_solved_by_s}, $\tF(\tS)\geq \tF(T)$.
        In particular, if $\sopt$ is feasible for \cref{prog:prog_solved_by_s}, then $\tF(\tS)\geq \tF(\sopt)$.
        If $\sopt$ satisfies the condition in \cref{eq:asmps_on_T} for any constant $x>0$, the from \cref{lem:invariance} we would get that $\tF(\sopt)\geq \opt\inparen{1-O(k^{-\sfrac{3}{2}})}$.
        Then the proof would follow from chaining the last two inequalities.
        However, $\sopt$ may not be feasible for \cref{prog:prog_solved_by_s} or satisfy \cref{eq:asmps_on_T} and, hence, we cannot use \cref{lem:invariance}.

        Instead, we construct a subset $\ovsopt$ from $\opt$ and show that $\ovsopt$ satisfies \cref{eq:asmps_on_T} for $x=\Omega(k^{-\sfrac{1}{2}})$ and $F(\ovsopt)\geq F(\opt).$
        We (for our analysis) construct a subset $\ovsopt$ from $\opt$ as follows:

        \medskip

        \begin{tcolorbox}[bottom=0.01cm,top=0.01cm,left=0.05cm,right=0.05cm]
            {\bf Algorithm to construct $\ovsopt$ (which is used in the proof of \cref{lem:disj_main})}
            \begin{itemize}[leftmargin=\leftmarginCUSTOM]
                \item {\bf Input:} $\sopt$, latent utilities $W$, and $\delta>0$
                \item {\bf Output:} $\ovsopt$
            \end{itemize}
            \hrule
            \begin{enumerate}[leftmargin=\leftmarginCUSTOM]
                \item Initialize $\ovsopt=\sopt$
                \item {\bf For} $j\in [m]$ {\bf do}
                \begin{itemize}
                    \item {\bf If} $\abs{\ovsopt\cap C_j}\leq 2 \gamma^{-1}k^{\sfrac12}$ {\bf then}
                    \begin{itemize}
                        \item Add any $2\gamma^{-1} k^{\sfrac12}$ elements from $C_j\backslash \ovsopt$ to $\ovsopt$
                    \end{itemize}
                \end{itemize}
                \item Set $j^\star\coloneqq \argmax_j\abs{\ovsopt\cap C_j}$\hspace{5mm} {\small $\triangleright$ By Pigeon hole  $\abs{\ovsopt\cap C_{j^\star}}\geq \frac{k}{m}\geq \Omega((m+1)k^{\sfrac{1}{2}})$}
                \item Remove $\abs{\ovsopt}-k$ elements from $\ovsopt\cap C_{j^\star}$ with the smallest latent utilities
                \hspace{20mm}\white{.} \white{.}\hfill {\small $\triangleright$ Here, $0\leq \sabs{\ovsopt}-k\leq m(2\gamma^{-1} k^{\sfrac12})$}
                \item {\bf return} $\ovsopt$
            \end{enumerate}
        \end{tcolorbox}

        \begin{proof}[Proof of \cref{lem:disj_main}]
            Set $k_0$ to satisfy the following inequality
            \begin{align*}
                k_0\geq 
                \max\inbrace{2\gamma^{-1}m(m+1), \inparen{\gamma\log\frac{m}{\delta}}^3, 16\gamma^{-2}m^4}.
                \yesnum\label{eq:lb_on_k0}
            \end{align*}
            Fix any $1\leq j\leq m$.
            $\ovsopt$ is a function of just $\sopt$, $W$, and $\delta$, and, $\sopt$ itself is just  a function of $W$. h
            Hence, $\ovsopt$ is independent of $G_1$.
            Repeating the concentration argument in \cref{lem:invariance}, but with $Z_i\coloneqq \mathbb{I}[i\in G_1]$, we get that for any $\delta>0$
            \begin{align*}
                \Pr\insquare{ \abs{\abs{\ovsopt \cap C_j\cap G_1} - \frac{\abs{G_1}}{n}\cdot \abs{\ovsopt \cap C_j}} > \sqrt{\frac{1}{2}\cdot {k^{\sfrac12}}\cdot \log{\frac{m}{\delta}}} } \leq \frac{\delta}{m}.
            \end{align*}
            Consequently,
            \begin{align*}
                &\quad\Pr\insquare{ {\abs{\ovsopt \cap C_j\cap G_1} } > \frac{\abs{G_1}}{n}\cdot \abs{\ovsopt \cap C_j} - \sqrt{\frac{1}{2}\cdot k^{\sfrac12}\cdot \log{\frac{m}{\delta}}} } \leq \frac{\delta}{m}.
            \end{align*}
            Moreover, since $\abs{\ovsopt \cap C_j}\geq \frac{2\sqrt{k}}{\gamma}$ and $\abs{G_1}\geq \gamma n$,  
            \[
                \Pr\insquare{ {\abs{\ovsopt \cap C_j\cap G_1} } > 2k^{\sfrac12} - \sqrt{\frac{1}{2}\cdot k^{\sfrac12}\cdot \log{\frac{m}{\delta}}} } \leq \frac{\delta}{m}.
            \]
            Since $k\geq k_0$  and $k_0$ satisfies \cref{eq:lb_on_k0},
            the above inequality implies
            \begin{align*}
                \Pr\insquare{ {\abs{\ovsopt \cap C_j\cap G_1} } > k^{\sfrac12}} \leq \frac{\delta}{m}.
                \yesnum\label{eq:event1lkjlksj}
            \end{align*}
            From the Hoeffding's bound, we also have that
            \begin{align*}
                \Pr\insquare{\abs{\ovsopt\cap G_1}\leq k\frac{\abs{G_1}}{n}+\sqrt{\frac{1}{2}\cdot k\frac{\abs{G_1}}{n}\cdot \log{\frac{m}{\delta}}}}\leq \frac{\delta}{m}.
                \yesnum\label{eq:event2lkjlksj}
            \end{align*}
            Let $\evE$ be the event that
            \begin{align*}
                \abs{\ovsopt\cap G_1}\leq k\frac{\abs{G_1}}{n}+\sqrt{k\log{\frac{m}{\delta}}}
                \quad\text{and}\quad
                \abs{\ovsopt \cap C_j\cap G_1} > k^{\sfrac12}.
            \end{align*}
            From \cref{eq:event1lkjlksj,eq:event2lkjlksj},
            $\Pr[\evE]\geq 1-2\delta$.
            Therefore, conditioned on $\evE$, $\ovsopt\cap G_1$ is feasible for \cref{prog:prog_solved_by_s}.
            Since $\tS$ is an optimal solution for \cref{prog:prog_solved_by_s},
            it follows that
            conditioned on $\evE$:
            \begin{align*}
                \tF(\tS)\geq \tF(\ovsopt).
                \yesnum\label{ovsopt_1}
            \end{align*}
            Further, as for each $j$, $\abs{\ovsopt \cap C_j}\geq 2\gamma^{-1}\sqrt{k}$ and $W_{ij}\geq \tau$, it follows that $\sum\nolimits_{i\in \ovsopt\cap C_j} W_{ij}\geq 2\tau\sqrt{k}$.
            Hence, $\ovsopt$ satisfies \cref{eq:asmps_on_T} with $x\geq 2\tau\gamma^{-1}  (k)^{-\sfrac{1}{2}}$.
            Using \cref{lem:invariance} (as $\ovsopt$ is independent of $G_1$), it follows that with probability at least $1-\delta$ (for any $\delta\geq m e^{-\tau^2 \gamma^{5}\sqrt{k}}$)
            \begin{align*}
                \tF(\ovsopt)\geq F(\ovsopt)\inparen{1 - \sqrt{\frac{1}{2\tau^2\gamma^5 \sqrt{k}}\log{\frac{1}{\delta}}}}.
                \yesnum\label{ovsopt_2}
            \end{align*}
            Finally, we lower bound $F(\ovsopt)$ by a multiple of $F(\opt)$.
            First observe that for all $j\neq j^\star$, $\ovsopt\cap C_j\supseteq \opt\cap C_j$ and, hence,
            \begin{align*}
                \forall j\neq j^\star,\quad
                g_j\inparen{ \sum\nolimits_{i\in \ovsopt} W_{ij}}
                \geq g_j\inparen{ \sum\nolimits_{i\in \sopt} W_{ij}}.
                \yesnum\label{eq:lb_star2}
            \end{align*}
            Further, by construction of $\ovsopt{}$, $(\ovsopt\cap C_{j^\star})\subseteq (\sopt\cap C_{j^\star})$ and $(\sopt\cap C_{j^\star})\backslash (\ovsopt\cap C_{j^\star})$
            consists of $\alpha$ items with the smallest latent utility in $\sopt\cap C_{j^\star}$ where $\alpha\leq 2m\gamma^{-1}k^{\sfrac12}$.
            Let $\Delta\coloneqq \abs{\sopt\cap C_{j^\star}}$.
            Let $i(1), i(2),\dots,i(\Delta)$ be items in $\sopt\cap C_{j^\star}$ such that
            $W_{i(1)j^\star}\geq W_{i(2)j^\star}\geq \cdots $.
            We have that
            \begin{align*}
                \frac{\sum\nolimits_{i\in \ovsopt} W_{ij^\star}}{\sum\nolimits_{i\in \sopt} W_{ij^\star}}
                &=\frac{\sum\nolimits_{i\in \ovsopt\cap C_{j^\star}} W_{ij^\star}}{\sum\nolimits_{i\in \sopt\cap C_{j^\star}} W_{ij^\star}}
                \tag{Using \cref{asmp:disj}}\\
                &= \frac{\sum\nolimits_{h=1}^{\sabs{\ovsopt\cap C_{j^\star}}} W_{i(h)j^\star}}{\sum\nolimits_{h=1}^{\Delta} W_{i(h)j^\star} }
                \tag{Using $\sopt\cap C_{j^\star}=\inbrace{i(1),\dots,i(\Delta)}$}.
                \intertext{By construction, $\ovsopt$ drops at most $\alpha$ items from $\opt$ and if it drop $x$ items, then these are the $x$ items with the smallest latent utilities in $\sopt\cap C_{j^\star}$.
                Consequently}
                \frac{\sum\nolimits_{i\in \ovsopt} W_{ij^\star}}{\sum\nolimits_{i\in \sopt} W_{ij^\star}} &\geq \frac{\sum\nolimits_{h=1}^{\Delta-\alpha} W_{i(h)j^\star}}{\sum\nolimits_{h=1}^{\Delta} W_{i(h)j^\star} }\\
                &\geq \inparen{1 + \frac{\alpha W_{i(\Delta-\alpha)j^\star}}{\sum\nolimits_{h=1}^{\Delta-\alpha} W_{i(h)j^\star} } }^{-1}\\
                &\geq \inparen{1+\frac{\alpha}{\Delta-\alpha}}^{-1}\tag{Using that $W_{i(\Delta-\alpha)j^\star}\leq W_{i(h)j^\star}$ for all $\alpha\leq j\leq \Delta-\alpha$}
            \end{align*}
            \begin{align*}
                &\geq \inparen{1+\frac{2m\gamma^{-1}\sqrt{k}}{\frac{k}{2m} - 2m\gamma^{-1}\sqrt{k}}}^{-1} \tag{Using that $\alpha\leq 2m\gamma^{-1}k^{\sfrac12}$ and by Pigeon hole principal $\Delta\geq \frac{k}{m}$}\\
                &= {1-\frac{4m^2}{\gamma\sqrt{k}}}.\yesnum\label{eq:lb_on_sum_of_util} %
            \end{align*}
            Since $g_{j^\star}$ is concave and increasing, we have that
            \begin{align*}
                g_{j^\star}\inparen{\sum\nolimits_{i\in \ovsopt} W_{ij^\star}}
                \ \ &\Stackrel{\eqref{eq:lb_on_sum_of_util}}{\geq}\ \ 
                g_{j^\star}\inparen{\inparen{1-4\gamma^{-1}m^2k^{-\sfrac12}}\cdot \sum\nolimits_{i\in \sopt} W_{ij^\star}}\\
                &\geq \inparen{1-4\gamma^{-1}m^2k^{-\sfrac12}}\cdot  g_{j^\star}\inparen{ \sum\nolimits_{i\in \ovsopt} W_{ij^\star}}.
                \tagnum{Using that $g_{j^\star}$ is concave and $g(0)\geq 0$}
                \customlabel{eq:lb_star}{\theequation}
            \end{align*}
            Hence, combining Equations~\eqref{eq:lb_star2} and \eqref{eq:lb_star} it follows that
            \begin{align*}
                F(\ovsopt)\geq F(\sopt) \cdot \inparen{1-4\gamma^{-1}m^2k^{-\sfrac12}}.
                \yesnum\label{eq:nenwklj}
            \end{align*}
            We get the required result by chaining Equations~\eqref{ovsopt_1}, \eqref{ovsopt_2}, and \eqref{eq:nenwklj} and taking the union bound over events in \cref{ovsopt_1,ovsopt_2}.
        \end{proof}
        \noindent 
        If $\tS$ is independent of $G_1$, then \cref{lem:invariance,lem:disj_main} show that with high probability there is a subset $S_E$ such that $\tS=S_E\cap G_1$ and $F(S_E)$ is at least $\opt\cdot \inparen{1-O(k^{-\sfrac14})}.$
        However, $\tS$ is not independent of $G_1$. %
        \cref{lem:high_utility_solution} addresses this.

        \begin{lemma}\label{lem:high_utility_solution}
            Suppose $C_1,\dots, C_m$ are disjoint and $\phi_1(x)=x$ for all $x$.
            Let $\tS$ be as constructed in Part 1 of \cref{alg:disj} and $k_1,\dots,k_j$ be as defined in Step 11 of \cref{alg:disj}.
            For any $\delta>0$, with probability at least $1-\delta$, there exists a subset $S_E\subseteq [n]$ satisfying %
            \begin{align*}
                \forall j\in [m],\quad \abs{S_E\cap C_j}\leq k_j
            \quad\text{and}\quad
                F(S_E)\geq \opt{} \cdot \inparen{1-O\inparen{\sqrt{\frac{m^4}{\tau^2\gamma^5k^{1/2}} \log{\frac{m}{\delta}}}  }}.
                \yesnum\label{eq:claim_in_lem:high_utility_solution}
            \end{align*}
        \end{lemma}
        \begin{proof}
            The proof relies on the fact that $\tS\subseteq G_1$ is locally optimal.
            Recall that $S\subseteq G_1$ is said to be locally optimal if for each $j$, $S\cap C_j$ is the set of $\abs{S\cap C_j}$ items with the highest latent utility in $C_j\cap G_1$. %
            Extend the definition of local optimality to sets which are not necessarily a subset of $G_1$:
            A subset $S\subseteq[n]$ is said to be locally optimal if for each $j\in [m]$,
            $S\cap C_j$ contains $\abs{S\cap C_j}$ items with the highest latent utility in $C_j$ (instead of $C_j\cap G_1$).
            A locally optimal subset $S$ is uniquely defined by the following values
            \begin{align*}
                x_{ j,S}\coloneqq \abs{S\cap C_j}\quad \text{for each } 1\leq j\leq m.
            \end{align*} 
            Let $\cS$ be the set of all locally optimal subsets $S$ of size at most $k$ satisfying $x_{j, S}\geq 2\gamma^{-1}\sqrt{k}$ for each $1\leq j\leq m$.
            Since each $S\in \cS$ is uniquely identified by $\inbrace{x_{j,S}}_{j}$ and there are $k+1$ choices for each $x_{j,S}$, it follows that $\abs{\cS}\leq (k+1)^{m}$. %
            Since $W_{ij}>\tau$ (for each $i$ and $j$),  
            \begin{align*}
                \forall S\in \cS,\quad 
                \sum\nolimits_{i\in S\cap C_j}\geq 2\tau\gamma^{-1}\sqrt{k}.
            \end{align*}
            Hence, each $S\in \cS$ satisfies \cref{eq:asmps_on_T} with $x=2\tau\gamma^{-1}k^{-1/2}$.
            $\cS$ is deterministically given $W$ and, hence, is independent of $G_1$.
            Consequently, \cref{lem:invariance} is applicable for any $S\in \cS$.
            Applying \cref{lem:invariance} for each $S\in \cS$ and taking the union bound, implies:
                For any $\delta \geq m\cdot (k+1)^{m}\cdot e^{-\tau^2\gamma^5 \sqrt{k}}$,
                with probability at least $1-\delta$,
                for all $T\in \cS$
            \begin{align*}
                \abs{\tF(T)-F(T)}\leq \sqrt{\frac{1}{2\tau^2\gamma^5 k^{1/2}}\log{\frac{m \cdot (k+1)^{m}}{\delta}}} \cdot F(T).
            \end{align*}
            In other words, for any $\delta\geq e^{-\Omega(\tau^2\gamma^5 \sqrt{k})}$, it holds that with probability at least $1-\delta$, for all $T\in \cS$
            \begin{align*}
                \abs{\tF(T)-F(T)}\leq \sqrt{\frac{m}{\tau^2\gamma^5 k^{1/2}}\log{\frac{mk}{\delta}}} \cdot F(T).
            \end{align*}

            \newcommand{\negspaceif}{\ifconf\else\negsp\fi}

            \noindent Select the subset $T\in \cS$ such that for each $1\leq j\leq m$
            $$\abs{T\cap C_j} = x_{j,T} \geq  \min\inbrace{\frac{n}{\abs{G_1}}\cdot \sabs{\tS\cap C_j},  \abs{C_j}}.$$
            Such a subset exists because $\sum\nolimits_{j=1}^m \frac{n}{\abs{G_1}}\cdot \sabs{\tS\cap C_j}=\frac{n}{\abs{G_1}}\sabs{\tS} = k$.
            Next, divide $1\leq j\leq m$ into two cases.

            \paragraph{Case A ($\sabs{T\cap C_j}=\sabs{C_j}$):}
            It must be true that $T\cap C_j=C_j$ and, hence, $(\tS\cap C_j)\subseteq (T\cap C_j)$.
            As $\tS\subseteq G_1$, this implies that $(\tS\cap C_j)\subseteq (T\cap C_j\cap G_1)$.
            Consequently 
            \begin{align*}
                {g_j\inparen{\frac{n}{\abs{G_1}} \sum\nolimits_{i\in T\cap C_j\cap G_1} W_{ij}}}
                \geq g_j\inparen{\frac{n}{\abs{G_1}}  \sum\nolimits_{i\in \tS\cap C_j} W_{ij}}.
                \yesnum\label{eq:for_other_jkk}
            \end{align*}

            \paragraph{Case B ($\sabs{T\cap C_j}\geq \frac{n}{\sabs{G_1}}\cdot \sabs{\tS\cap C_j}$):}
            In this case, we will $(T\cap C_j)$ contains at least $1-O(k^{-\sfrac14})$ fraction of the items in $(\tS\cap C_j)$.
            To see this note that
            $$\Ex\nolimits_{G_1}[\abs{T\cap C_j\cap G_1}] = \frac{\abs{G_1}}{n}\abs{T\cap C_j}\geq \sabs{\tS\cap C_j}.$$
            Applying the Chernoff bound to the following indicator random variables $\inbrace{\mathds{I}[i\in G_1]: i\in T\cap C_j},$ we get that
            for any $\delta>0$
            \begin{align*}
                \Pr\insquare{
                \abs{T\cap C_j\cap G_1}\geq \sabs{\tS\cap C_j} - \sqrt{3\cdot \sabs{\tS\cap C_j}\cdot \log{\frac{m}{\delta}}}
                }\leq \frac{\delta}{m}.
            \end{align*}
            Because $\sabs{\tS\cap C_j}\geq \sqrt{k}$, this implies that
            \begin{align*}
                \Pr\insquare{
                \abs{T\cap C_j\cap G_1}\geq \sabs{\tS\cap C_j} \inparen{1- \sqrt{3\cdot k^{-\sfrac{1}{2}}\cdot \log{\frac{m}{\delta}}}}
                }\leq \frac{\delta}{m}.
                \yesnum\label{eq:cond_on_elkdj}
            \end{align*}
            Let the event in the above equation be $\evE$.
            Let $\Delta\coloneqq \abs{C_j\cap G_1}$, $\phi\coloneqq \abs{\tS\cap C_j}$, and let $i(1),i(2),\dots,i(\Delta)$ be the elements of $C_j\cap G_1$ ordered in decreasing order of latent utility, i.e.,
            $W_{i(1) j}\geq W_{i(2) j}\geq \dots\geq W_{i(\Delta) j}$.
            Then we have, conditioned on $\evE$
            \begin{align*}
                \ifconf\else\white{.}\hspace{-4mm}\fi
                \frac{\sum\nolimits_{i\in T\cap C_j\cap G_1} W_{ij}}{\sum\nolimits_{i\in \tS\cap C_j} W_{ij}}
                &= \frac{\sum\nolimits_{h=1}^{\sabs{T\cap C_j\cap G_1}}  W_{i(h)j}}{\sum\nolimits_{h=1}^{\phi} W_{i(h)j}}
                \tag{Using that $\tS,T$ are locally optimal and $\phi\coloneqq \sabs{\tS\cap C_j}$}\\
                &\geq \frac{\sum\nolimits_{h=1}^{
                    \phi\cdot \inparen{1-O\inparen{k^{-\sfrac{1}{4}}\sqrt{\log{m/\delta}}}}} W_{i(h)j}
                }
                { \sum\nolimits_{h=1}^{\phi}     W_{i(h)j}
                }\tag{Using that $\evE$ is the event in \cref{eq:cond_on_elkdj}}
            \end{align*}
            \begin{align*}
                &= \inparen{1+ \frac{\sum\nolimits_{h=\phi\cdot \inparen{1-O\inparen{k^{-\sfrac{1}{4}}\sqrt{\log{m/\delta}}}}+1}^{\phi} W_{i(h)j}} { \sum\nolimits_{h=1}^{
                    \phi\cdot \inparen{1-O\inparen{k^{-\sfrac{1}{4}}\sqrt{\log{m/\delta}}}}} W_{i(h)j} } }^{-1}\\
                &\geq
                \ifconf\else\negsp\negsp \fi
                \inparen{1+
                \frac{O\inparen{k^{-\sfrac{1}{4}}\sqrt{\log{m/\delta}}}}{
                \negspaceif
                1
                \negspaceif - \negspaceif O\inparen{k^{-\sfrac{1}{4}}\sqrt{\log{m/\delta}}}}
                }^{-1}
                \ifconf\else\hspace{-6mm}\fi
                \tag{Using $W_{i(1)j}\negspaceif \geq\negspaceif  W_{i(2)j}\negspaceif\geq\negspaceif \dots\negspaceif\geq\negspaceif W_{i(\Delta)j}$ and $W_{ij}\negspaceif\geq\negspaceif 0$}\\
                &\geq 1 - O\inparen{k^{-\sfrac{1}{4}}\sqrt{\log{m/\delta}}}.
            \end{align*}
            Consequently, conditioned on $\evE$,
            \begin{align*}
                g_j\inparen{\frac{n}{\abs{G_1}} \sum\nolimits_{i\in T\cap C_j\cap G_1} W_{ij}}
                \geq g_j\inparen{\frac{n}{\abs{G_1}}\cdot \inparen{1 - O\inparen{k^{-\sfrac{1}{4}}\sqrt{\log{m/\delta}}}}\cdot \sum\nolimits_{i\in \tS\cap C_j} W_{ij}}\tag{$g_j$ is increasing}\\
                \geq \inparen{1 - O\inparen{k^{-\sfrac{1}{4}}\sqrt{\log{m/\delta}}}}\cdot g_j\inparen{ \frac{n}{\abs{G_1}}\cdot \sum\nolimits_{i\in \tS\cap C_j} W_{ij}}. \tagnum{Using that $g_j$ is concave and $g(0)\geq 0$}
                \customlabel{eq:bound_for_onejlkj}{\theequation}
            \end{align*}
            Taking the union bound over all $j$ where $\abs{T\cap C_j}\neq\abs{C_j}$, we get that
            with probability at least $1-\delta$, for all $1\leq j \leq m$ Equation~\eqref{eq:bound_for_onejlkj} holds.
            Combining this with \cref{eq:for_other_jkk}, we get that with probability at least $1-\delta$,
            \begin{align*}
                &\sum_{j=1}^m g_j\inparen{\frac{n}{\abs{G_1}} \sum_{i\in T\cap C_j\cap G_1} W_{ij}}
                \geq \inparen{1 - O\inparen{k^{-\sfrac{1}{4}}\sqrt{\log{m/\delta}}}} \cdot\sum_{j=1}^m  g_j\inparen{ \frac{n}{\abs{G_1}}\cdot \sum_{i\in \tS\cap C_j} W_{ij}}.
                \yesnum\label{eq:case_2_lb}
            \end{align*} 
            \medskip 
            
            \noindent Combining, Cases A and B and using the definition of $\tF$ (\cref{{eq:def_of_fT_proof_overview}}), it follows that with probability at least $1-\delta$,
            \begin{align*}
                \tF(T)\geq
                \inparen{1 - O\inparen{k^{-\sfrac{1}{4}}\sqrt{\log{m/\delta}}}}\cdot \tF(\tS).
            \end{align*} 
            Chaining the above inequality with the lower bound on $\tF(\tS)$ from \cref{lem:disj_main}, \cref{lem:high_utility_solution} follows.
        \end{proof}

        \begin{lemma}\label{lem:pos_result_linear}
            With probability at least $1-\delta$, it holds that for all $j\in [m]$, the subset $S_j$ computed in \cref{alg:disj} satisfies
            \begin{align*}
                g_j(S_j)
                \geq
                \inparen{1-O\inparen{\gamma^{-1}k^{-\sfrac{1}{2}}\log{\frac{1}{\delta}}}}\cdot \max_{T\subseteq C_j:\abs{T}\leq k_j}g_j\inparen{\sum\nolimits_{i\in T}W_{ij}}.
            \end{align*}
            The algorithm runs in time $O(\abs{C_j}\log{\abs{C_j}})$.
        \end{lemma}
        \noindent We present the proof of \cref{lem:pos_result_linear} in \cref{sec:proofof:lem:pos_result_linear}.

        \subsubsection{Proof of \cref{thm:disj} using \cref{lem:high_utility_solution,lem:pos_result_linear}}

        \begin{proof} The proof is divided into two parts.
        First, we use  \cref{asmp:disj} to divide the maximization problem into $m$ parts and then use  \cref{lem:high_utility_solution} to complete the proof.

            \paragraph{(Consequence of \cref{asmp:disj}).}
                Since \cref{asmp:disj} holds, we have that
                \ifconf
                \begin{align*}
                    F(S) = \sum\nolimits_{j=1}^m  F(S\cap C_j).
                \end{align*}
                \else
                    $F(S) = \sum\nolimits_{j=1}^m  F(S\cap C_j).$
                \fi
                Hence,
                \begin{align*}
                    \max_{T\subseteq [n]:\forall j,\ \abs{T\cap C_j}\leq k_j} F(T)
                    &\ \ =\ \  \max_{T\subseteq [n]:\forall j,\ \abs{T\cap C_j}\leq k_j} \sum\nolimits_{j=1}^m  F(T\cap C_j)\\
                    &\ \ =\ \  \max_{T_1\subseteq C_1, \dots, T_m\subseteq C_m: \forall j,\ \abs{T_j}\leq k_j} \sum\nolimits_{j=1}^m F(T_j)\\
                    &\ \ =\ \  \sum\nolimits_{j=1}^m \max_{T_j\subseteq C_j: \abs{T_j}\leq k_j} F(T_j)\\
                    &\ \ =\ \  \sum\nolimits_{j=1}^m \max_{T_j\subseteq C_j: \abs{T_j}\leq k_j} g_j(T_j).
                    \tagnum{Using that under \cref{asmp:disj} for any $h, j\in [m]$, with $h\neq j$ and set $A\subseteq C_j$, $g_h(A)=0$}
                    \customlabel{eq:proof_thmm_24_2}{\theequation}
                \end{align*}
                
            \paragraph{(Consequence of \cref{lem:high_utility_solution}).}
                Let $\evE$ be the event that a subset $S_E$ satisfying \cref{eq:claim_in_lem:high_utility_solution} exists.
                Conditioned on $\evE$, it holds that
                \begin{align*}
                    \max_{S\subseteq [n]:\forall j,\ \abs{S\cap C_j}\leq k_j} F(S)
                    \geq \opt{} \cdot \inparen{1-O\inparen{\sqrt{\frac{m^4}{\tau^2\gamma^5k^{1/2}} \log{\frac{m}{\delta}}}  }}.
                    \yesnum\label{eq:proof_thmm_24_1}
                \end{align*}

                \bigskip 
                \bigskip

            \noindent Combining Equations~\eqref{eq:proof_thmm_24_2} and \eqref{eq:proof_thmm_24_1}, we get that
            conditioned on $\evE$, the following holds
            \begin{align*}
                \sum\nolimits_{j=1}^m \max_{T_j\subseteq C_j: \abs{T_j}\leq k_j} F(T_j)
                &\geq
                \opt{} \inparen{1-O\inparen{\sqrt{\frac{m^4}{\tau^2\gamma^5k^{1/2}} \log{\frac{m}{\delta}}}  }}.
                \yesnum\label{eq:proof_thmm_24_3}
            \end{align*}
            Let $\evF$ be the event that
            \begin{align*}
                \forall j\in [m],\quad
                g_j(S_j)\geq \inparen{1-O\inparen{\sqrt{\frac{1}{\gamma^2 k}} \log{\frac{1}{\delta}}}}\cdot  \max_{T_j\subseteq C_j: \abs{T_j}\leq k_j} g_j(T_j).
                \yesnum\label{eq:proof_thmm_24_4}
            \end{align*}
            Summing \cref{eq:proof_thmm_24_4} over all $j\in [m]$ and chaining it with \cref{eq:proof_thmm_24_3}, we get the following:
            Conditioned on $\evE$ and $\evF$
            it holds that
            \begin{align*}
                F(S) = \sum\nolimits_{j=1}^m  g_j(S_j)
                &\geq
                \opt{}\cdot  \inparen{1-O\inparen{\sqrt{\frac{m^4}{\tau^2\gamma^5k^{1/2}}  }\cdot \log{\frac{m}{\delta}}} }.
                \yesnum\label{eq:lb_bound_final}
            \end{align*}
            \cref{lem:high_utility_solution} implies that $\Pr[\evE]\geq 1-\delta$ and \cref{lem:optimality:dist} implies that $\Pr[\evF]\geq 1-\delta.$
            Hence, $\Pr[\evE, \evF]\geq 1-\Pr[\lnot \evE]-\Pr[\lnot \evF]\geq 1-2\delta$.
            Since, $k\geq k_0$ the result follows by choosing a small enough $k_0$ such that \cref{eq:lb_bound_final} implies that $F(S)\geq (1-\eps)\cdot \opt{}$ and $2\delta\leq \eps.$
        \end{proof}

    \subsubsection{Proof of \cref{lem:pos_result_linear}} %
    \label{sec:proofof:lem:pos_result_linear}

            \begin{algorithm}[t]
                \caption{Algorithm from \cref{lem:pos_result_linear} }\label{alg:greedy_alg_for_pos_result_1}
                \begin{algorithmic}[1]
                    \Require
                    Observed utilities $\hW\in \R_{\geq 0}^{n\times m}$, a number $k$, and protected groups $G_1,G_2,\dots,G_p$
                    \Ensure  A subset $S$ with $\abs{S}\leq k$
                    \State Initialize $S=\emptyset$
                    \For{groups $t\in [p]$}
                        \State Initialize $T\coloneqq G_t$
                        \For{iterations $r\in [\frac{k}{p}]$}
                            \State Let $i$ be the item that maximizes $\hW_{i1}$ among all items in $T$
                            \State Set $T=T\backslash\inbrace{i}$
                            \State Set $S=S\cup\inbrace{i}$
                        \EndFor
                    \EndFor
                    \State \Return $S$
                \end{algorithmic}
            \end{algorithm}

    \begin{proof}
            This proof only considers the $j$-th attribute and items in $C_j$.
            To simplify the notation, for each $i\in C_j$, let $w_i\coloneqq W_{ij}$ and $\hw_i\coloneqq \hW_{ij}$.
            Further, define
            \begin{align}
                H(S) = g\inparen{\sum\nolimits_{i\in S}w_i}
                \quad \text{and}\quad
                \hH(S) = g\inparen{\sum\nolimits_{i\in S}\hw_i}.
                \yesnum\label{eq:lingF_tmp}
            \end{align}
            Let $S^\star\subseteq C_j$ be the subset that maximizes $H(S)$ subject to having size at most $k_j$.
            Let $S_j\subseteq C_j$ be the subset that maximizes $H(S)$ subject to having size at most $k_j$ and satisfying the proportional representation constraints (i.e., the constraints encoded by $\forall t\in [p]$, $U_t=\abs{G_t\cap C_j}\cdot \frac{k}{n}$).
            Note that this $S_j$ is the same as $S_j$ computed in \cref{alg:disj}.
            We will prove that with high probability
            $$\sum\nolimits_{i\in \scons{}} w_{i}\geq (1-\eps)\cdot \sum\nolimits_{i\in S^\star} w_{i}.$$
            \cref{lem:pos_result_linear} follows from the above because $0\leq \eps < 1$ and $x>0$,
            $$g((1-\eps) \cdot x)\geq (1-\eps)\cdot  g(x)+\eps\cdot g(0)=(1-\eps)\cdot  g(x),$$
            as $g$ is concave and $g(0)=0$.
            In the remainder of the proof, we set
            \begin{align}
                H(S) = {\sum\nolimits_{i\in S}w_i}
                \quad \text{and}\quad
                \hH(S) = {\sum\nolimits_{i\in S}\hw_i}.
                \yesnum\label{eq:lingF}
            \end{align}
            Suppose that all values in $\inbrace{w_i:i\in C_j}$ are unique.
            This can be ensured by perturbing the values by an infinitesimal amount.
            It only changes the value of $H(S)$ by an infinitesimal amount.

            \smallskip

            \paragraph{Claim A.} We claim that for each $\ell\in [p]$, with probability at least $1-\delta$, it holds that
            \begin{align*}
                \frac{H(\scons{}\cap G_\ell)}{H(S^\star\cap G_\ell)}\geq 1-\eps.
                \yesnum\label{eq:cases_claim}
            \end{align*}
            \paragraph{Proof of \cref{lem:pos_result_linear} assuming Claim A is true.}
            If this claim is true, then the result follows:
            Using the union bound over all $\ell\in [p]$, we get that with probability at least $1-\delta$, \cref{eq:cases_claim} holds for all $\ell\in [p]$.
            Conditioned on the event that \cref{eq:cases_claim} holds for all $\ell\in [p]$, we have
            \begin{align*}
                \frac{H(\scons{})}{H(S^\star)}
                &= \frac{\sum\nolimits_{\ell\in [p]} H(\scons{}\cap G_\ell)}{\sum\nolimits_{\ell\in [p]} H(S^\star\cap G_\ell)}\\
                &\geq  \frac{(1-\eps)\cdot \sum\nolimits_{\ell\in [p]} H(S^\star\cap G_\ell)}{\sum\nolimits_{\ell\in [p]} H(S^\star\cap G_\ell)}\tag{Using \cref{eq:cases_claim}}\\
                &= (1-\eps).
            \end{align*}

            \paragraph{Proof of Claim A.}
            Fix any $\ell\subseteq [p]$.
            Let $i(h)$ have the $h$-th largest value of $w$ in $\inbrace{w_i: i\in G_\ell\cap C_j}$.
            Because $\phi_\ell$ is increasing, it follows that $i(h)$ also has the $h$-th largest value of $\hw$ in $\inbrace{\hw_i: i\in G_\ell\cap C_j}=\inbrace{\phi_\ell(w_i): i\in G_\ell\cap C_j}$.
            Define
            \begin{align*}
                r\coloneqq \abs{S^\star\cap G_\ell}
                \quad\text{and}\quad
                \wt{r}\coloneqq \abs{S_j \cap G_\ell}.
            \end{align*}
            It holds that
            $$S^\star\cap G_\ell=\inbrace{i(1),\dots,i(r)}.$$
            (Otherwise, we can increase $H(S^\star)$ by swapping an element of $S^\star$ by an element in $\inbrace{i(1),\dots,i(r)}\backslash S^\star$.)
            Further, it also holds that $$\scons{}\cap G_\ell=\inbrace{i(1),\dots,i(\wt{r})}.$$
            (Otherwise, we can increase $\hH(\scons{})$ by swapping an element $i$ of $\scons{}$ by an element $i'$ in $\inbrace{i(1),\dots,i(\wt{r})}\backslash \scons{}$; {note that this does not violate the proportional representation constraint because both $i,i'\in G_\ell$.)}

            \medskip
            \noindent \textit{(Step 1: Lower bound on $\frac{H(\scons{}\cap G_\ell)}{H(S^\star\cap G_\ell)}$)}
            Using the above observations we have the following expression
            \begin{align*}
                \frac{H(\scons{}\cap G_\ell)}{H(S^\star\cap G_\ell)}
                &= \frac{
                    w_{i(1)}+w_{i(2)}+\cdots+w_{i(\wt{r})}%
                    }{
                    w_{i(1)}+w_{i(2)}+\cdots+w_{i(r)}
                    }.
            \end{align*}
            Consequently, if $\wt{r}\geq r$, then $\frac{H(\scons{}\cap G_\ell)}{H(S^\star\cap G_\ell)}\geq 1.$
            Otherwise, we have the following lower bound
            \begin{align*}
                \frac{H(\scons{}\cap G_\ell)}{H(S^\star\cap G_\ell)}
                &= \frac{
                    w_{i(1)}+w_{i(2)}+\cdots+w_{i(\wt{r})}%
                    }{
                    w_{i(1)}+w_{i(2)}+\cdots+w_{i(r)}
                    }\\
                &\geq \frac{
                    w_{i(1)}+w_{i(2)}+\cdots+w_{i(\wt{r})}%
                    }{
                    w_{i(1)}+w_{i(2)}+\cdots+w_{i(\wt{r})}\cdot (r-\wt{r}+1)
                    }
                    \tag{Using that $w_{i(1)}\geq w_{i(2)}\geq \cdots\geq w_{i(r)}$}\\
                &\geq 1- \frac{
                    w_{i(\wt{r})} \cdot (r-\wt{r})%
                    }{
                    w_{i(1)}+w_{i(2)}+\cdots+w_{i(\wt{r})}\cdot (r-\wt{r}+1)
                    }
                    \tag{Using that $w_{i(1)}\geq w_{i(2)}\geq \cdots\geq w_{i(r)}$}\\
                &\geq \frac{\wt{r}}{r}.\tag{Using that $w_{i(\wt{r})}\leq w_{i(i)}$ for all $1\leq i\leq \wt{r}$}
            \end{align*}
            Hence, in either case
            \begin{align*}
                \frac{H(\scons{}\cap G_\ell)}{H(S^\star\cap G_\ell)}\geq \frac{\wt{r}}{r}.
                \yesnum\label{eq:tmptmp_lb}
            \end{align*}

            \smallskip
            \noindent \textit{(Step 2: Lower bound on $\frac{\wt{r}}{r}$).}
            First,
            \begin{align*}
                \wt{r}=k\cdot \frac{\abs{G_\ell\cap C_j}}{n}.
                \yesnum\label{eq:val_r_beta}
            \end{align*}
            (Otherwise, adding $i(\wt{r}+1)$ to $\scons{}$ increases its utility and does not violate the proportional representation constraint).
            Next, because $S^\star\subseteq C_j$ is independent of the protected groups $G_1,G_2,\dots,G_p$ and because $G_\ell$ is constructed by drawing $\abs{G_\ell}$ elements uniformly without replacement from $C_j$, it follows that $r\coloneqq \abs{S^\star\cap G_\ell},$
            is a random variable distributed according to the hyper-geometric distribution with parameters $n=\abs{G_\ell}$, $K=\abs{S^\star}$, and $N=\abs{C_j}$.
            (Recall that for a hyper-geometric random variable $X$, $\Pr[X=r]$ denotes the probability of obtaining $r$ red balls in $n$ draws, without replacement, from an urn with $K$ red balls and $N-K$ blue balls.)
            From standard properties of the hyper-geometric distribution \cite[Theorem 1]{hush2005concentration}, it follows that
            \begin{align}
                \Ex[r] = \abs{G_\ell\cap C_j}\cdot \frac{\abs{S^\star}}{n} = \abs{G_\ell\cap C_j}\cdot \frac{k}{n}
                \ \ \text{and}\ \
                \Pr\insquare{\abs{r- \abs{G_\ell\cap C_j}\cdot \frac{k}{n}}\leq \ln\inparen{\frac{1}{\delta}}\sqrt{k}}\leq \delta.
            \end{align}
            Consequently, with probability at least $1-\delta$
            \begin{align*}
                r\leq \abs{G_\ell\cap C_j}\cdot \frac{k}{n}+\ln\inparen{\frac{1}{\delta}}\sqrt{k}=\wt{r}+\ln\inparen{\frac{1}{\delta}}\sqrt{k}.
                \yesnum\label{eq:tmp_event}
            \end{align*}

            \smallskip

            \noindent\textit{(Step 3: Proof of Claim A).} It follows that, with probability at least $1-\delta$,
            \begin{align*}
                \frac{H(\scons{}\cap G_\ell)}{H(S^\star\cap G_\ell)}
                &\geq \frac{1}{1+\frac{\ln\inparen{\frac{1}{\delta}}\sqrt{k}}{\wt{r}}} \tag{Using \cref{eq:tmptmp_lb,eq:tmp_event}}\\
                &\geq 1 - \frac{\ln\inparen{\frac{1}{\delta}}\sqrt{k}}{\wt{r}}
                \tag{Using $1-x\leq \frac{1}{1+x}$ for all $x\in \R$}\\
                &= 1 - \frac{\ln\inparen{\frac{1}{\delta}}\cdot n}{\sqrt{k}\cdot \abs{G_\ell}}\tag{Using \cref{eq:val_r_beta}}\\
                &\geq 1 - \frac{\ln\inparen{\frac{1}{\delta}}}{\gamma\sqrt{k}}.
            \end{align*}
            
            \vspace*{-0.5cm}
        \end{proof}

\vspace*{-0.5cm}
\section{Limitations and conclusion}\label{sec:lim_conc} 
    This work studies the maximization of submodular functions which have been used to capture the utility of subsets of items in recommendation systems and web search.
    It studies this in the setting where the inputs defining the submodular function have social biases--modeled by an extension of \cite{KleinbergR18}'s bias model--and these biases lead to a reduction in the latent utility of the output subset.
    Our first result shows that maximizing the observed utility subject to fairness constraints is not sufficient to recover any fraction of the optimal latent utility (\cref{thm:main_negative_result}).
    On the positive side, we give an algorithm (\cref{alg:disj}) for submodular maximization that works for a general family of submodular functions capturing submodular functions used in recommendation and web search.
    Under mild assumptions, the algorithm provably outputs a subset with near-optimal latent utility (\cref{thm:disj}).
    Empirically, the subsets output by this algorithm have higher latent utility than baselines even when the assumptions required by the theoretical results do not hold (\cref{fig:syn_data_combined,fig:real_world_4}).

    Our work raises several questions for future work.
    {Our algorithm achieves near-optimal latent utility when the number of items selected $k$ is large compared to the number of attributes $m$.
    Designing algorithms that achieve a high latent utility for a larger range of $k$ is an interesting direction.}
    {Further, our model (like several others \cite{KleinbergR18,celis2020interventions}) has the limitation that it does not model intersectionality, while this can be partially addressed by defining each intersection as a separate group, this may be unfeasible when the sizes of the intersections are too small \cite{kearns2018gerrymandering}.}
    Empirical results on real-world data showed that our algorithm can outperform baselines, even in some cases where data does not follow the theoretical model considered.
    However, a careful assessment of our algorithms' performance on application-specific data, both pre-deployment and post-deployment, would be important to avoid any unintended harm. %
    {Moreover, in certain contexts, affirmative action policies have been shown to have positive long-term effects on the bias in the ``system'' \cite{ceslis2021longterm,heidari2021intergenerational}.
    Our algorithm can be seen as an affirmative action policy with data-dependent constraints, and studying its long-term effects is an interesting direction.}
    Furthermore, submodular maximization is one part of the larger information retrieval or recommendation system; examining the effect of biases in the input on other parts of the system and evaluating the efficacy of our algorithm in conjunction with the broader system are interesting directions.

\paragraph{Acknowledgments} 
    This project is supported in part by NSF Awards  (CCF-2112665 and IIS-2045951), and an AWS MLRA Award.

\clearpage

\bibliographystyle{plain}
\bibliography{bib-v1.bib}

\appendix

\renewcommand{\algorithmicrequire}{\textbf{Input:}}
\renewcommand{\algorithmicensure}{\textbf{Output:}}

\clearpage

\section{Further discussion of submodular maximization and related works}
    \subsection{Additional examples of submodular functions used by prior works}\label{app:additional_examples}
        The diminishing returns property, of submodular functions, arises in many applications, including content recommendation \cite{spotify_submod,amazon_submod}, web search \cite{microsoft_diverse}, text summarization \cite{lin2009graph}, and team selection \cite{hong2004groups,jeppesen2010marginality, KleinbergRaghu18}.
        This property is one of the key reasons for the use of submodular functions in the above applications \cite{krause2014submodular}.
   
        For instance, submodular functions are used when a recommendation system has different objectives and each additional item satisfying the same objective has a diminishing return.
        Concretely, \cite{spotify_submod} explain that ``a fundamental requirement of [Spotify's] music recommender system is its ability to accommodate considerations from the users (e.g. short-term satisfaction objectives), artists (e.g. exposure of emerging artists) and platform (e.g. facilitating discovery and boosting strategic content) when surfacing music content to users''
        \cite{spotify_submod} design the recommendation system for Spotify that uses a submodular objective function.
        For each item $i$ (e.g., song or podcast), let $W_{i1}\geq 0$ denote its relevance to a user (predicted by a learning algorithm) and $W_{i2},\dots, W_{im}\in \zo$ indicate artist and platform-specific metrics (e.g., if $i$ is created by an emerging artist or if the platform wants to promote $i$).
        \cite{spotify_submod} use the following objective to capture the utility of a playlist $S$ for the user
        \begin{align*}
            F(S) \coloneqq \sum\nolimits_{i\in S}W_{i1} + \sqrt{\sum\nolimits_{i\in S} W_{i2}}+\dots + \sqrt{\sum\nolimits_{i\in S} W_{im}}.
            \yesnum\label{eq:example:2}
        \end{align*}

        \noindent Submodular functions also arise in web search, where each query $q$ can have multiple interpretations:
        For instance, the query ``flash'' can refer to the Adobe Flash player, the superhero ``The Flash'', or the village Flash with the highest elevation in Great Britain.
        Irrespective of the intended interpretation, each additional result related to the same interpretation offers a smaller marginal utility to the user  \cite{microsoft_diverse}.
        Suppose a query $q$ (e.g., ``Flash'') belongs to category $j\in [m]$ (e.g., technology, movies, or location) with probability $\Pr[j \mid q]$ and, conditioned on the event that $q$ belongs to category $j$,  a result $i$ satisfies the user with probability $\Pr[i\mid j, q]$ (independent of other items).
        \cite{microsoft_diverse} observe that the set of search results, $S$, that maximizes the following submodular objective has a higher quality than search results of a commercial search engine.
        \begin{align*}
            F(S)\coloneqq \sum\nolimits_{j=1}^m \Pr[j \mid q]\inparen{1-\prod_{i\in S}\inparen{1 - \Pr[i\mid j, q]}}.
            \yesnum\label{eq:example:3}
        \end{align*}
        This is in the family $\evF$.
        To see this set, $W_{ij}=\log{\frac{1}{1-\Pr[i\mid j, q]}}$ and $g_j(x)=-\Pr[j \mid q]\cdot e^{-x}$ for all $i$ and $1\leq j\leq m$, and $g_{j+1}(x)=1$ for all $x$.
        
        Given weights $w_1,\dots,w_{m-1}$ and scores $s_1,\dots,s_n$,
        \cite{amazon_stream_submod} consider the following submodular function $F\in \evF$ defined by (1) $g_j(x)=w_j\cdot \log(1+x)$ for all $1\leq j\leq m-1$ and $g_m(x)=x$, (2) for each $1\leq j\leq m-1$, $W_{ij}=1$ if $i$ has the $j$-th attribute and $W_{ij}=0$ otherwise, and (3) $W_{im}=s_i$.

    \subsection{Further discussion of works on benefits of fairness constraints on utility}\label{sec:related_work_app}
 
        Recent works \cite{KleinbergR18,celis2020interventions, EmelianovGGL20,mehrotra2022intersectional} have demonstrated the benefit of imposing fairness constraints on the output of subset selection on the latent utility of the output when the objective is additive or {\em linear}.
        \cite{KleinbergR18} introduce the mathematical model of bias mentioned in \cref{sec:intro}. 
        They consider two groups, with $G_2$ being the disadvantaged group, and study conditions on $\beta,$ group sizes, $k$, and the distribution of $W$, where requiring the output $S$ to satisfy $\abs{S\cap G_2}\geq 1$ increases the latent utility of the output.
        \cite{celis2020interventions} study a generalization of subset selection, ranking, where the selected individuals also need to be ordered.
        Specializing their work to subset selection:
        They consider two groups $G_1$ and $G_2$ and show that if $U$ encodes proportional representation and entries of $W$ are drawn i.i.d. from the uniform distribution on $[0,1]$, then $F(S_U)\geq (1-o_k(1))\cdot \opt{}$.\footnote{To be precise they required the output to have at least $L_\ell$ items from group $G_\ell$. For two groups, this is equivalent to requiring the output subset to have at most $U_{\ell}=k-L_{\lnot\ell}$ items from $G_\ell$.}
        \cite{EmelianovGGL20} study selection under a different model of bias, where the observed utility has higher than average noise for individuals in one group.
        They give a family of fairness constraints (including proportional representation) which increases the output's latent utility.
        \cite{mehrotra2022intersectional} study subset selection under the same model as \cite{celis2020interventions}, but with multiple and overlapping groups.
        Consider two overlapping groups $G_1,G_2$, these divide the items into four disjoint ``intersections:'' $I_{a}\coloneqq G_1\cap G_2$, $I_{b}\coloneqq G_1\backslash G_2$, $I_{c}\coloneqq G_2\backslash G_1$, and $I_{d}\coloneqq [n]\backslash (G_1\cup G_2)$.
        Similarly, $p$ overlapping groups $G_1,G_2,\dots,G_p$ divide the items into up to $2^p$ intersections.
        \cite{mehrotra2022intersectional} study the model where each item $i$ in a different intersection faces a different amount of bias:
        The observed utility of $i$ is ${\prod_{\ell:G_\ell \ni i} \beta_\ell}$ times smaller than its latent utility.
        \cite{mehrotra2022intersectional} compare the efficacy of the fairness constraints applied on groups $G_1, G_2,\dots, G_p$ to  fairness constraints applied on the disjoint intersections formed by these groups. 
        Most relevant to this work, they extend the result of \cite{celis2020interventions} for proportional representation constraints to multiple disjoint groups $G_1, G_2,\dots, G_p$.
        We can take a similar approach to extend our results to multiple overlapping groups:
        \begin{remark}[\textbf{Overlapping groups}]\label{rem:overlapping_to_disjoint}
            Instead of $p$ overlapping groups $G_1,G_2,\dots,G_p$, one can consider the groups as the intersections formed by $G_1,G_2,\dots,G_p$
            For example, with $p=2$, one can consider $I_a$, $I_b$, $I_c$, and $I_d$ as the groups with
            bias functions $\phi_a=\phi_1\circ\phi_2$, $\phi_b=\phi_1$,$\phi_c=\phi_2,$ and $\phi_d(x)=x$ respectively.
            Since there are at most $\min\inparen{n,2^p}$ non-empty intersections, this does not increase the running time of the algorithm proposed in the paper (\cref{thm:disj}).
            For simplicity, in this paper, we assume that the groups are disjoint.
        \end{remark}
        \noindent Unlike all of the above works, we study the efficacy of fairness constraints when the objective of subset selection is submodular and not necessarily linear.

        \paragraph{The need for stochasticity in groups.}
        In our model, we assume that the latent utilities are deterministically chosen and the protected groups are stochastic.
        This generalizes the model of \cite{KleinbergR18,celis2020interventions,mehrotra2022intersectional} which is equivalent to the model that draws latent utilities i.i.d. from some distribution and also constructs protected groups stochastically.
        A further generalization could consider the case where neither the latent utilities nor protected groups are stochastic.
        However, in this model, it is information-theoretically impossible to output a subset whose latent utility is to guaranteed to be at least a positive fraction  of $\opt$.
        Formally, we can show the following result if both latent utilities and protected groups can be arbitrary then no algorithm can recover any constant factor of approximation of the optimal utility.
        Let $S_A$ be the subset output by algorithm $A$ when given $\smash{\hW}$ and $G_1,G_2,\dots,G_p$ as input.
        For any $\eps>0$ and any algorithm $A$, there are two disjoint protected groups $G_1$ and $G_2$ and bias functions $\phi_1,\phi_2:\R\to\R$ such that
        $F(S_A)\leq \eps$.

    \subsection{Standard algorithms for submodular maximization}\label{sec:standard_algs}
    In this section, we present a standard greedy algorithm by \cite{nemhauser1978analysis} for maximizing monotone submodular functions with, e.g., cardinality constraints.
    
    \begin{algorithm}[H] %
            \caption{The standard greedy algorithm (\cite{nemhauser1978analysis}) }\label{alg:model_greedy}
            \begin{algorithmic}[1]
                \Require An evaluation oracle for $F$, a number $k$, and a set of feasible sets $I\subseteq 2^{[n]}$
                \Ensure  A subset $S\in I$ with $\abs{S}\leq k$
                \State Initialize $S=\emptyset$
                \While{$\abs{S} < k$}
                    \State Let $i^\star$ be the item in $[n]$ that maximizes $F(S\cup \inbrace{i^\star})-F(S)$ subject to $S\cup \inbrace{i^\star}\in I$
                \State Set $S=S\cup \inbrace{i^\star}$
                \EndWhile
                \State \Return $S$
            \end{algorithmic}
        \end{algorithm}
        
        \begin{algorithm}[H] %
        \caption{A variant of the standard greedy algorithm that satisfies given upper bound constraints}\label{alg:greedy_with_ub}
        \begin{algorithmic}[1]
            \Require An evaluation oracle for $F$, a numbers $k, U_1,\dots,U_p$, and groups $G_1,G_2,\dots,G_p$
            \Ensure  A subset $S$ with $\abs{S}\leq k$ and $\abs{S\cap G_\ell}\leq U_\ell$
            \State Initialize $S=\emptyset$ 
            \While{$\abs{S} < k$}
                \State Let $i^\star$ be the item in $[n]$ that maximizes $F(S\cup \sinbrace{i^\star})-F(S)$ subject to $\abs{(S\cup \sinbrace{i^\star})\cap G_\ell}\leq U_\ell$ (for each $\ell$)
                \State Set $S=S\cup \inbrace{i^\star}$
            \EndWhile
            \State \Return $S$
        \end{algorithmic}
    \end{algorithm}

\section{Implementation details and additional plots}\label{sec:implementation_details}
    
    \subsection{Implementation details}
        \smallskip
        \paragraph{\bf Code.}  The code for our simulations is available at 
        
        \url{https://github.com/AnayMehrotra/Submodular-Maximization-in-the-Presence-of-Biases}
         
        \smallskip
        \paragraph{\bf Synthetic dataset 1.}
        \cref{fig:syn_iid_non_disj} presents results with this data.
        This data has $m\coloneqq 3$ attributes and $n\coloneqq 250$ songs, and uses the objective $F(S)\coloneqq \sum\nolimits_{i\in S}W_{i1} + \lambda \sum\nolimits_{j=2}^3 \sqrt{\sum\nolimits_{i\in S}W_{ij}}$ (with $\lambda=\frac{1}{20}$).
        \begin{itemize}[leftmargin=10pt]
            \item  First, we select a subset $S_E$ of songs with size $\abs{S_E}=0.8n$ and label all songs in $S_E$ to be from an emerging artist and all other songs as songs from non-emerging artists.
            This implies that $W_{i2}=1$ if $i\in S_E$ and $W_{i2}=0$ otherwise.
            \item Next, for each song $i\in [n]$, with probability $p_{\rm NH}\coloneqq 0.9$, we label it as \texttt{``not heard by the current user''} and set $W_{i3}=1$
            and, otherwise, we label it as \texttt{``heard by the current user''} and set $W_{i3}=0$.
            \item Finally, for each song $i\not\in S_E$, we independently draw a value $X$ from the power-law distribution with exponent $\delta$ and set $W_{i1}=X\cdot 1000$.\footnotemark{}
            \addtocounter{footnote}{-1}
            For each $i\in S_E$, we independently draw a value $X$ from the power-law distribution with exponent $\delta$ {\em conditioned on $X\leq 2$} and set $W_{i1}=X\cdot 1000$.
            The conditioning encodes the fact that emerging artists do not have any ``popular'' songs yet.
        \end{itemize}
        We fix $\lambda\coloneqq \frac{1}{20}$, $\abs{S_E}=0.8n$, and $p_{\rm NH}=0.9$ and vary $\beta\in [0,1]$, $\delta\in \inbrace{1, 1.5, 2, 2.5, 3}$, and  $\frac{\abs{G_1}}{n}\in \inbrace{0.25, 0.5, 0.75}$.

        \begin{remark}
            We also repeated the simulation with {$\abs{S_E}\in \inbrace{0.4n,0.9n}$, $p_{\rm NH}\in \inbrace{0.4,1.0}$, and $\lambda\in \inbrace{\frac{1}{10},\frac{1}{5}}$}, and observed similar results as \cref{fig:syn_data_combined,fig:syn_iid_disj_full,fig:syn_iid_non_disj_full}.
        \end{remark}

        \smallskip
        \paragraph{Synthetic dataset 2.}
        The second synthetic dataset corresponds \cref{fig:syn_iid_disj}, has $m\coloneqq 3$ attributes and $n\coloneqq 250$ songs, and uses the objective function $F(S)\coloneqq \sum\nolimits_{j=1}^3 \log\inparen{1+\sum\nolimits_{i\in S}W_{ij}}.$

        \begin{itemize}[leftmargin=10pt]
            \item We generate sets $C_1,C_2,C_3$ uniformly at random:
            For each item $i$, with probability $\frac{1}{3}$ we assign it to $C_1$, otherwise with probability $\frac{1}{3}$ we assign it to $C_2$, and otherwise we assign it to $C_3$.
            \item For each $h\in [3]$ and $i\in C_h$, we independently draw a value $X$ from the power-law distribution with exponent $\delta$ and set $W_{i1}=X\cdot 1000$.\footnote{This is natural as power-law distributions have been observed to arise in the performance of musicians \cite{power_law_music} and in the performance of other creative professionals \cite{clauset2009power}.}
        \end{itemize}

        \noindent Like the first synthetic dataset, we generate groups $G_1$ and $G_2$ by assigning $\abs{G_1}$ items chosen uniformly at random without replacement to $G_1$ and the remaining items to $G_2$
        Given $\beta\in [0,1]$, we generate observed utilities $\smash{\hW}$ as in \cref{def:bias_model} with $\phi_1(x)=x$ and $\phi_2(x)=\beta x$.
        The simulation on this data fixed varies $\beta\in [0,1]$, $\delta\in \inbrace{1, 2, 3}$, and  $\frac{\abs{G_1}}{n}\in \inbrace{0.25, 0.5, 0.75}$.

    \subsection{Additional plots for simulations from Section~\ref{sec:empirical_results}}\label{sec:additional_simulations}

    In this section, we give additional plots for the simulations in \cref{sec:empirical_results} on synthetic data (\cref{sec:additional_results_synthetic_data1,sec:additional_results_synthetic_data2}) and on MovieLens 20M (\cref{sec:additional_results_real_world_data}).

    \newpage
    
    \subsection{Additional plots for the simulation with synthetic data 1}\label{sec:additional_results_synthetic_data1}

    \renewcommand{\folder}{./figures/synthetic-iid-non-disjoint}
    \begin{figure}[h!]
        {\includegraphics[width=0.4\linewidth, trim={0cm 0.1cm 0cm 0.1cm},clip]{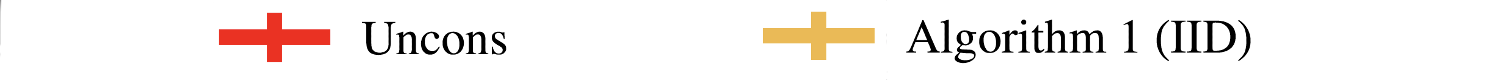}}
        \centering
        \par
            \subfigure[$\abs{G_1}=n/4$ and $\delta=1$]{
                \begin{tikzpicture}
              \node (image) at (0,-0.17) {\includegraphics[width=\ifconf0.25\linewidth\else0.25\linewidth\fi, trim={0cm 0cm 1.5cm 2cm},clip]{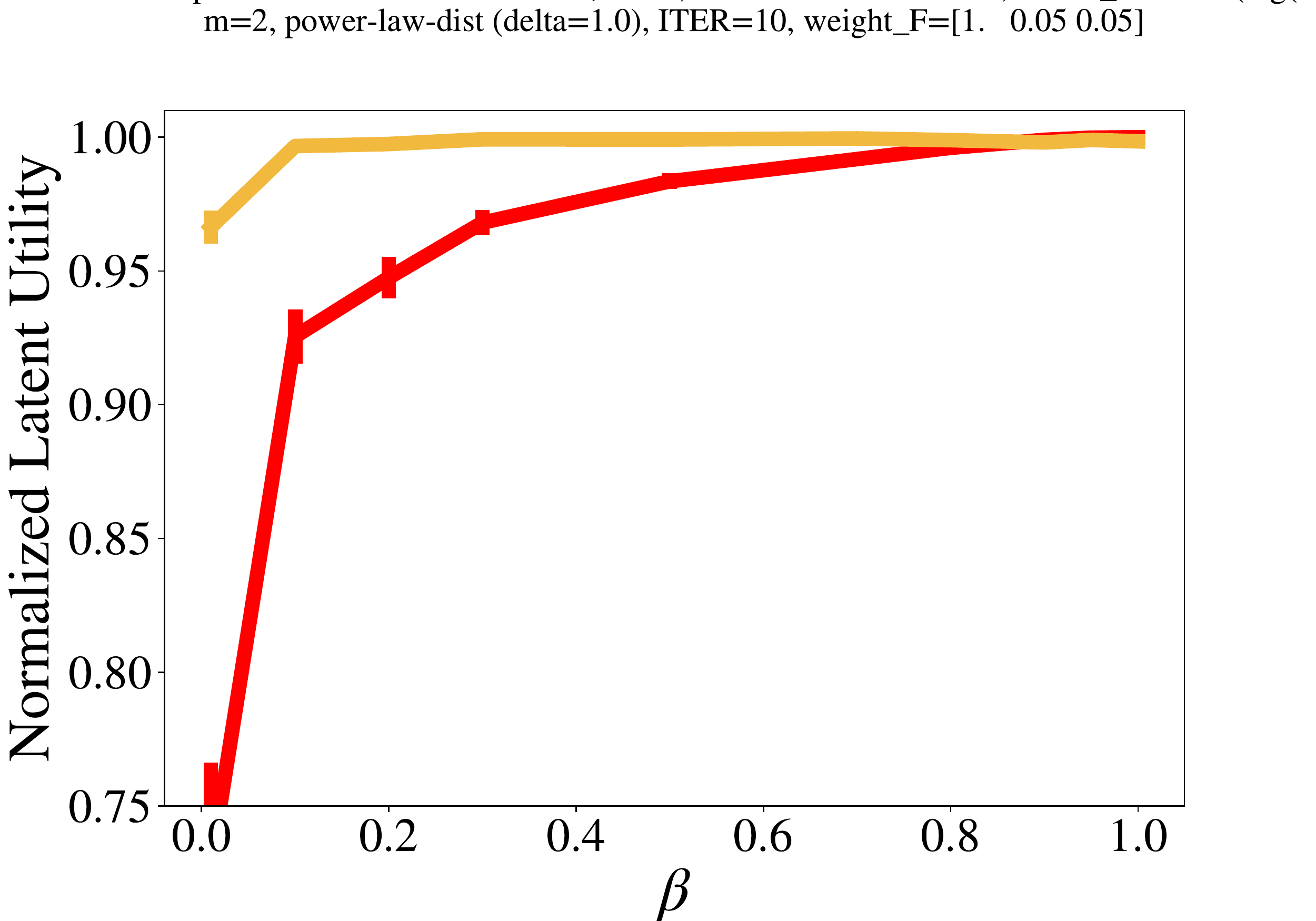}};
              \node[rotate=0] at (2.35 - 1,-1.7) {\textit{\small(less bias)}};
              \node[rotate=0] at (-2.14375 + 1,-1.7) {\textit{\small(more bias)}};
            \end{tikzpicture}
            }
            \subfigure[$\abs{G_1}=n/2$ and $\delta=1$]{
                \begin{tikzpicture}
              \node (image) at (0,-0.17) {\includegraphics[width=\ifconf0.25\linewidth\else0.25\linewidth\fi, trim={0cm 0cm 1.5cm 2cm},clip]{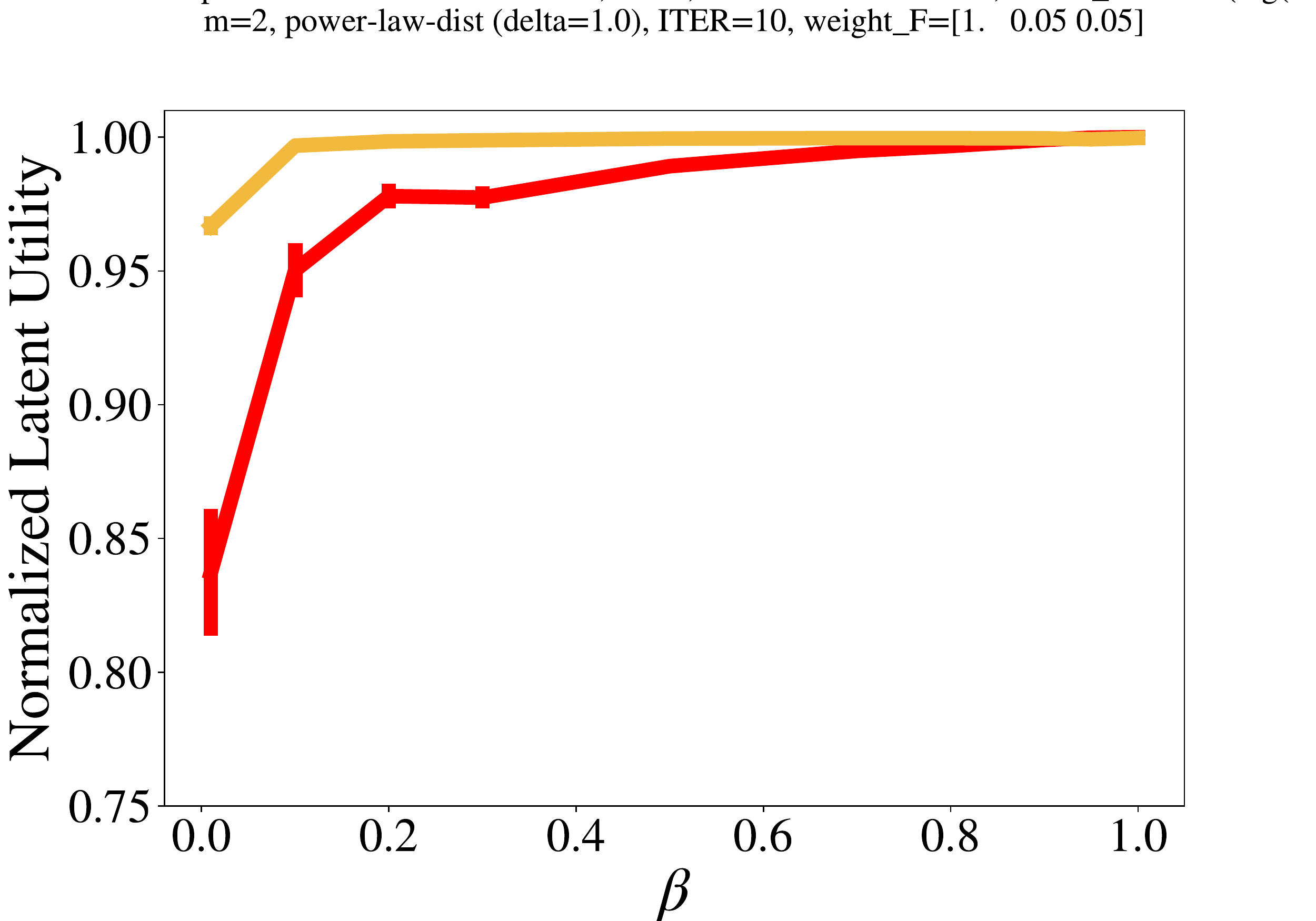}};
              \node[rotate=0] at (2.35 - 1,-1.7) {\textit{\small(less bias)}};
              \node[rotate=0] at (-2.14375 + 1,-1.7) {\textit{\small(more bias)}};
            \end{tikzpicture}
            }
            \subfigure[$\abs{G_1}=(3n)/4$ and $\delta=1$]{
                \begin{tikzpicture}
              \node (image) at (0,-0.17) {\includegraphics[width=\ifconf0.25\linewidth\else0.25\linewidth\fi, trim={0cm 0cm 1.5cm 2cm},clip]{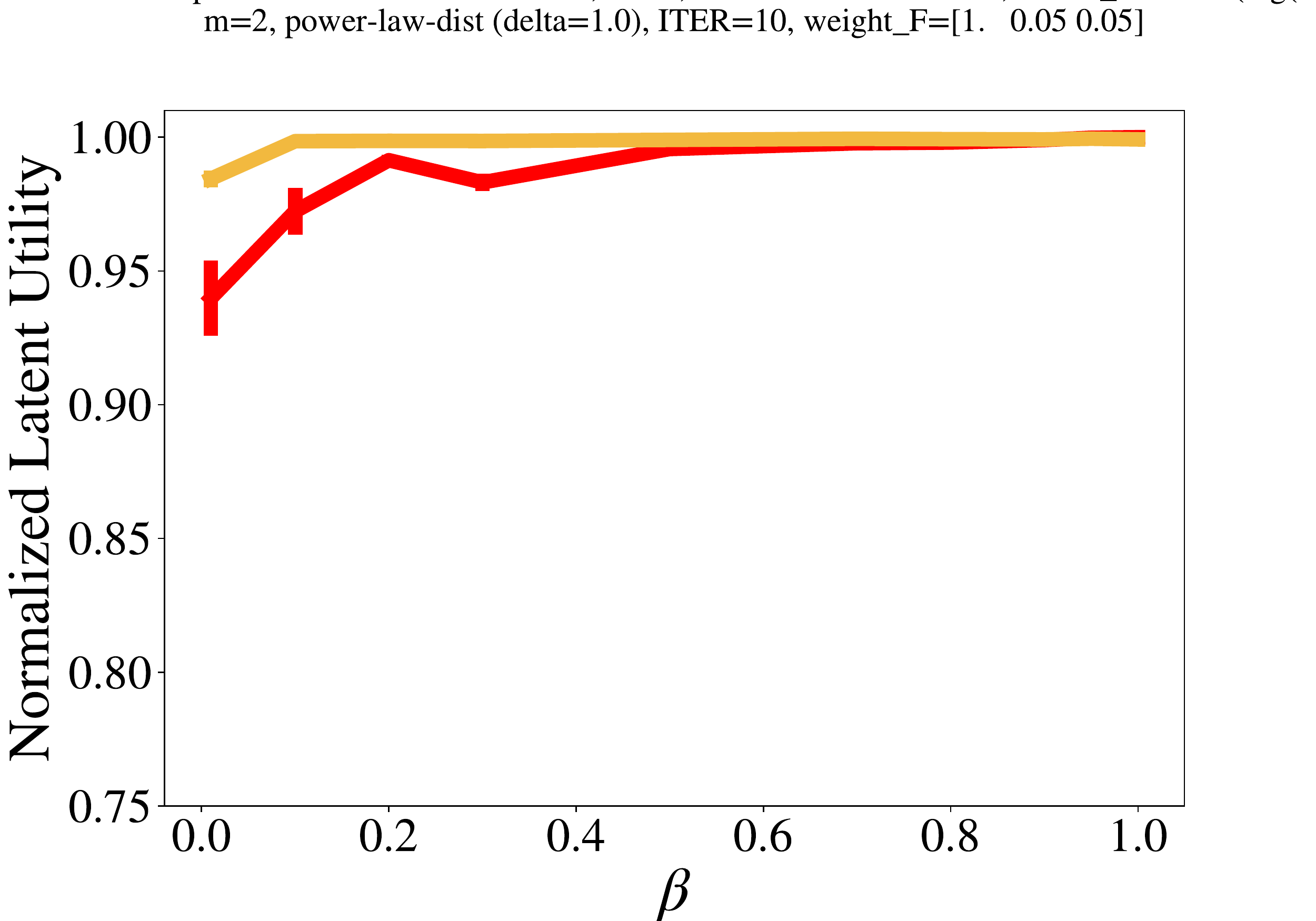}};
              \node[rotate=0] at (2.35 - 1,-1.7) {\textit{\small(less bias)}};
              \node[rotate=0] at (-2.14375 + 1,-1.7) {\textit{\small(more bias)}};
            \end{tikzpicture}
            }
        \ifconf\vspace{-3mm}\else \vspace{-4mm}\fi
        \par
            \subfigure[$\abs{G_1}=n/4$ and $\delta=2$]{
                \begin{tikzpicture}
              \node (image) at (0,-0.17) {\includegraphics[width=\ifconf0.25\linewidth\else0.25\linewidth\fi, trim={0cm 0cm 1.5cm 2cm},clip]{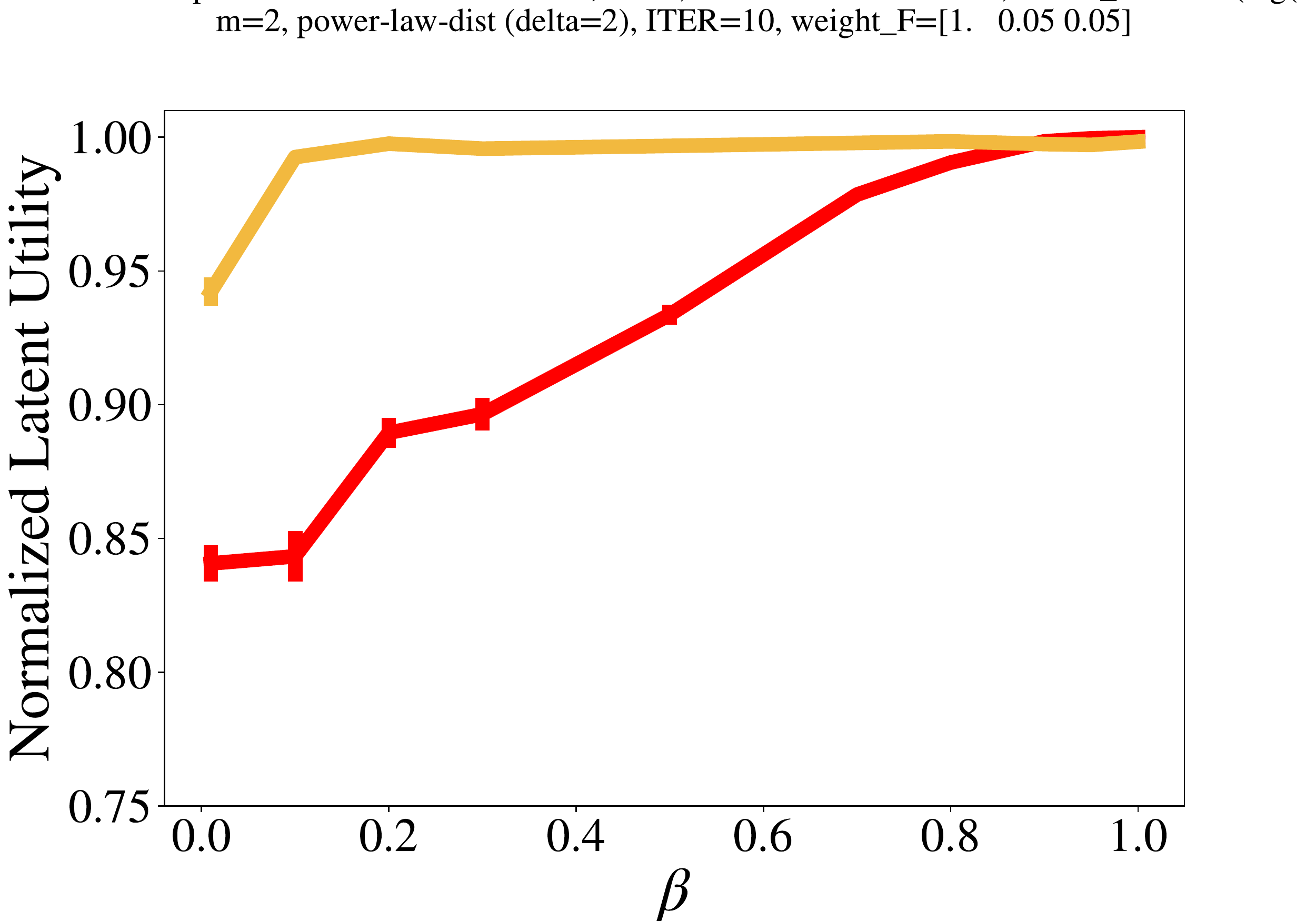}};
              \node[rotate=0] at (2.35 - 1,-1.7) {\textit{\small(less bias)}};
              \node[rotate=0] at (-2.14375 + 1,-1.7) {\textit{\small(more bias)}};
            \end{tikzpicture}
            }
            \subfigure[$\abs{G_1}=n/2$ and $\delta=2$]{
                \begin{tikzpicture}
              \node (image) at (0,-0.17) {\includegraphics[width=\ifconf0.25\linewidth\else0.25\linewidth\fi, trim={0cm 0cm 1.5cm 2cm},clip]{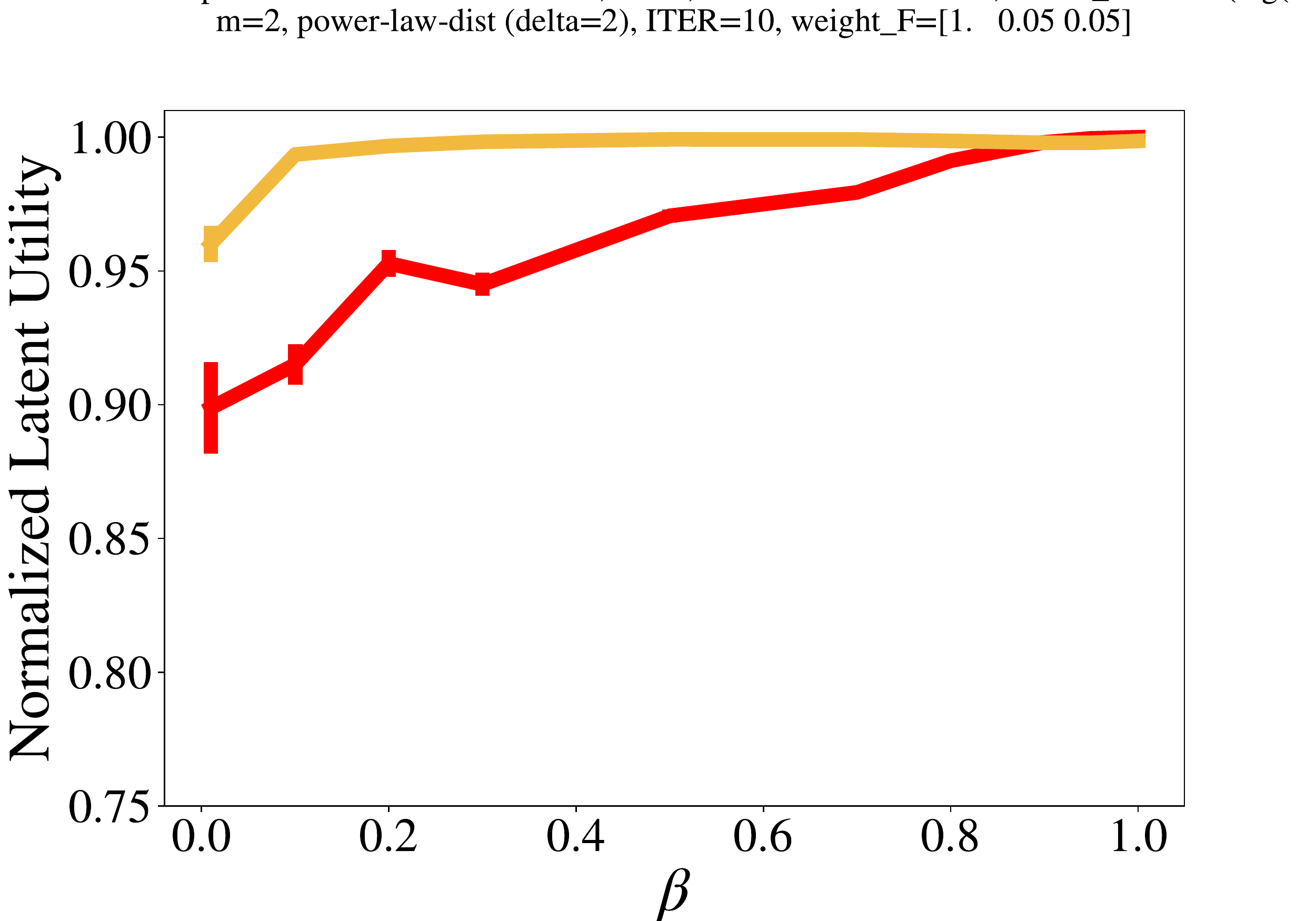}};
              \node[rotate=0] at (2.35 - 1,-1.7) {\textit{\small(less bias)}};
              \node[rotate=0] at (-2.14375 + 1,-1.7) {\textit{\small(more bias)}};
            \end{tikzpicture}
            }
            \subfigure[$\abs{G_1}=(3n)/4$ and $\delta=2$]{
                \begin{tikzpicture}
              \node (image) at (0,-0.17) {\includegraphics[width=\ifconf0.25\linewidth\else0.25\linewidth\fi, trim={0cm 0cm 1.5cm 2cm},clip]{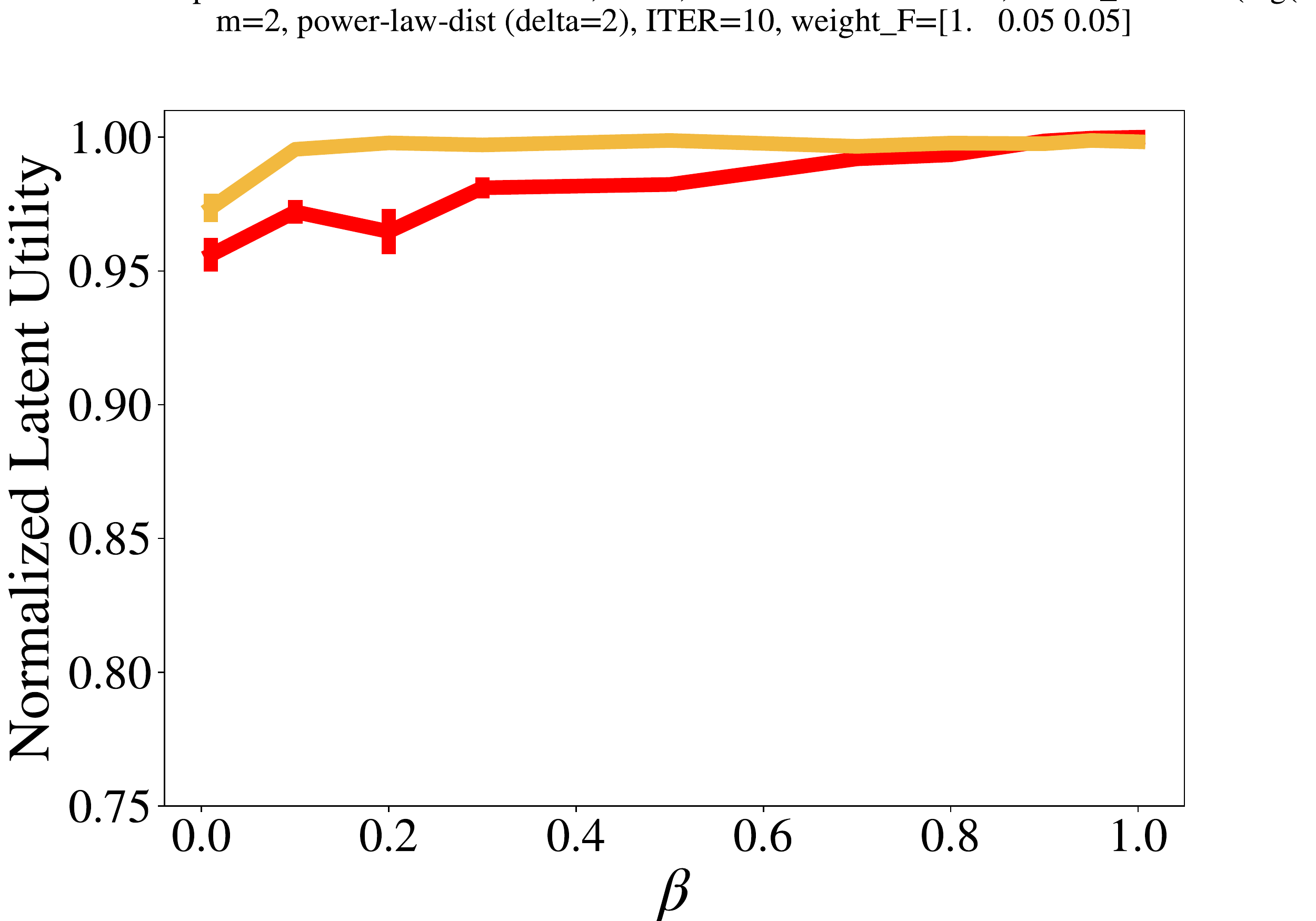}};
              \node[rotate=0] at (2.35 - 1,-1.7) {\textit{\small(less bias)}};
              \node[rotate=0] at (-2.14375 + 1,-1.7) {\textit{\small(more bias)}};
            \end{tikzpicture}
            }
        \ifconf\vspace{-3mm}\else \vspace{-4mm}\fi
        \par
            \subfigure[$\abs{G_1}=n/4$ and $\delta=3$]{
                \begin{tikzpicture}
              \node (image) at (0,-0.17) {\includegraphics[width=\ifconf0.25\linewidth\else0.25\linewidth\fi, trim={0cm 0cm 1.5cm 2cm},clip]{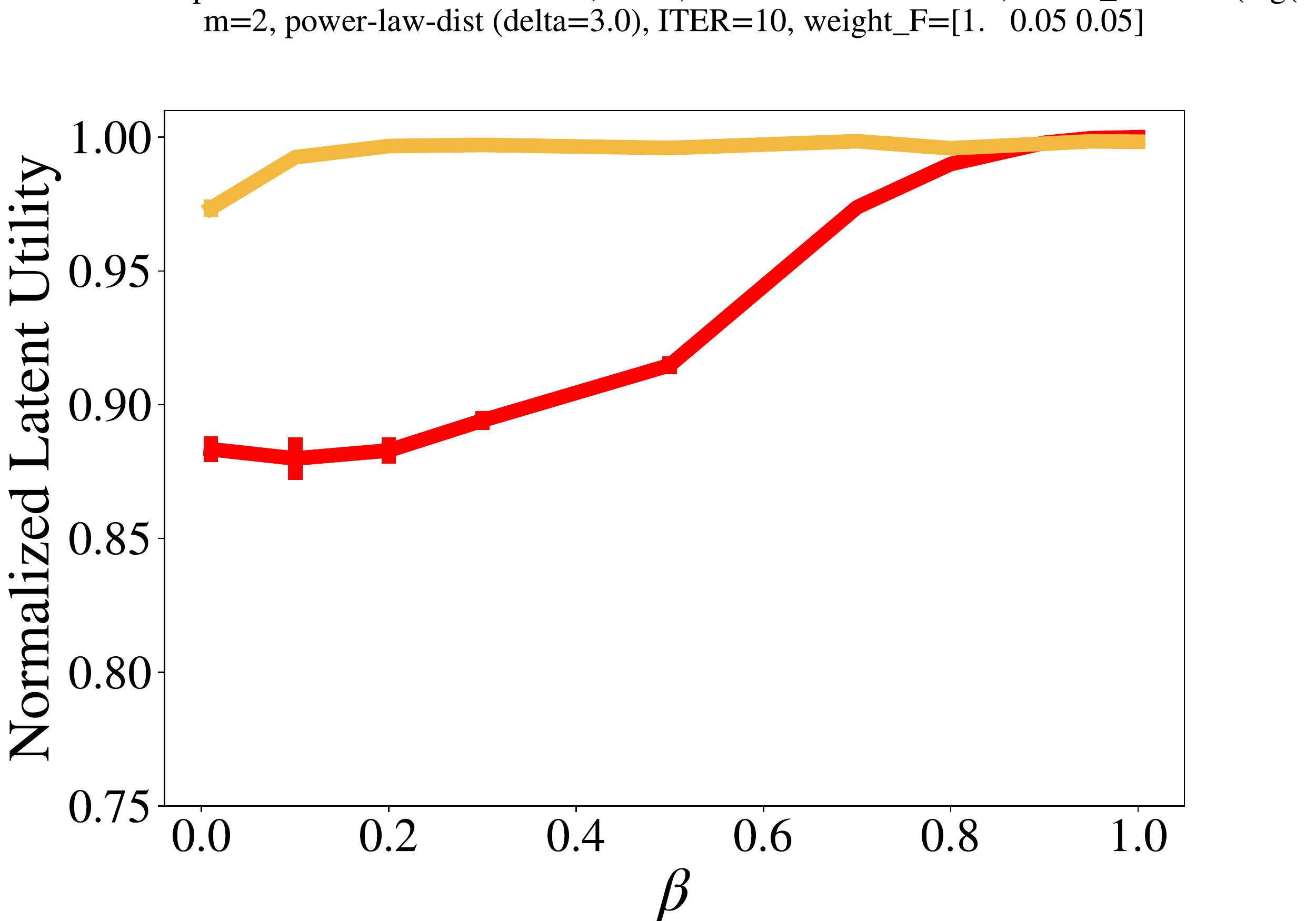}};
              \node[rotate=0] at (2.35 - 1,-1.7) {\textit{\small(less bias)}};
              \node[rotate=0] at (-2.14375 + 1,-1.7) {\textit{\small(more bias)}};
            \end{tikzpicture}
            }
            \subfigure[$\abs{G_1}=n/2$ and $\delta=3$]{
                \begin{tikzpicture}
              \node (image) at (0,-0.17) {\includegraphics[width=\ifconf0.25\linewidth\else0.25\linewidth\fi, trim={0cm 0cm 1.5cm 2cm},clip]{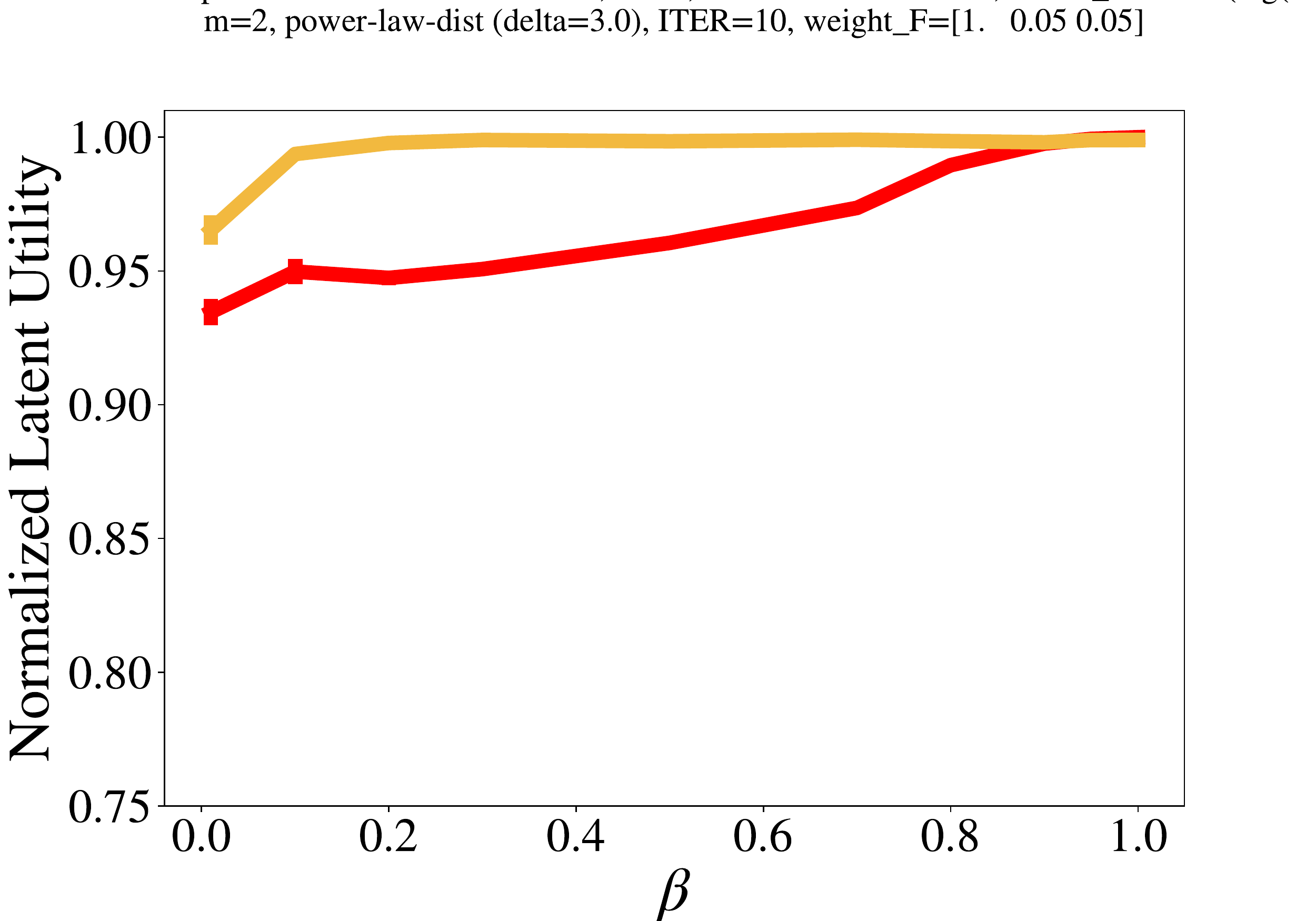}};
              \node[rotate=0] at (2.35 - 1,-1.7) {\textit{\small(less bias)}};
              \node[rotate=0] at (-2.14375 + 1,-1.7) {\textit{\small(more bias)}};
            \end{tikzpicture}
            }
            \subfigure[$\abs{G_1}=(3n)/4$ and $\delta=3$]{
                \begin{tikzpicture}
              \node (image) at (0,-0.17) {\includegraphics[width=\ifconf0.25\linewidth\else0.25\linewidth\fi, trim={0cm 0cm 1.5cm 2cm},clip]{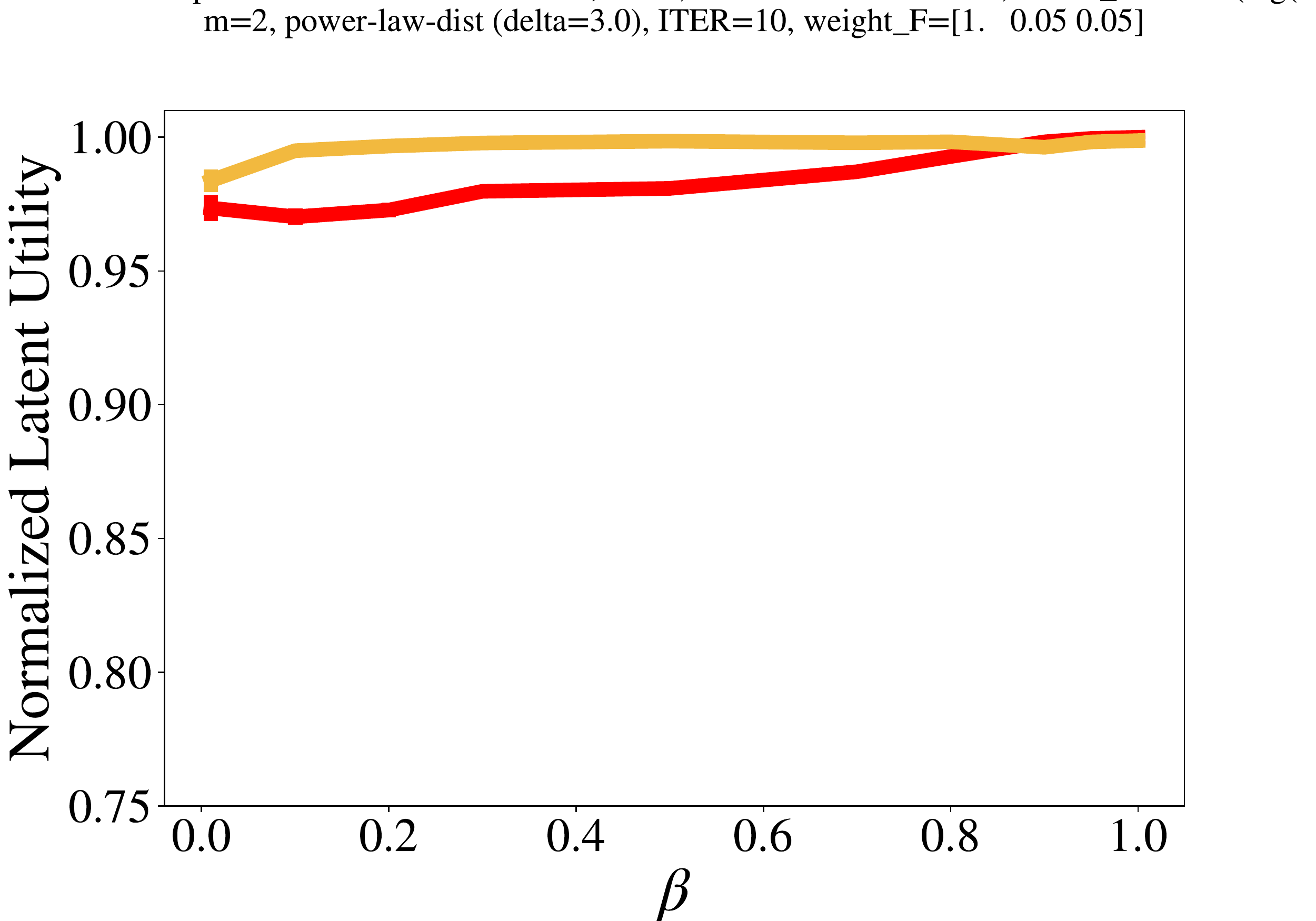}};
              \node[rotate=0] at (2.35 - 1,-1.7) {\textit{\small(less bias)}};
              \node[rotate=0] at (-2.14375 + 1,-1.7) {\textit{\small(more bias)}};
            \end{tikzpicture}
            }
        \caption{
        {\em Simulation on synthetic data 1:}
        We vary the bias parameter $\beta\in [0,1]$, run \cref{alg:disj} and \uncons{} on synthetic data 1, and report their normalized latent  utility (error bars denote standard error of the mean).
        The $x$-axis shows $\beta\in [0,1]$ and the $y$-axis plots the normalized latent utility.
        The results show that \uncons{} can lose a significant fraction of the optimal latent utility in the presence of bias (up to 15\% for $\beta<0.1$, $\frac{\abs{G_1}}{n}= \frac{1}{4}$, and $\delta=3$).
        While \cref{alg:disj}  loses less than 5\% of the optimal latent utility across all choices of parameters.
        }
        \label{fig:syn_iid_non_disj_full}
    \end{figure} 
    
    \newpage

    \subsection{Additional plots for the simulation with synthetic data 2}\label{sec:additional_results_synthetic_data2} 
    
    \renewcommand{\folder}{./figures/synthetic-iid-disjoint}
    \begin{figure}[h!]
        {\includegraphics[width=0.8\linewidth, trim={0cm 0.1cm 0cm 0.1cm},clip]{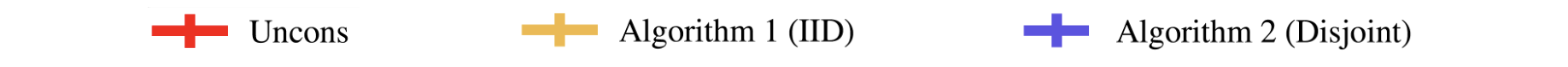}}
        \centering
        \par
            \subfigure[$\abs{G_1}=n/4$ and $\delta=1$]{
                \begin{tikzpicture}
              \node (image) at (0,-0.17) {\includegraphics[width=\ifconf0.25\linewidth\else0.25\linewidth\fi, trim={0cm 0cm 1.5cm 2cm},clip]{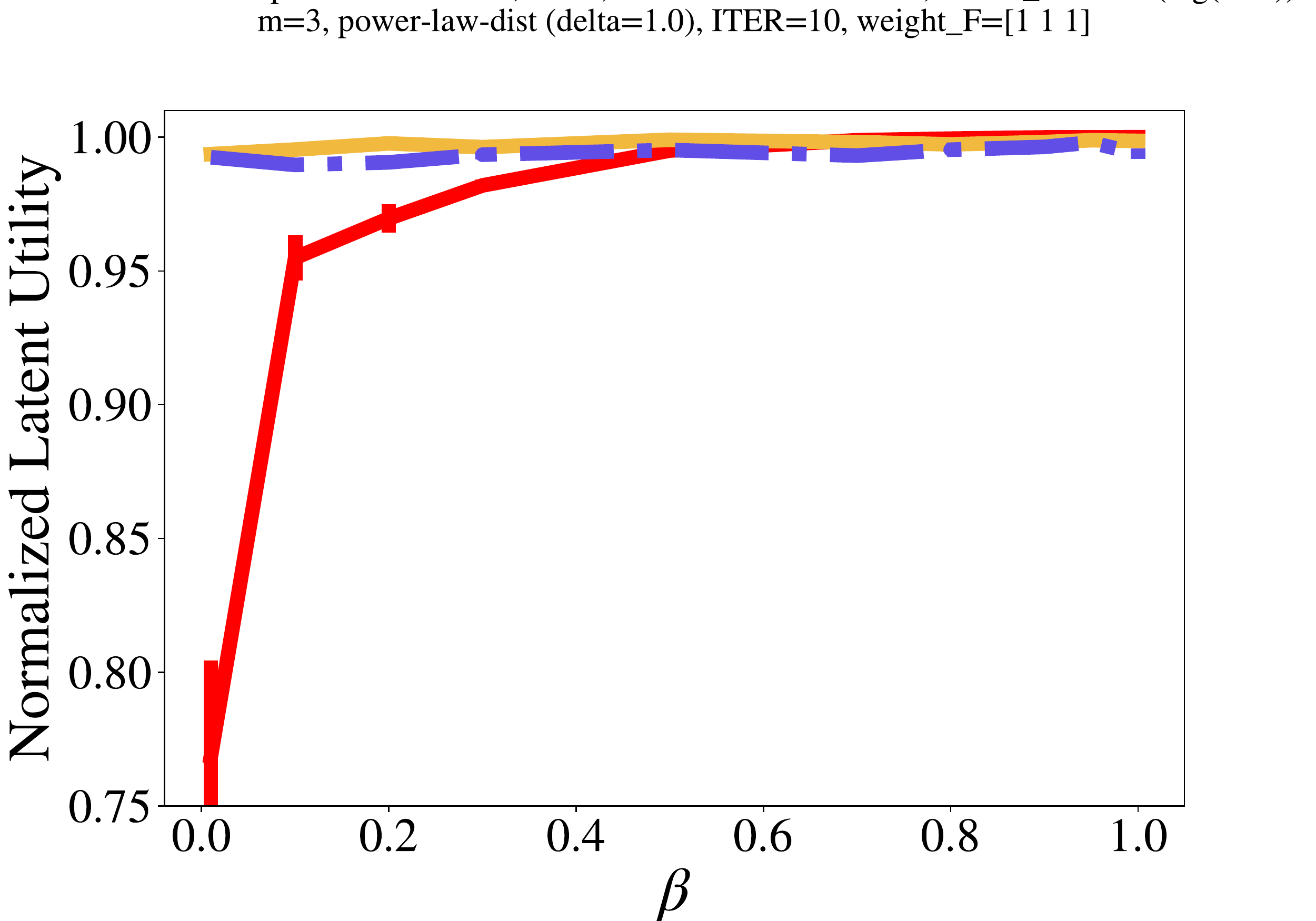}};
              \node[rotate=0] at (2.35 - 1,-1.7) {\textit{\small(less bias)}};
              \node[rotate=0] at (-2.14375 + 1,-1.7) {\textit{\small(more bias)}};
            \end{tikzpicture}
            }
            \subfigure[$\abs{G_1}=n/2$ and $\delta=1$]{
                \begin{tikzpicture}
              \node (image) at (0,-0.17) {\includegraphics[width=\ifconf0.25\linewidth\else0.25\linewidth\fi, trim={0cm 0cm 1.5cm 2cm},clip]{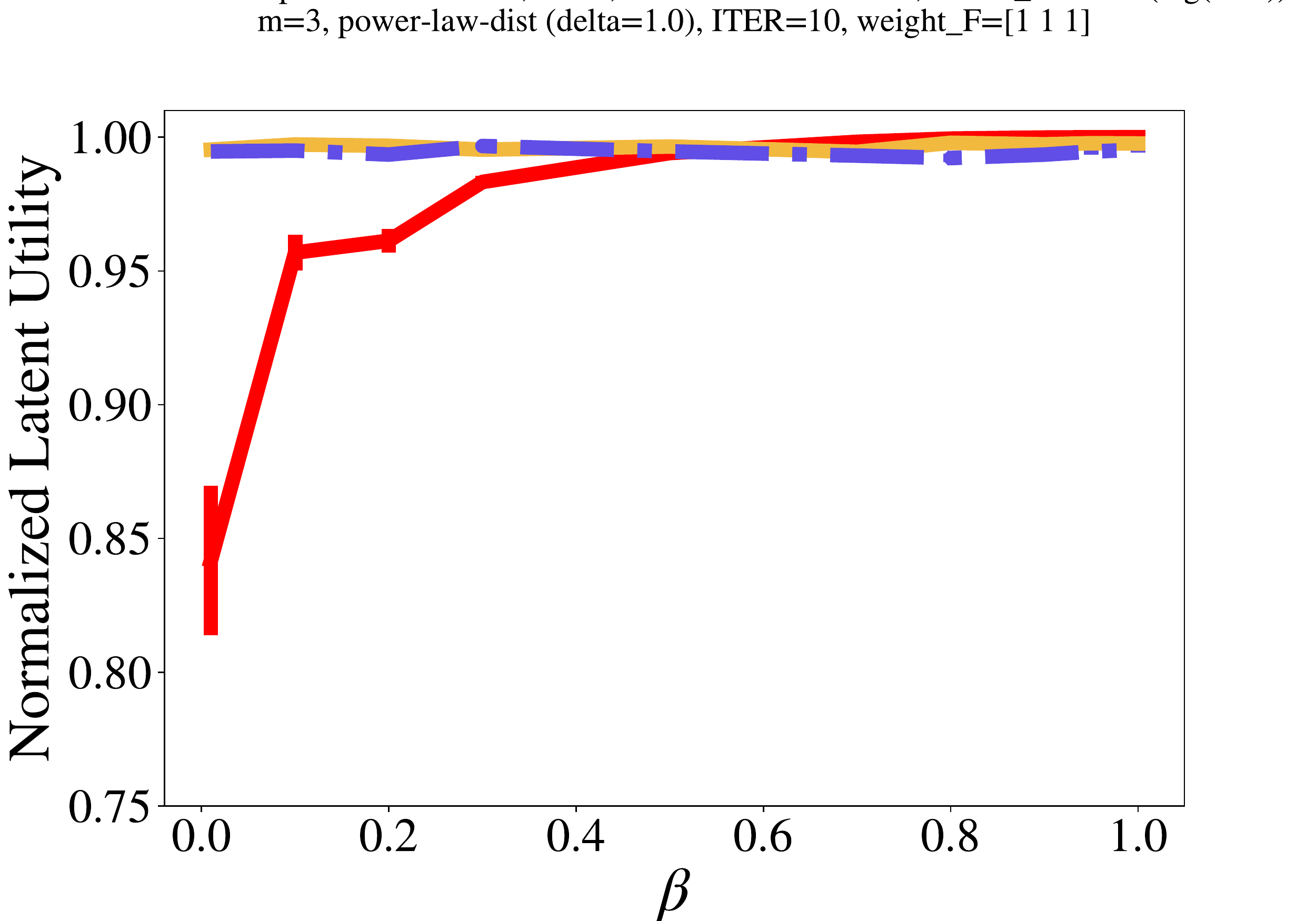}};
              \node[rotate=0] at (2.35 - 1,-1.7) {\textit{\small(less bias)}};
              \node[rotate=0] at (-2.14375 + 1,-1.7) {\textit{\small(more bias)}};
            \end{tikzpicture}
            }
            \subfigure[$\abs{G_1}=(3n)/4$ and $\delta=1$]{
                \begin{tikzpicture}
              \node (image) at (0,-0.17) {\includegraphics[width=\ifconf0.25\linewidth\else0.25\linewidth\fi, trim={0cm 0cm 1.5cm 2cm},clip]{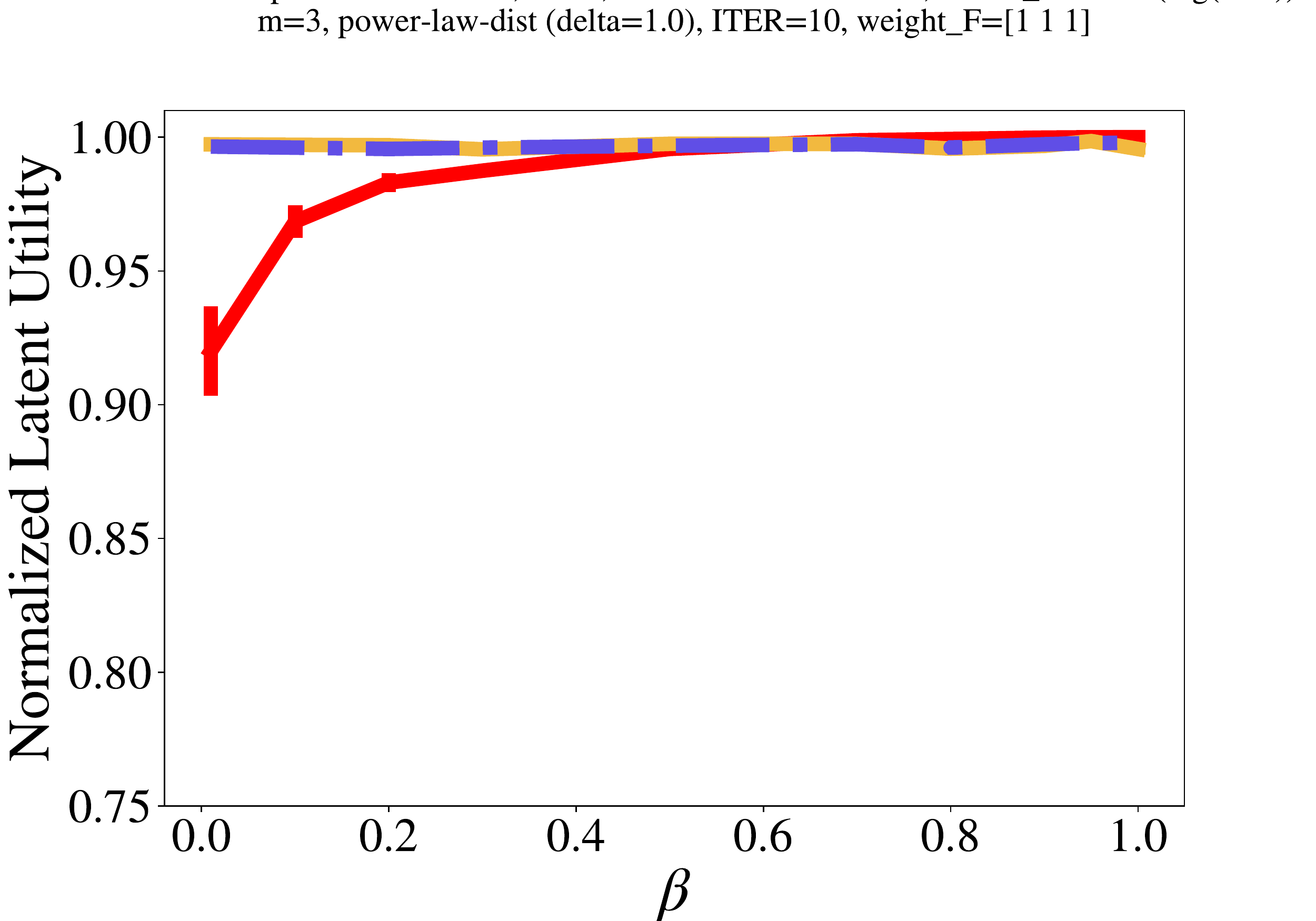}};
              \node[rotate=0] at (2.35 - 1,-1.7) {\textit{\small(less bias)}};
              \node[rotate=0] at (-2.14375 + 1,-1.7) {\textit{\small(more bias)}};
            \end{tikzpicture}
            }
        \ifconf\vspace{-4mm}\else \vspace{-4mm}\fi
        \par
            \subfigure[$\abs{G_1}=n/4$ and $\delta=2$]{
                \begin{tikzpicture}
              \node (image) at (0,-0.17) {\includegraphics[width=\ifconf0.25\linewidth\else0.25\linewidth\fi, trim={0cm 0cm 1.5cm 2cm},clip]{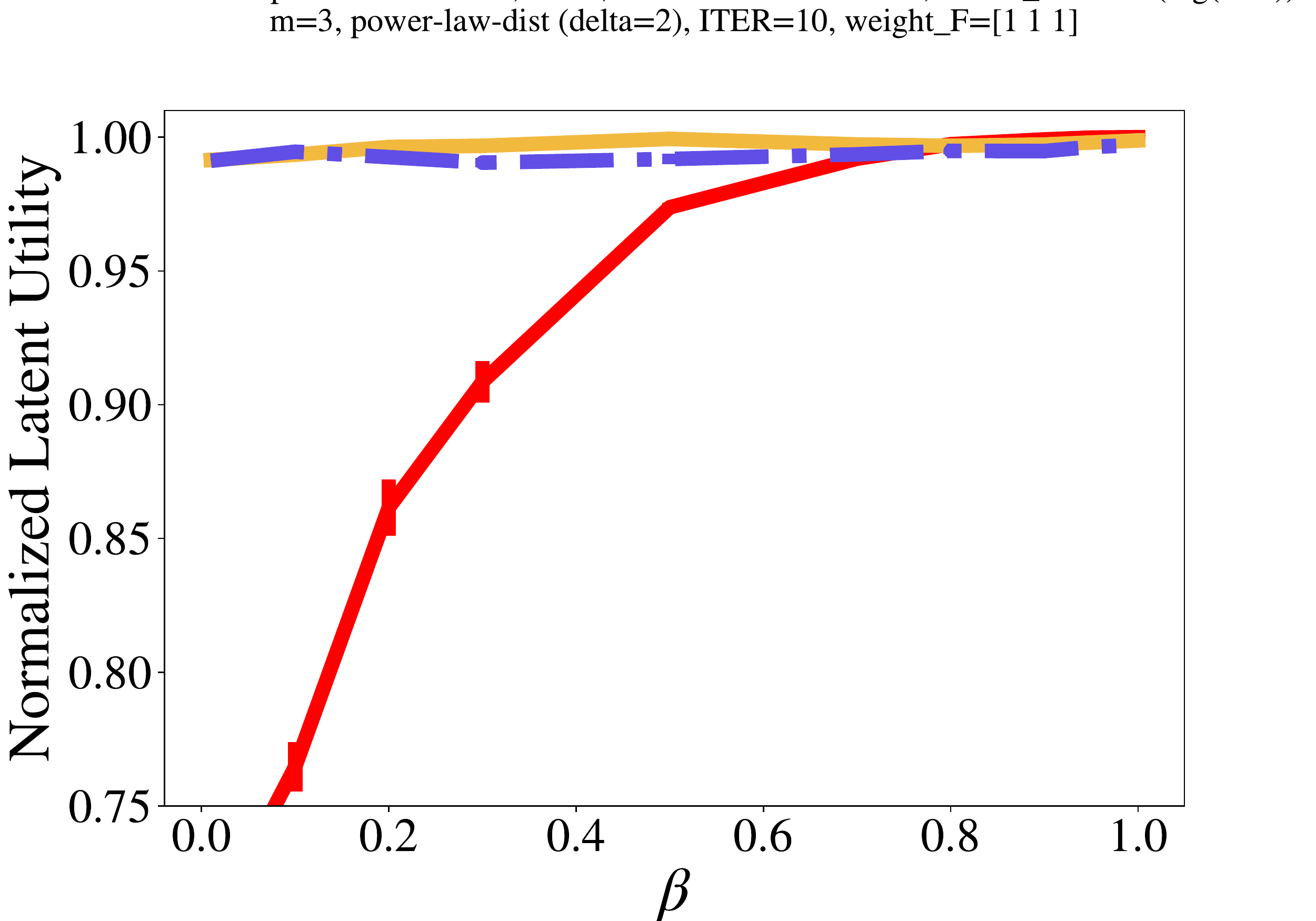}};
              \node[rotate=0] at (2.35 - 1,-1.7) {\textit{\small(less bias)}};
              \node[rotate=0] at (-2.14375 + 1,-1.7) {\textit{\small(more bias)}};
            \end{tikzpicture}
            }
            \subfigure[$\abs{G_1}=n/2$ and $\delta=2$]{
                \begin{tikzpicture}
              \node (image) at (0,-0.17) {\includegraphics[width=\ifconf0.25\linewidth\else0.25\linewidth\fi, trim={0cm 0cm 1.5cm 2cm},clip]{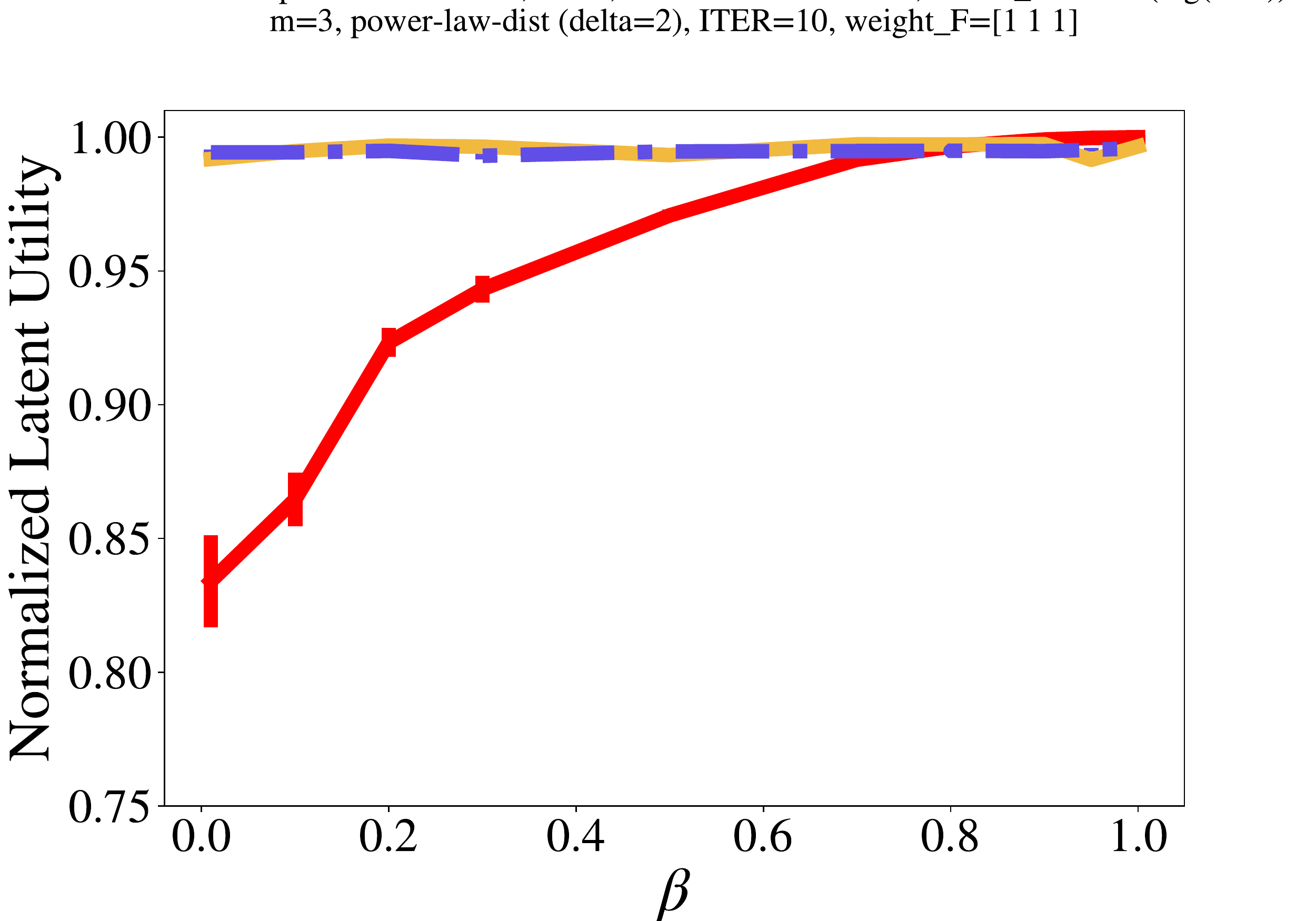}};
              \node[rotate=0] at (2.35 - 1,-1.7) {\textit{\small(less bias)}};
              \node[rotate=0] at (-2.14375 + 1,-1.7) {\textit{\small(more bias)}};
            \end{tikzpicture}
            }
            \subfigure[$\abs{G_1}=(3n)/4$ and $\delta=2$]{
                \begin{tikzpicture}
              \node (image) at (0,-0.17) {\includegraphics[width=\ifconf0.25\linewidth\else0.25\linewidth\fi, trim={0cm 0cm 1.5cm 2cm},clip]{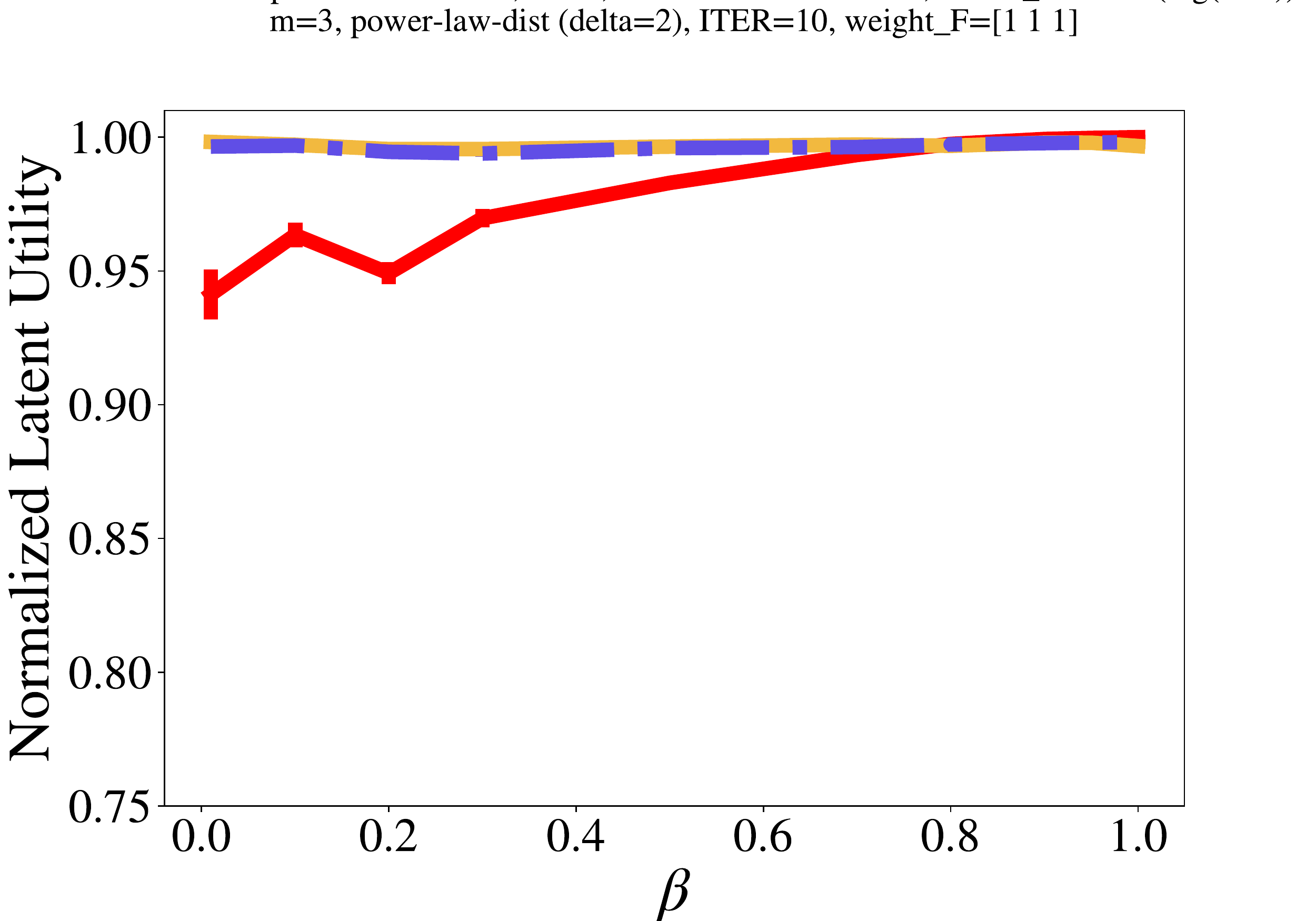}};
              \node[rotate=0] at (2.35 - 1,-1.7) {\textit{\small(less bias)}};
              \node[rotate=0] at (-2.14375 + 1,-1.7) {\textit{\small(more bias)}};
            \end{tikzpicture}
            }
        \ifconf\vspace{-4mm}\else \vspace{-4mm}\fi
        \par
            \subfigure[$\abs{G_1}=n/4$ and $\delta=3$]{
                \begin{tikzpicture}
              \node (image) at (0,-0.17) {\includegraphics[width=\ifconf0.25\linewidth\else0.25\linewidth\fi, trim={0cm 0cm 1.5cm 2cm},clip]{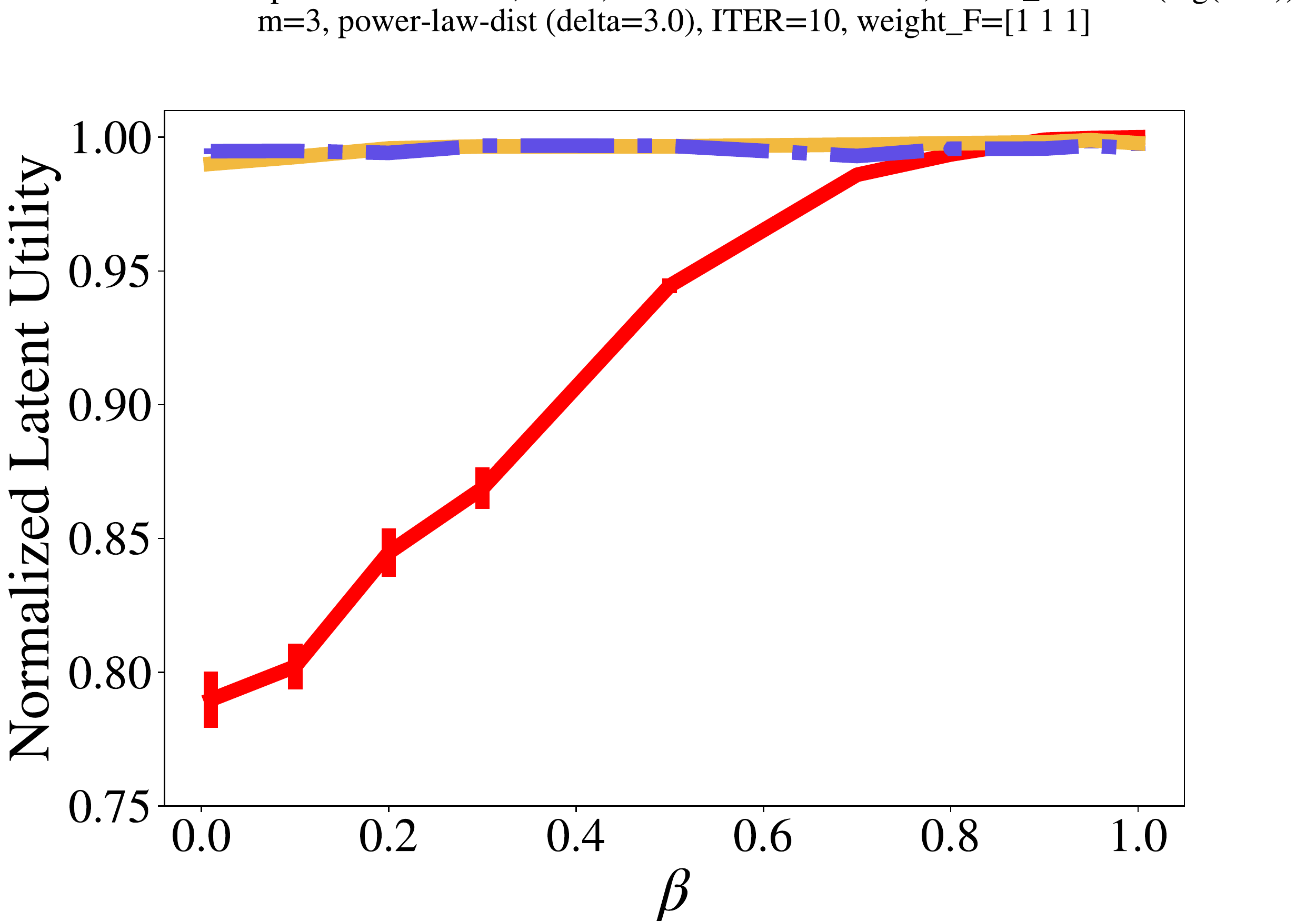}};
              \node[rotate=0] at (2.35 - 1,-1.7) {\textit{\small(less bias)}};
              \node[rotate=0] at (-2.14375 + 1,-1.7) {\textit{\small(more bias)}};
            \end{tikzpicture}
            }
            \subfigure[$\abs{G_1}=n/2$ and $\delta=3$]{
                \begin{tikzpicture}
              \node (image) at (0,-0.17) {\includegraphics[width=\ifconf0.25\linewidth\else0.25\linewidth\fi, trim={0cm 0cm 1.5cm 2cm},clip]{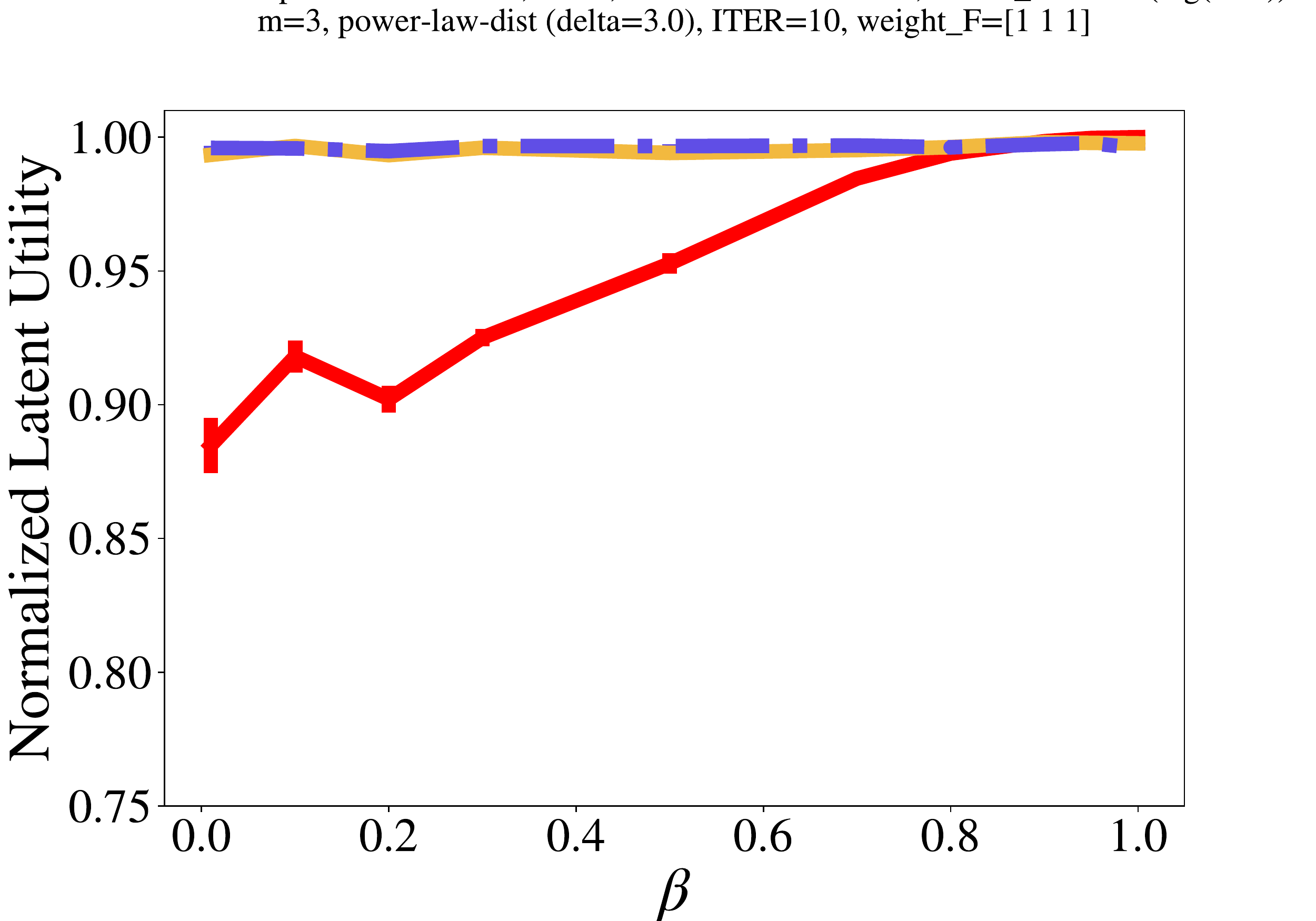}};
              \node[rotate=0] at (2.35 - 1,-1.7) {\textit{\small(less bias)}};
              \node[rotate=0] at (-2.14375 + 1,-1.7) {\textit{\small(more bias)}};
            \end{tikzpicture}
            }
            \subfigure[$\abs{G_1}=(3n)/4$ and $\delta=3$]{
                \begin{tikzpicture}
              \node (image) at (0,-0.17) {\includegraphics[width=\ifconf0.25\linewidth\else0.25\linewidth\fi, trim={0cm 0cm 1.5cm 2cm},clip]{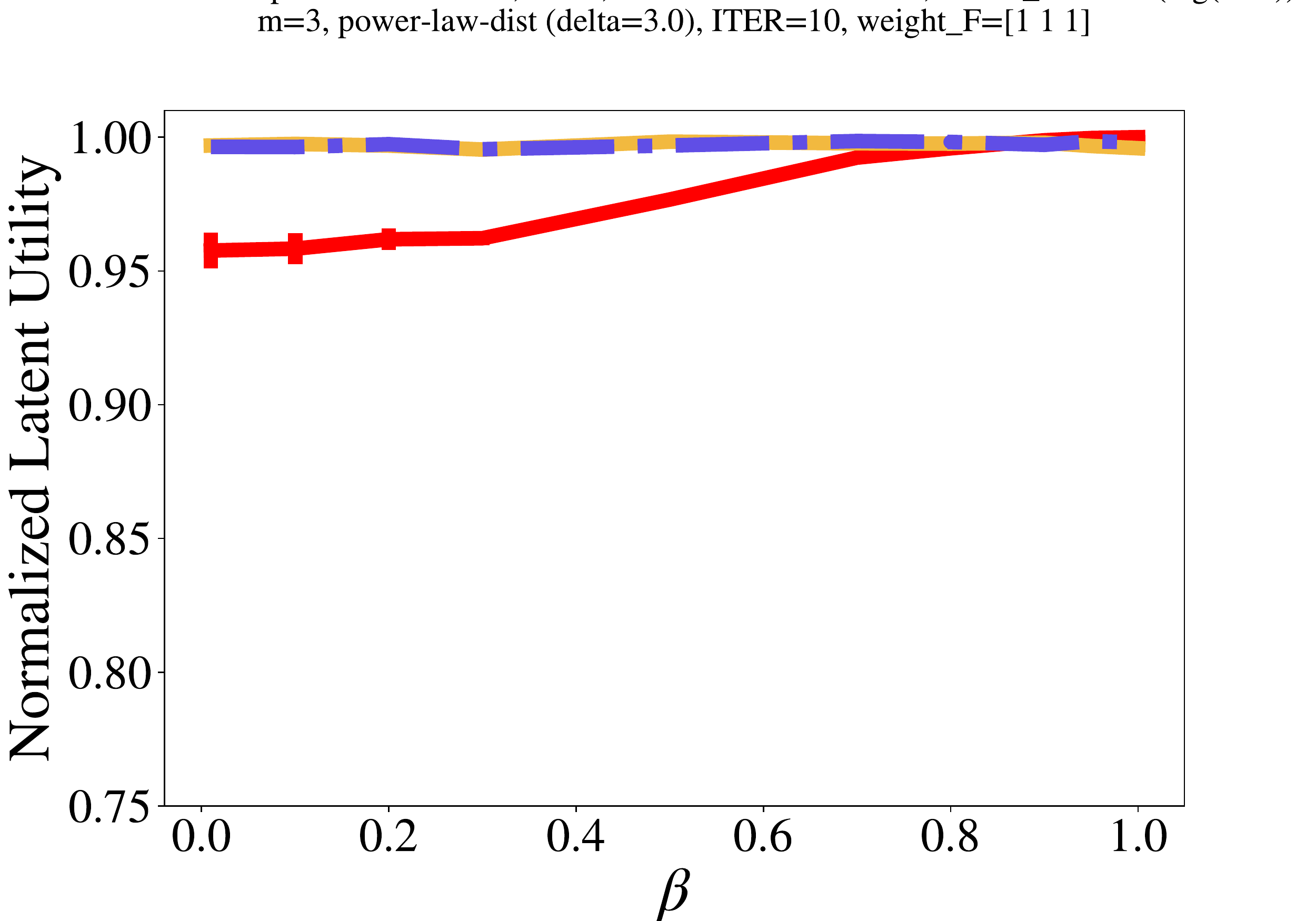}};
              \node[rotate=0] at (2.35 - 1,-1.7) {\textit{\small(less bias)}};
              \node[rotate=0] at (-2.14375 + 1,-1.7) {\textit{\small(more bias)}};
            \end{tikzpicture}
            }
        \caption{
        {\em Simulation on synthetic data 2:}
        {We vary the bias parameter $\beta\in [0,1]$, run \cref{alg:disj} and \uncons{} on synthetic data 2 and report their normalized latent  utility (error bars denote standard error of the mean).}
        The $x$-axis shows $\beta\in [0,1]$ and the $y$-axis plots the normalized latent utility.
        The results show that \uncons{} can lose a significant fraction of the optimal latent utility in the presence of bias (up to 25\% for $\beta<0.1$, $\frac{\abs{G_1}}{n}= \frac{1}{4}$, and $\delta\leq \frac{5}{2}$).
        {While \cref{alg:disj} loses less than 1\% of the optimal latent utility across all choices of parameters.}
        }
        \label{fig:syn_iid_disj_full}
    \end{figure} 

    \newpage

    \subsection{Additional plots for simulation with MovieLens 20M data}\label{sec:additional_results_real_world_data}

    \midsepremove{}
    \begin{table}[h!]
        \centering
        \small \vspace{-2mm} 
        \begin{tabular}{l|c|c|c}
             {\bf Genre}          &   {\bf Male led}  &  {\bf Non-Male led} &   {\bf Ratio} $\inparen{\frac{\text{Non-Male led}}{\text{Male led}}}$\\[2mm]
            \midrule
            Action        &   0.0604     &   0.0213   &  0.352\\
            Adventure       &   0.0334   &   0.0146   &  0.437\\
            Animation       &   0.0187   &   0.0138   &  0.741\\
            Children      &   0.0168     &   0.0164   &  0.9779\\
            Comedy        &   0.0929     &   0.0646   &  0.6951\\
            Crime         &   0.0311     &   0.0135   &  0.4351\\
            Documentary     &   0.0157   &   0.0141   &  0.903\\
            Drama         &   0.1035     &   0.1238   &  1.1961\\
            \white{.}\\
        \end{tabular}
        \par 
        \begin{tabular}{l|c|c|c}
             {\bf Genre}          &   {\bf Male led}  &  {\bf Non-Male led} &   {\bf Ratio} $\inparen{\frac{\text{Non-Male led}}{\text{Male led}}}$\\[2mm]
            \midrule
            Fantasy       &   0.0165     &   0.0145   &  0.8798\\
            Horror        &   0.0293     &   0.0489   &  1.6675\\
            Musical       &   0.0105     &   0.0162   &  1.5458\\
            Mystery       &   0.0128     &   0.0123   &  0.9612\\
            Romance       &   0.0306     &   0.0689   &  2.2492\\
            Sci-fi        &   0.0287     &   0.0169   &  0.5878\\
            Thriller      &   0.0378     &   0.0316   &  0.8368\\
            War           &   0.014    &     0.0061   &  0.4371\\
            Western       &   0.0093     &   0.0027   &  0.2923
        \end{tabular}
        \vspace{2mm}
        \caption{
                {\em Average relevance scores for movies which are led by Male and Non-Male actors respectively, across different genres.}
                Let $S_{\rm M}$ be the set of movies led by male actors and $S_{\rm NM}$ be the set of movies led by non-male actors (see \cref{sec:empirical_results} for details about computing $S_{\rm M}$ and $S_{\rm NM}$).
                The MovieLens 20M data specifies sets the $S_g$ of movies in genre $g$ and for each movie $i$ and genre $g$, it specifies a predicted relevance score $r_{ig}\in [0,1]$ indicating ``how strongly [movie $i$] exhibits particular properties represented by [genre $g$].''
                For each genre $g$, we report the average relevance scores of movies in $S_{\rm M}\cap S_g$ and $S_{\rm NM}\cap S_g$
                We observe that in genres that are stereotypically associated with men (e.g., ``Action'' or ``War'') movies led by male actors have disproportionately higher relevance-scores on compared to movies led by non-male actors (differing by up to 300\%). 
        }
        \label{table:rel_score_movie}
        \vspace{-15mm}
    \end{table}
    \midsepadd{}

    \newpage
    
    \begin{table}[h!]
        \centering
        \small
        \ifconf\else\vspace{-2mm}\fi
        \begin{tabular}{l|c|c|c}
             {\bf Genre}          &   {\bf Male led}  &  {\bf Non-Male led} &   {\bf Ratio} $\inparen{\frac{\text{Non-Male led}}{\text{Male led}}}$\\
            \midrule
            Action      &  3.11 (0.53)  &   3.04 (0.54)  &   0.98\\
            Adventure   &  3.22 (0.51)  &   3.07 (0.60)  &   0.95\\
            Animation   &  3.32 (0.45)  &   3.43 (0.49)  &   1.03\\
            Children    &  3.03 (0.54)  &   3.18 (0.53)  &   1.05\\
            Comedy      &  3.16 (0.53)  &   3.15 (0.48)  &   1.00\\
            Crime       &  3.36 (0.45)  &   3.24 (0.50)  &   0.96\\
            Documentary &  3.61 (0.41)  &   3.59 (0.34)  &    1.00\\
            Drama       &  3.44 (0.40)  &   3.40 (0.40)  &   0.99\\
            Fantasy     &  3.22 (0.52)  &   3.24 (0.48)  &   1.01\\
            Horror      &  2.88 (0.57)  &   2.88 (0.53)  &   1.00\\
            Musical     &  3.36 (0.43)  &   3.30 (0.46)  &   0.98\\
            Mystery     &  3.41 (0.44)  &   3.21 (0.50)  &   0.94\\
            Romance     &  3.37 (0.43)  &   3.31 (0.45)  &   0.98\\
            Sci-fi      &  3.11 (0.57)  &   3.07 (0.57)  &   0.99\\
            Thriller    &  3.25 (0.47)  &   3.12 (0.49)  &   0.96\\
            War         &  3.51 (0.45)  &   3.61 (0.31)  &   1.03\\
            Western     &  3.38 (0.41)  &   3.36 (0.34)  &   0.99\\
        \end{tabular}
        \caption{
                {\em Average user-ratings for movies which are led by Male and Non-Male actors respectively, across different genres.}
                Let $S_{\rm M}$ be the set of movies led by male actors and $S_{\rm NM}$ be the set of movies led by non-male actors (see \cref{sec:empirical_results} for details about computing $S_{\rm M}$ and $S_{\rm NM}$).
                The MovieLens 20M data specifies sets the $S_g$ of movies in genre $g$ and for each movie $i$, it specifies the average user rating ${\rm rat}_i\in [0,1]$ for movie $i$.
                For each genre $g$, we report the average user rating of movies in $S_{\rm M}\cap S_g$ and $S_{\rm NM}\cap S_g$
                We observe that in genres that are stereotypically associated with men (e.g., ``Action'' or ``War'') movies led by male actors have disproportionately higher relevance-scores on compared to movies led by non-male actors (differing by up to 300\%).
        }
        \ifconf\else\vspace{-8mm}\fi
        \label{table:user_rating_per_genre}
    \end{table}
    \begin{table}[h!]
        \centering
        \small
        \begin{tabular}{l|c|c|c}
             {\bf Genre}          &   {\bf Male led}  &  {\bf Non-Male led} &   {\bf Ratio} $\inparen{\frac{\text{Non-Male led}}{\text{Male led}}}$\\
            \midrule
            Action        &   0.061     &  0.0274       &   0.4486\\
            Adventure     &   0.0331        & 0.0168        &   0.5079\\
            Animation     &   0.0185        & 0.0161        &   0.8704\\
            Children      &   0.0169        & 0.0171        &   1.0114\\
            Comedy        &   0.0937        & 0.0719        &   0.7679\\
            Crime         &   0.0312        & 0.0148        &   0.475\\
            Documentary     &   0.0157      &  0.0126       &    0.8029\\
            Drama         &   0.1042        & 0.1181        &   1.1333\\
            Fantasy       & 0.0162      & 0.0154        &   0.945\\
            Horror        & 0.0287      & 0.0492        &   1.7129\\
            Musical       & 0.0107      & 0.0148        &   1.3819\\
            Mystery       & 0.0126      & 0.0118        &   0.9346\\
            Romance       & 0.0306      & 0.0653        &   2.1365\\
            Sci-fi        & 0.0285      & 0.0179        &   0.6298\\
            Thriller      & 0.0372      & 0.0334        &   0.8988\\
            War           & 0.0141      & 0.0063        &   0.4467\\
            Western       & 0.0091      & 0.0027        &   0.3024\\
        \end{tabular}
        \vspace{4mm}
        \caption{
            {\em Average user-ratings for movies which are led by Male and Non-Male actors respectively, across different genres.}
            For each movie $i$, we predict the (probable) gender of its lead actor using the Genderize API (\url{gender-api.com}).
            For the main simulation, we remove all movies for which this prediction has a confidence of less than 0.9.
            Here, we consider the set of all movies where Genderize outputs a prediction that is not NA, and for which we have both user ratings and relevance scores.
            This results in 7382 movies $S_{\rm M}'$ led by male actors and 2572 movies $S_{\rm NM}'$ led by non-male actors.
            The MovieLens 20M data specifies sets the $S_g$ of movies in genre $g$ and for each movie $i$, it specifies the average user rating ${\rm rat}_i\in [0,1]$ for movie $i$.
            For each genre $g$, we report the average user rating of movies in $S_{\rm M}\cap S_g$ and $S_{\rm NM}\cap S_g$
            We observe that in genres that are stereotypically associated with men (e.g., ``Action'' or ``War'') movies led by male actors have disproportionately higher relevance-scores on  compared to movies led by non-male actors (differing by up to 300\%).
        }
        \vspace{-5mm}
        \label{table:rel_score_movie_all}
    \end{table}

    \newpage \white{.} \newpage

    \subsubsection{Plots with one genre ($\abs{T}=1$)} \white{...A}

    \renewcommand{\folder}{./figures/real-world-data/1}
    \begin{figure}[h!]
        \centering
        \subfigure[\white{.} $T=\inbrace{\texttt{Action}}$]{
            \begin{tikzpicture}
          \node (image) at (0,-0.17) {\includegraphics[width=\ifconf0.25\linewidth\else0.29\linewidth\fi, trim={0cm 0cm 0cm 0cm},clip]{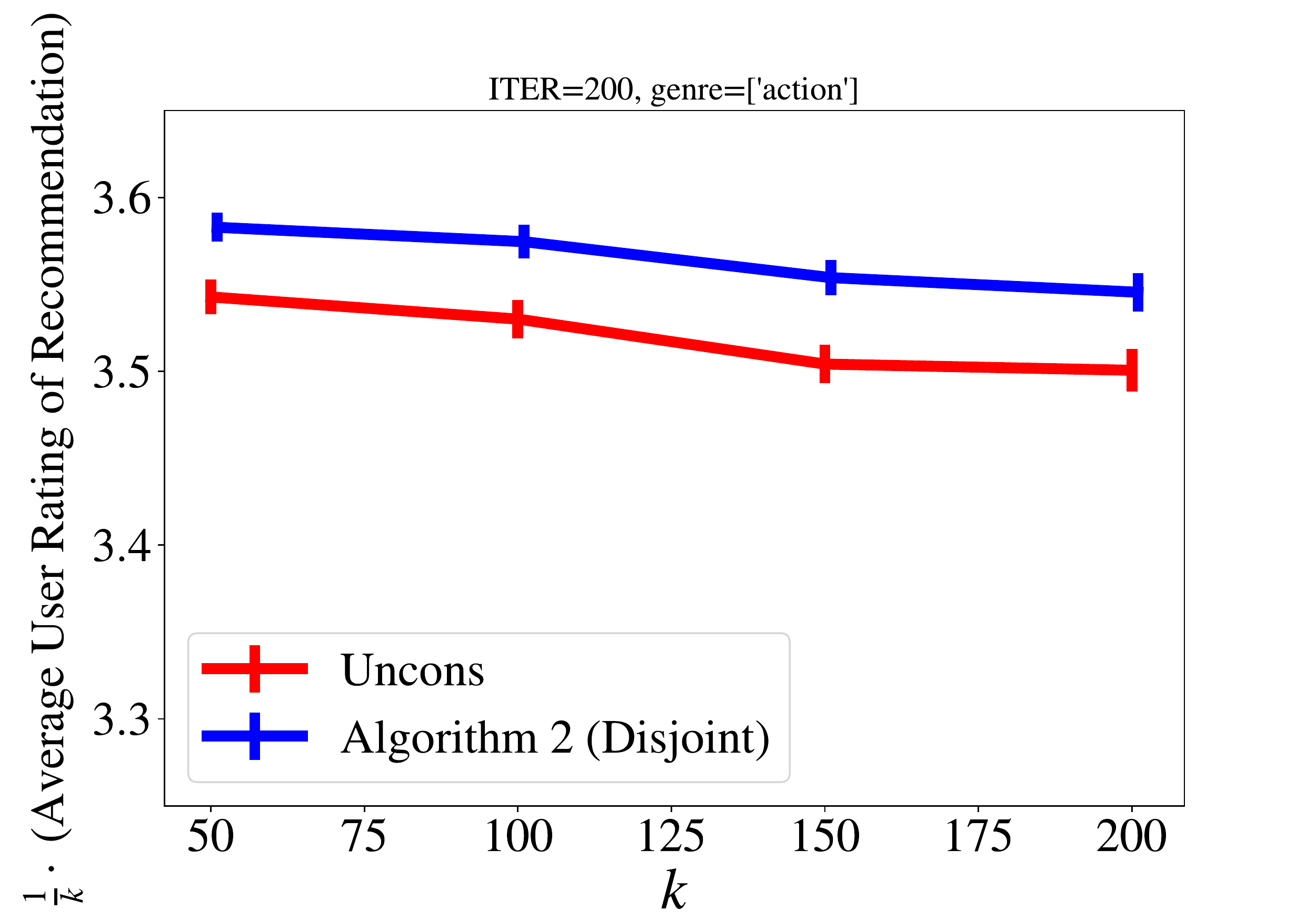}};
          \node[rotate=0,fill=white] at (0.15,-1.75){$k$};
        \end{tikzpicture}
        }
        \subfigure[\white{.} $T=\inbrace{\texttt{Adventure}}$]{
            \begin{tikzpicture}
          \node (image) at (0,-0.17) {\includegraphics[width=\ifconf0.25\linewidth\else0.29\linewidth\fi, trim={0cm 0cm 0cm 0cm},clip]{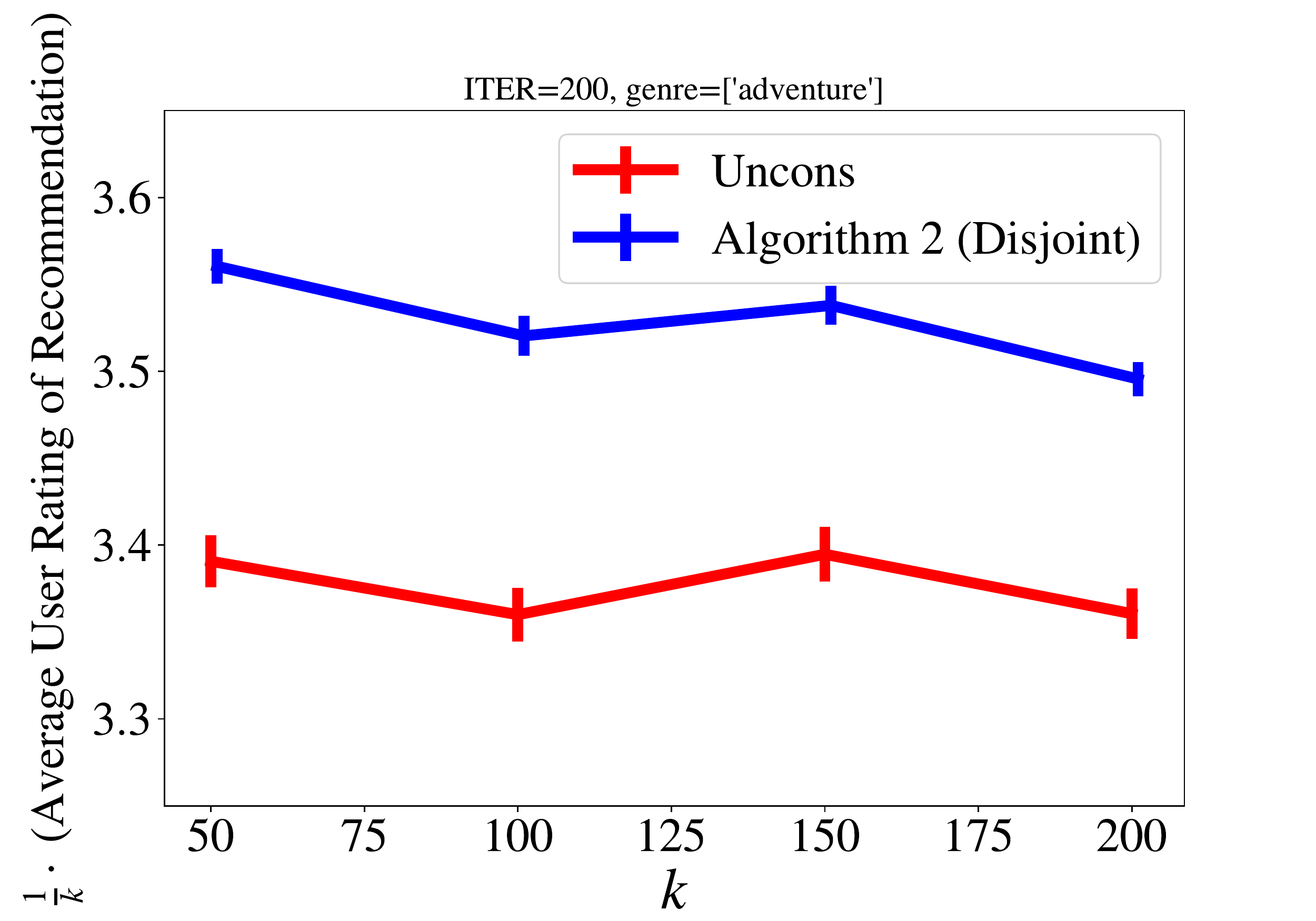}};
          \node[rotate=0,fill=white] at (0.15,-1.75){$k$};
        \end{tikzpicture}
        }
        \par%
        \subfigure[\white{.} $T=\inbrace{\texttt{Crime}}$]{
            \begin{tikzpicture}
          \node (image) at (0,-0.17) {\includegraphics[width=\ifconf0.25\linewidth\else0.29\linewidth\fi, trim={0cm 0cm 0cm 0cm},clip]{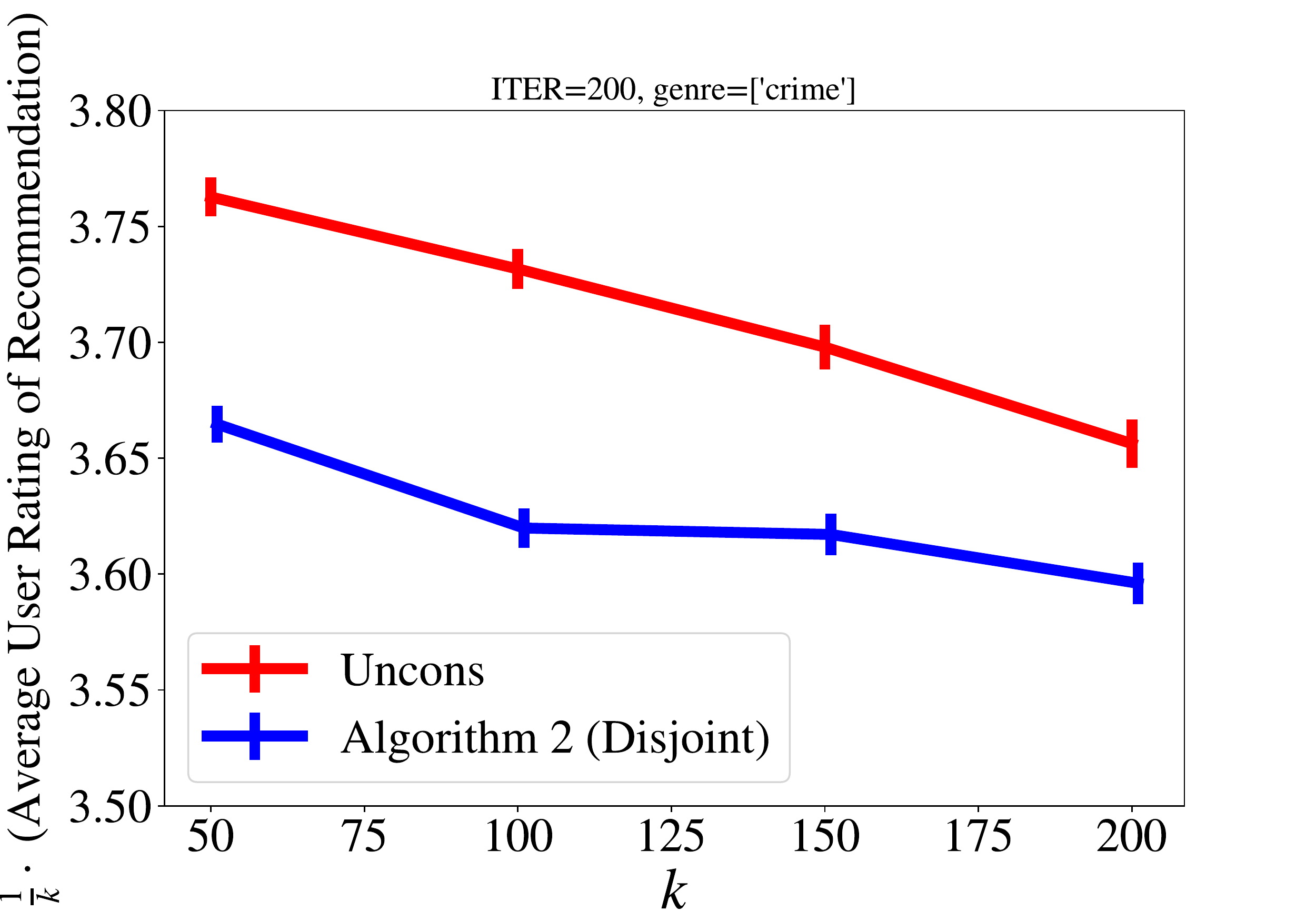}};
          \node[rotate=0,fill=white] at (0.15,-1.75){$k$};
        \end{tikzpicture}
        }
        \subfigure[\white{.} $T=\inbrace{\texttt{War}}$]{
            \begin{tikzpicture}
          \node (image) at (0,-0.17) {\includegraphics[width=\ifconf0.25\linewidth\else0.29\linewidth\fi, trim={0cm 0cm 0cm 0cm},clip]{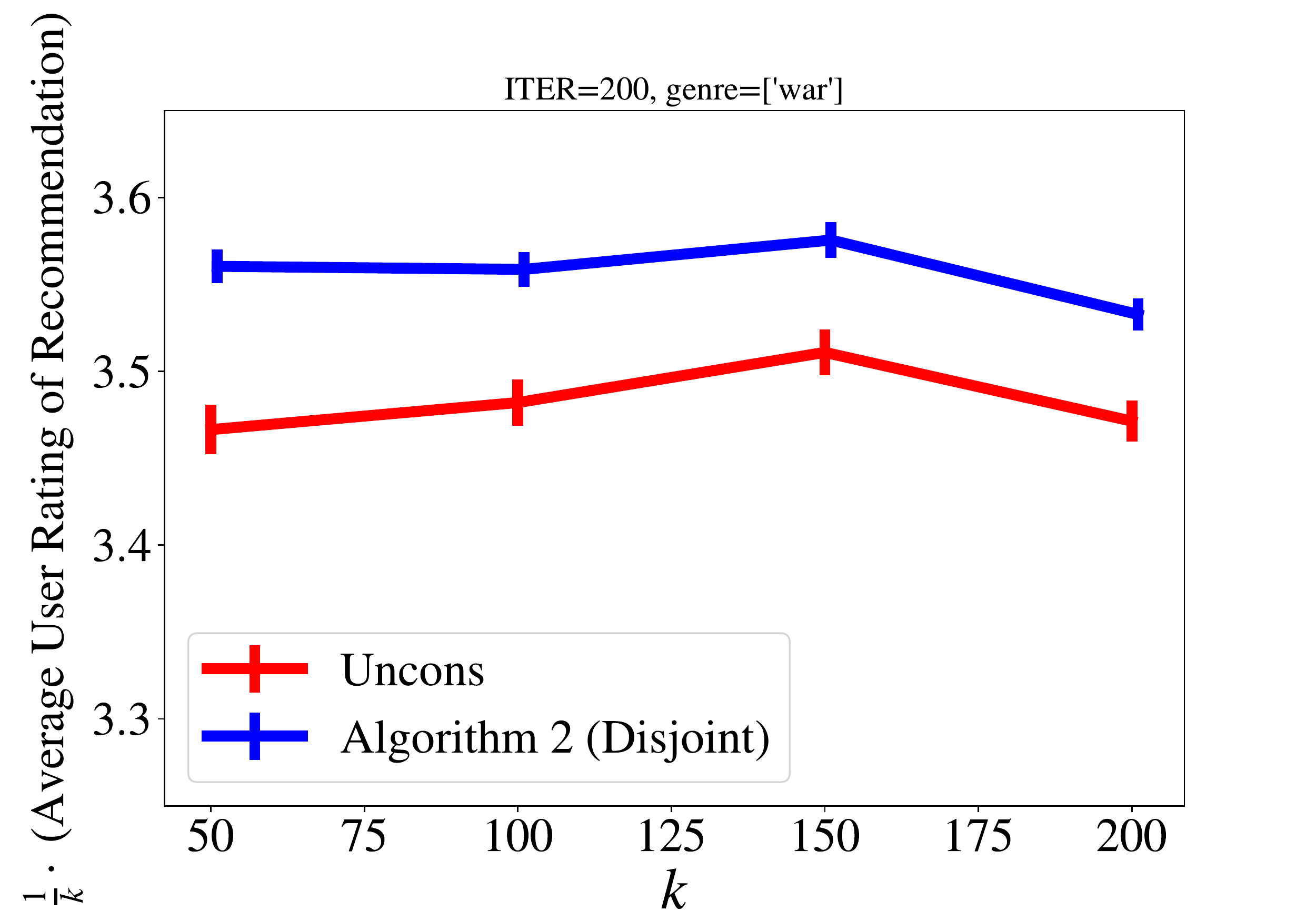}};
          \node[rotate=0,fill=white] at (0.15,-1.75){$k$};
        \end{tikzpicture}
        }
        \subfigure[\white{.} $T=\inbrace{\texttt{Western}}$]{
            \begin{tikzpicture}
          \node (image) at (0,-0.17) {\includegraphics[width=\ifconf0.25\linewidth\else0.29\linewidth\fi, trim={0cm 0cm 0cm 0cm},clip]{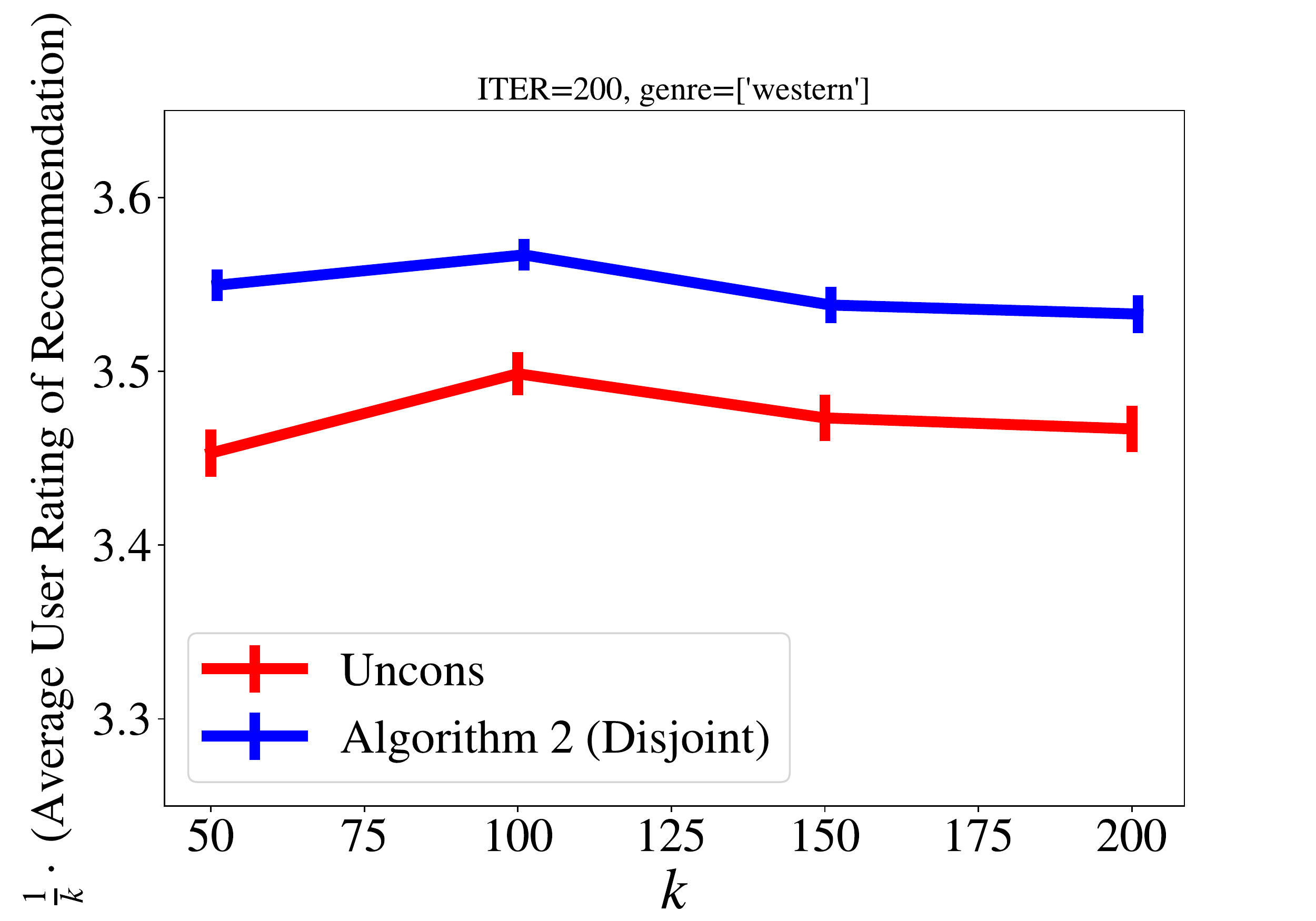}};
          \node[rotate=0,fill=white] at (0.15,-1.75){$k$};
        \end{tikzpicture}
        }
        \caption{
        {\em Simulation with MovieLens data with movie recommendations from a single men-stereotypical genre:}
        We observe that the relevance scores in the MovieLens data are disproportionately higher (by up to 3 times) for movies led by male actors compared to movies led by non-male actors in genres stereotypically associated with men.
        In contrast, user ratings for these sets of movies are within 6\% of each other in all genres.
        We chose genres where the ratio of average relevance scores of men-led movies is at least twice that of non-men-led movies, they are:
        $B=\{$\texttt{action}, \texttt{adventure}, \texttt{crime}, \texttt{western}, and \texttt{war}$\}$.
        We use relevance scores to recommend $k\in \inbrace{50, 100, 150, 200}$ movies from different subsets $T$ of $B$.
        This figure presents the results for all subsets $T\subseteq B$ of size 1.
        We observe that in 4 out of 5 subfigures \cref{alg:disj} has a higher normalized latent utility than \uncons{} for all $k$ (by up to 5\%). 
        {\em Note that the range of $y$-axis varies across sugfigures and is either $[3.25, 3.65]$ or $[3.5,3.8]$.}
        } 
        \label{fig:movie_lens_pairs_action}
    \end{figure}

    \newpage
    \subsubsection{Plots with two genres ($\abs{T}=2$)} 
    \renewcommand{\folder}{./figures/real-world-data/2}
    \begin{figure}[b!]
        \centering 
        \vspace{-8mm}
        \subfigure[\white{.} $T=\inbrace{\texttt{Action}, \texttt{Adventure}}$]{
            \begin{tikzpicture}
          \node (image) at (0,-0.17) {\includegraphics[width=\ifconf0.25\linewidth\else0.29\linewidth\fi, trim={0.2cm 0cm 1.5cm 2cm},clip]{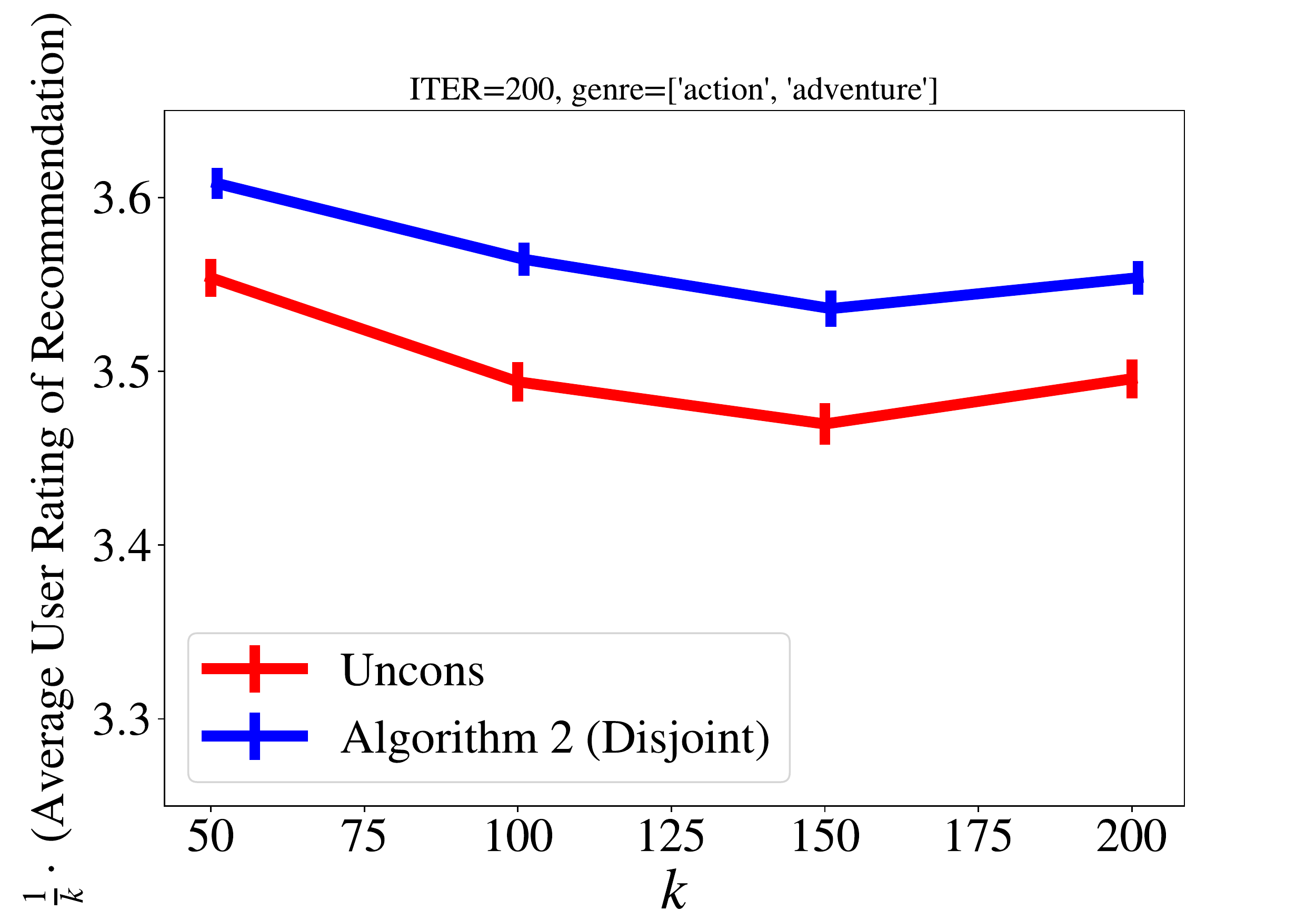}};
        \end{tikzpicture}
        }
        \subfigure[\white{.} $T=\inbrace{\texttt{Action}, \texttt{Crime}}$]{
            \begin{tikzpicture}
          \node (image) at (0,-0.17) {\includegraphics[width=\ifconf0.25\linewidth\else0.29\linewidth\fi, trim={0.2cm 0cm 1.5cm 2cm},clip]{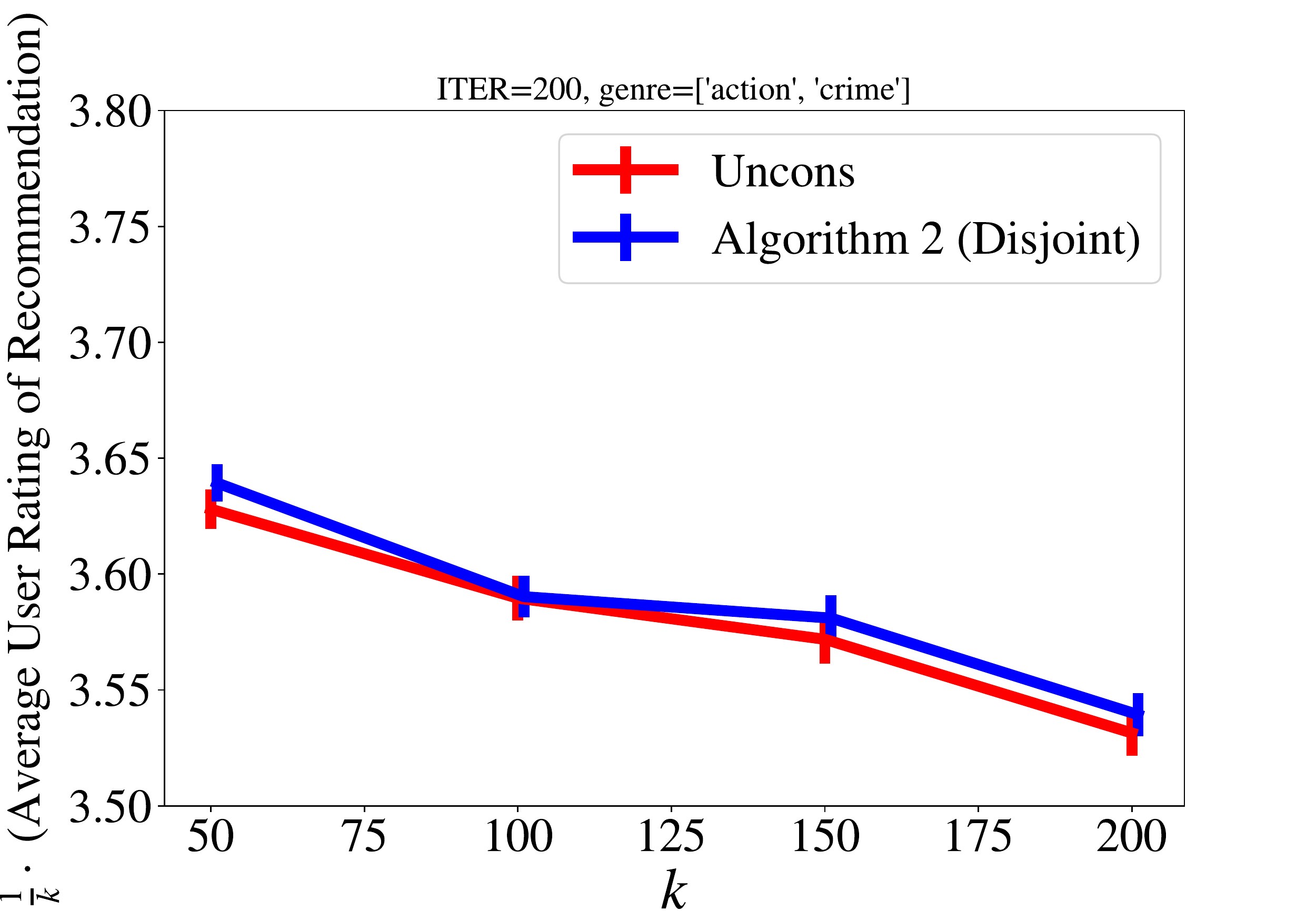}};
        \end{tikzpicture}
        }
        \subfigure[\white{.} $T=\inbrace{\texttt{Action}, \texttt{War}}$]{
            \begin{tikzpicture}
          \node (image) at (0,-0.17) {\includegraphics[width=\ifconf0.25\linewidth\else0.29\linewidth\fi, trim={0.2cm 0cm 1.5cm 2cm},clip]{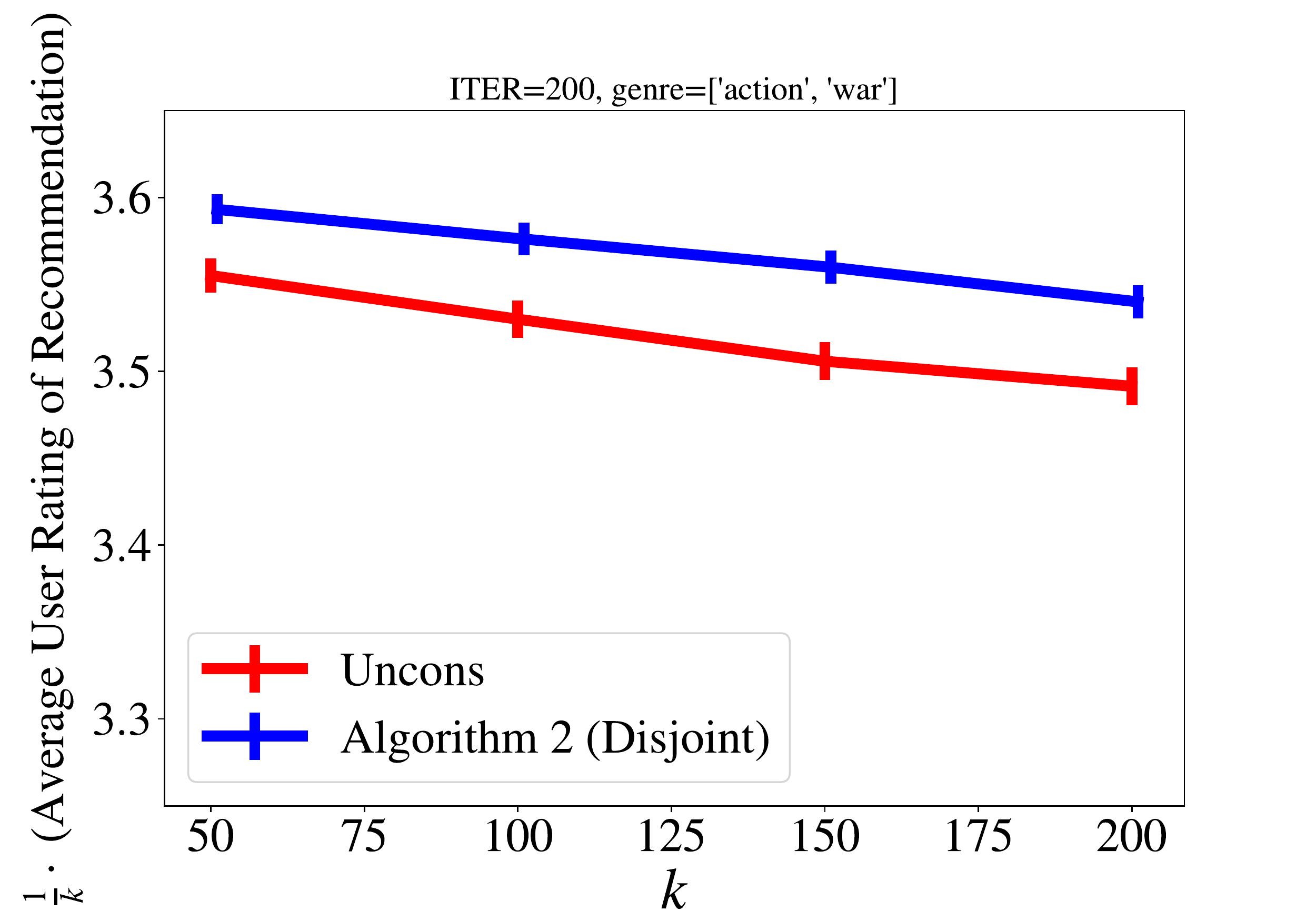}};
        \end{tikzpicture}
        }
        \ifconf\else\vspace{-4mm}\fi
        \par
        \subfigure[\white{.} $T=\inbrace{\texttt{Action}, \texttt{Western}}$]{
            \begin{tikzpicture}
          \node (image) at (0,-0.17) {\includegraphics[width=\ifconf0.25\linewidth\else0.29\linewidth\fi, trim={0.2cm 0cm 1.5cm 2cm},clip]{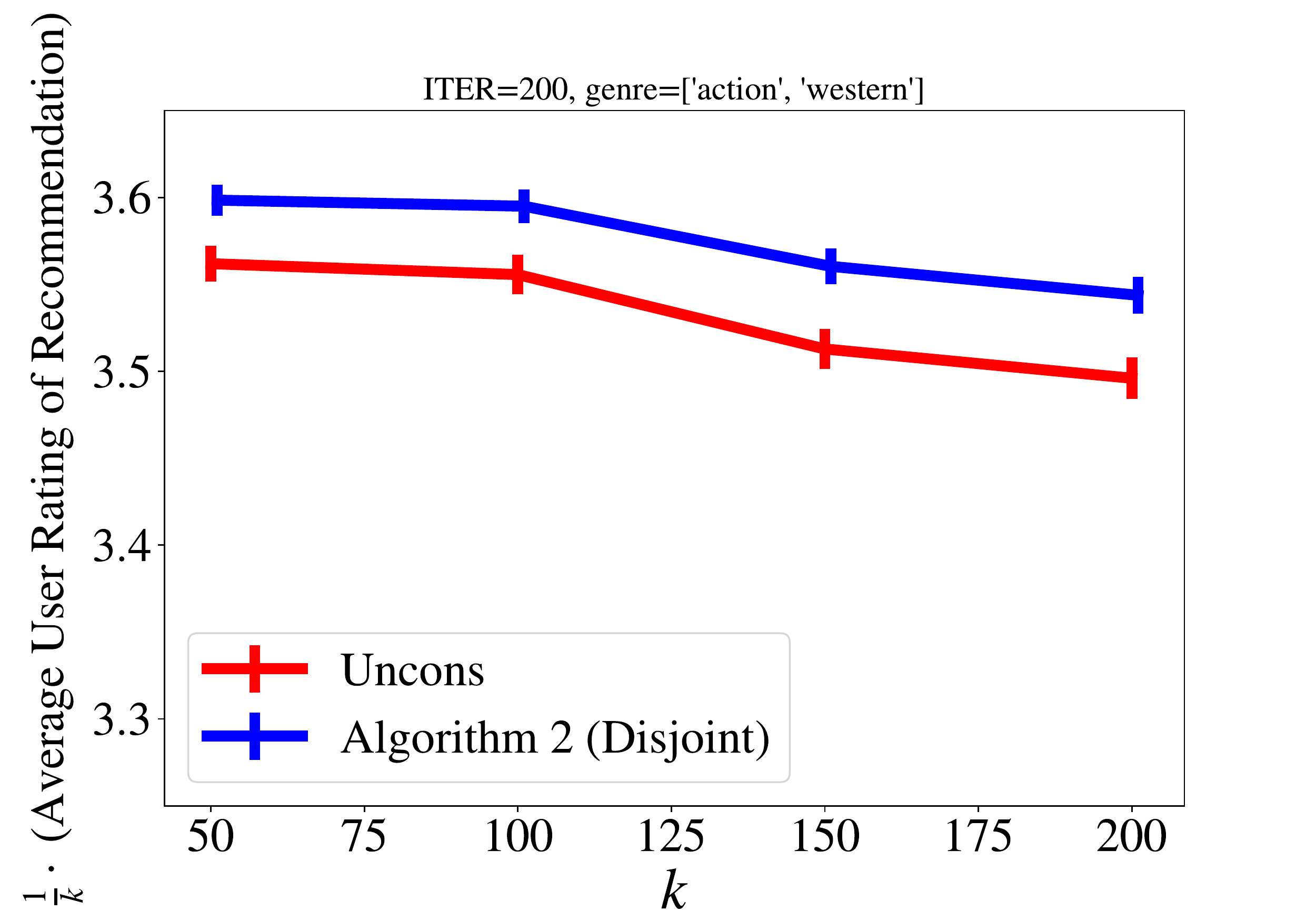}};
        \end{tikzpicture}
        }
        \subfigure[\white{.} $T=\inbrace{\texttt{Adventure}, \texttt{Crime}}$]{
            \begin{tikzpicture}
          \node (image) at (0,-0.17) {\includegraphics[width=\ifconf0.25\linewidth\else0.29\linewidth\fi, trim={0.2cm 0cm 1.5cm 2cm},clip]{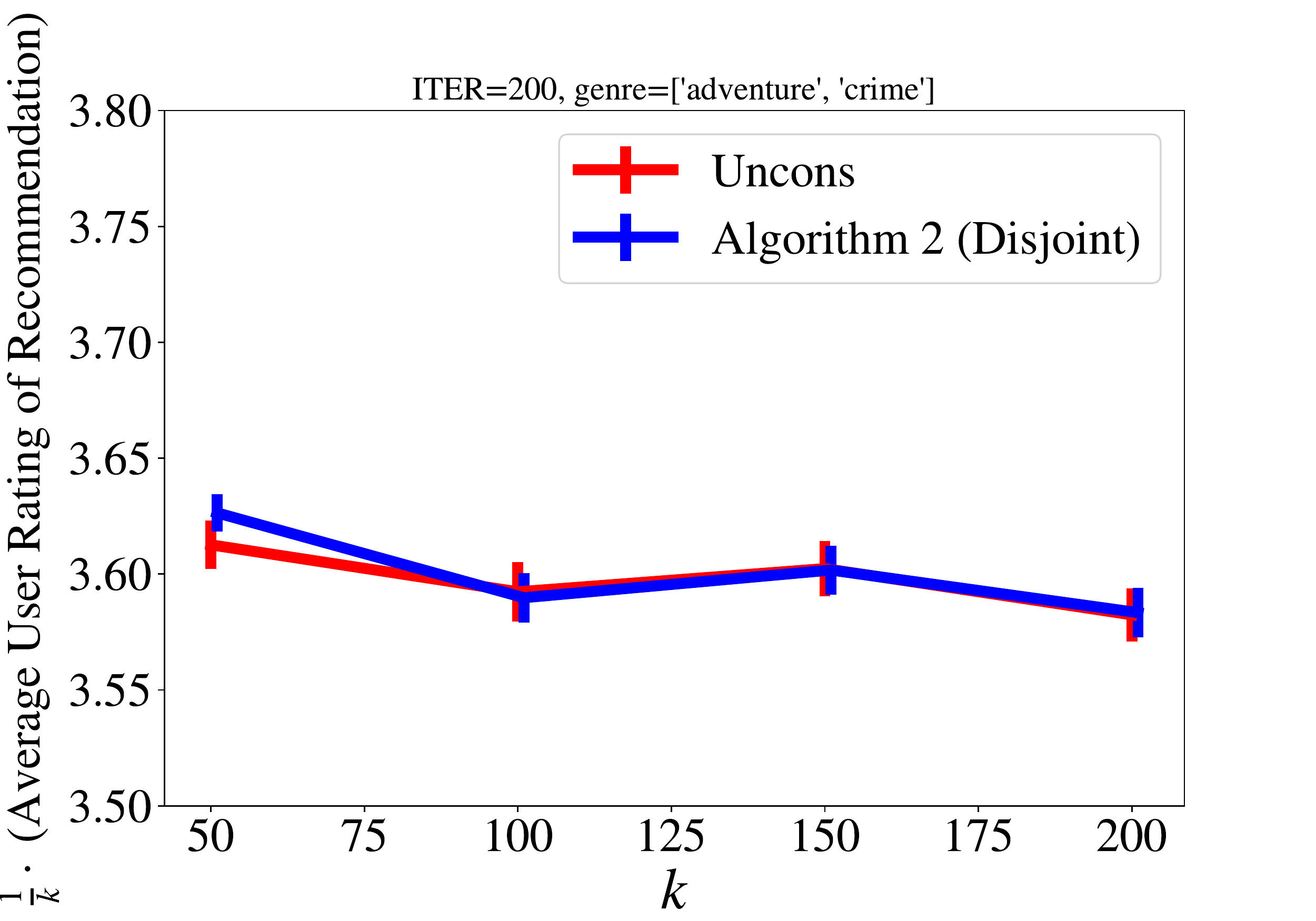}};
        \end{tikzpicture}
        }
        \subfigure[\white{.} $T=\inbrace{\texttt{Adventure}, \texttt{War}}$ ]{
            \begin{tikzpicture}
          \node (image) at (0,-0.17) {\includegraphics[width=\ifconf0.25\linewidth\else0.29\linewidth\fi, trim={0.2cm 0cm 1.5cm 2cm},clip]{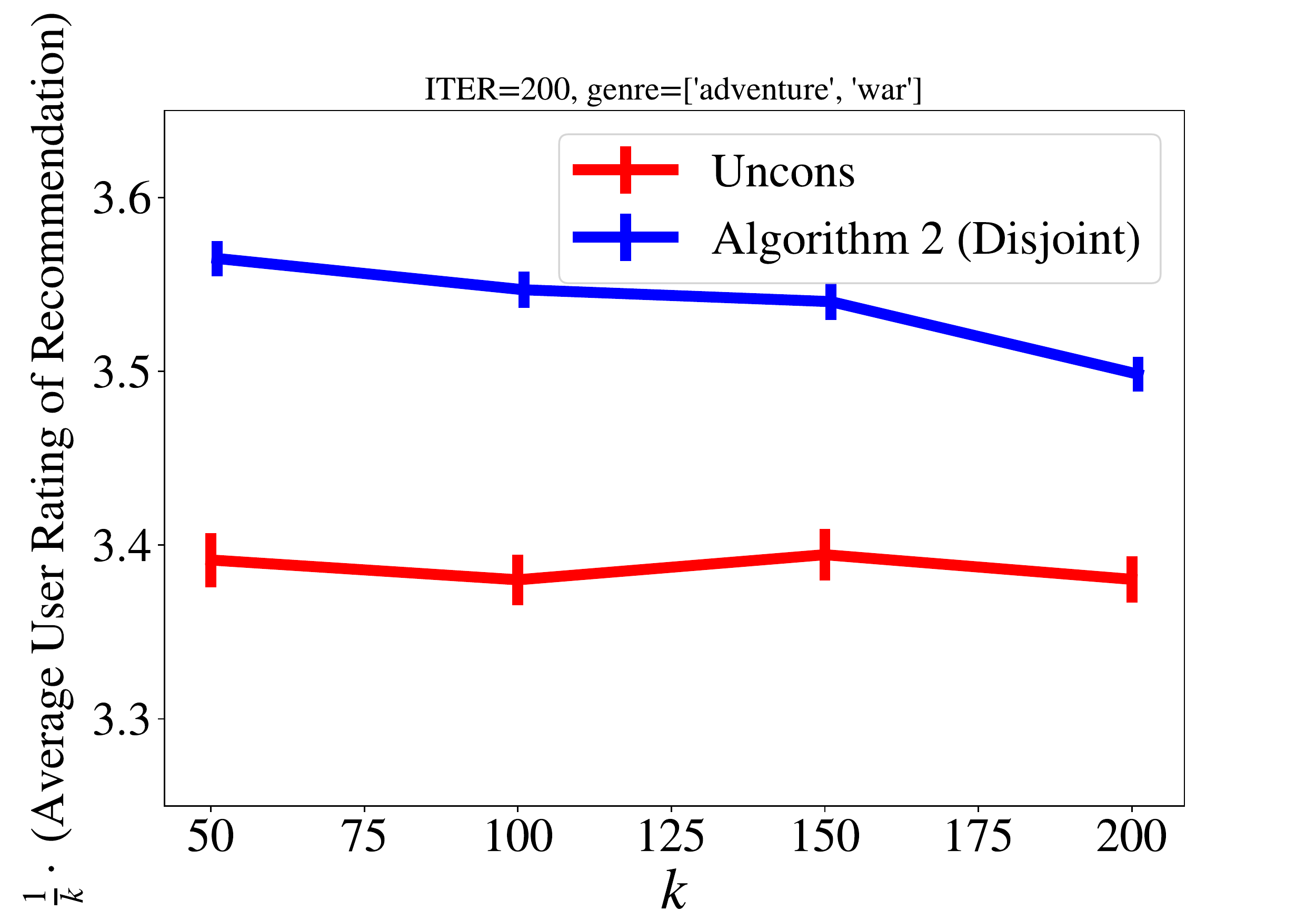}};
        \end{tikzpicture}
        }
        \ifconf\else\vspace{-4mm}\fi
        \par
        \subfigure[\white{.} $T=\inbrace{\texttt{Adventure}, \texttt{Western}}$]{
            \begin{tikzpicture}
          \node (image) at (0,-0.17) {\includegraphics[width=\ifconf0.25\linewidth\else0.29\linewidth\fi, trim={0.2cm 0cm 1.5cm 2cm},clip]{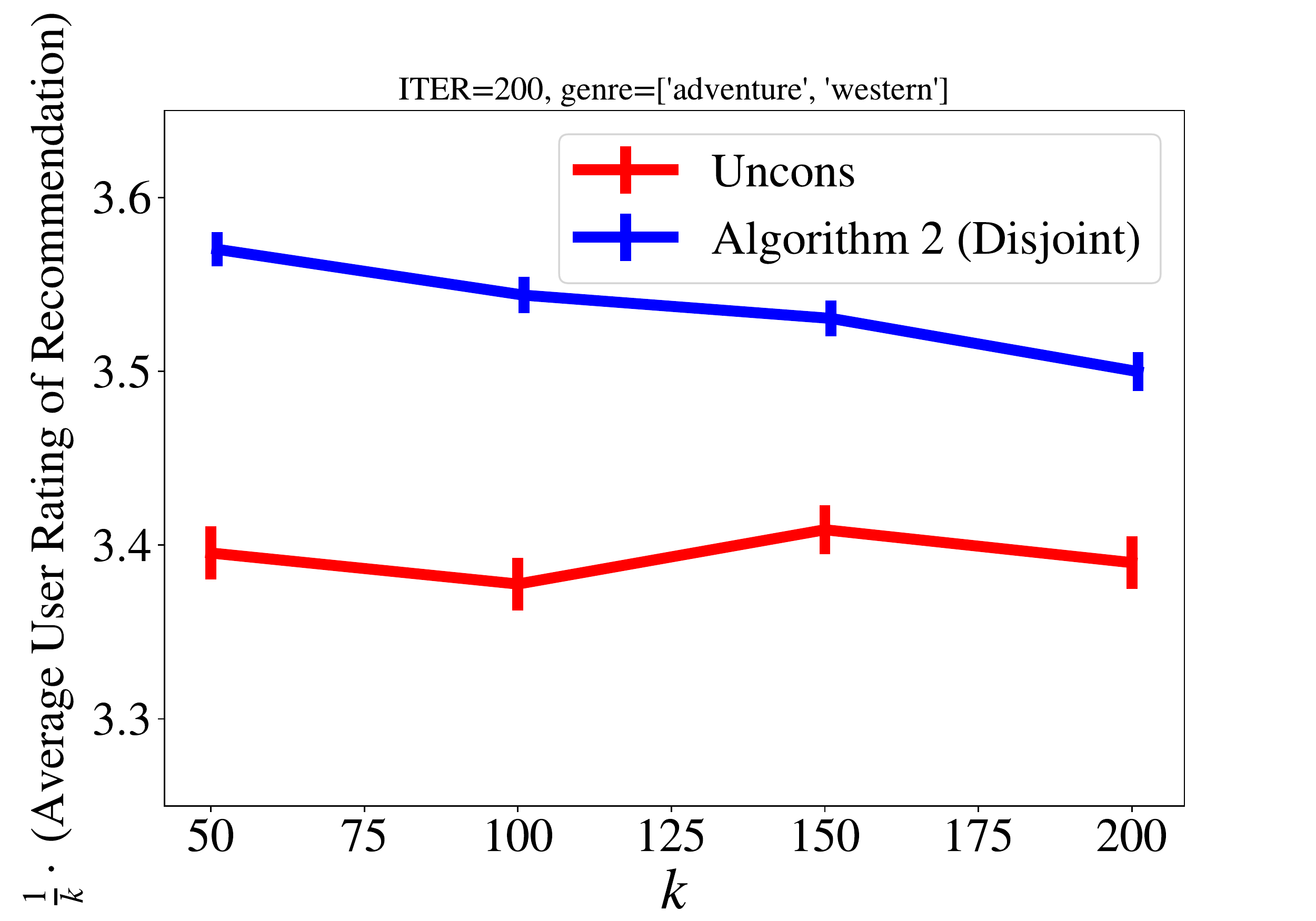}};
        \end{tikzpicture}
        }
        \subfigure[\white{.} $T=\inbrace{\texttt{Crime}, \texttt{War}}$]{
            \begin{tikzpicture}
          \node (image) at (0,-0.17) {\includegraphics[width=\ifconf0.25\linewidth\else0.29\linewidth\fi, trim={0.2cm 0cm 1.5cm 2cm},clip]{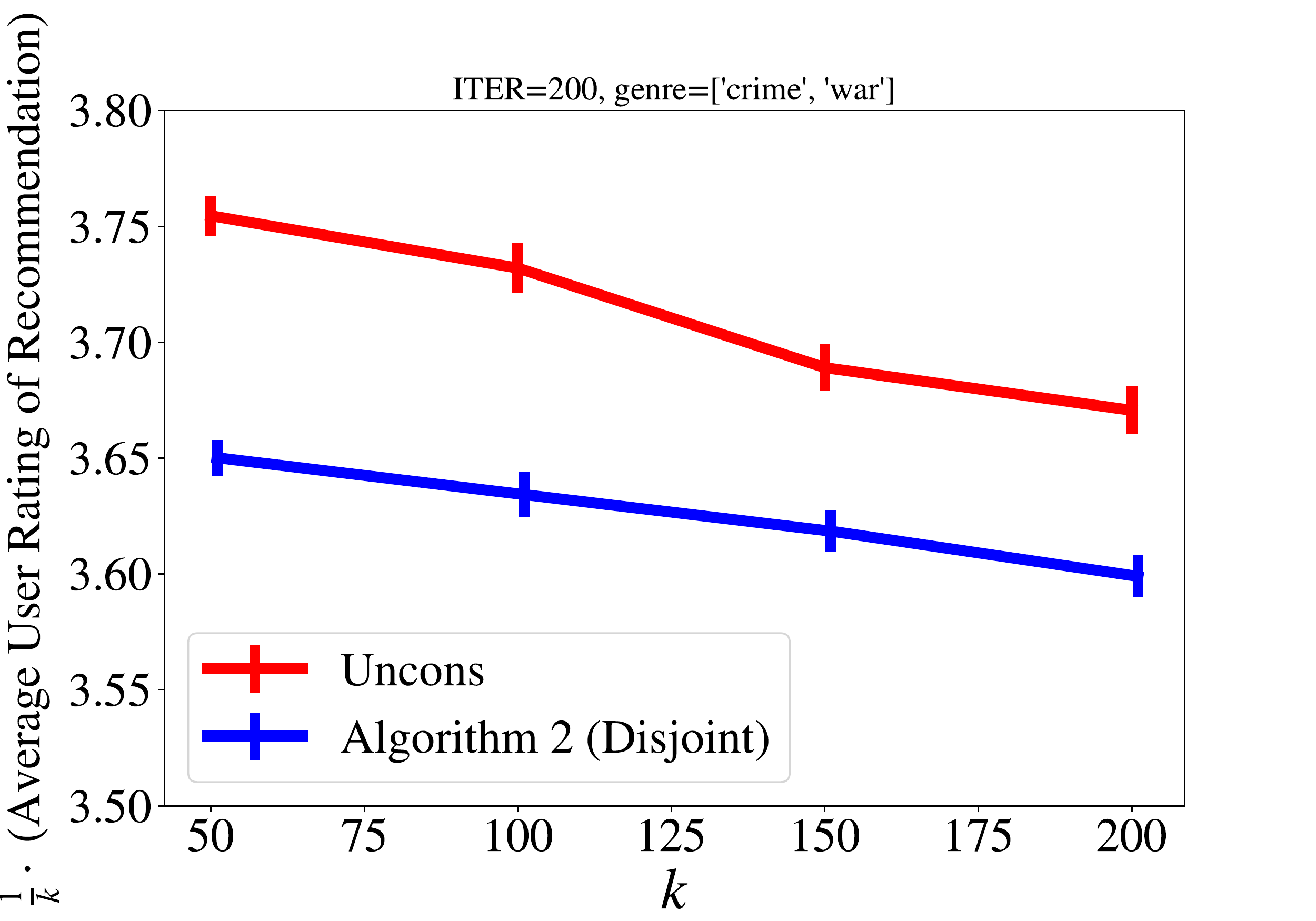}};
        \end{tikzpicture}
        }
        \subfigure[\white{.} $T=\inbrace{\texttt{Crime}, \texttt{Western}}$]{
            \begin{tikzpicture}
          \node (image) at (0,-0.17) {\includegraphics[width=\ifconf0.25\linewidth\else0.29\linewidth\fi, trim={0.2cm 0cm 1.5cm 2cm},clip]{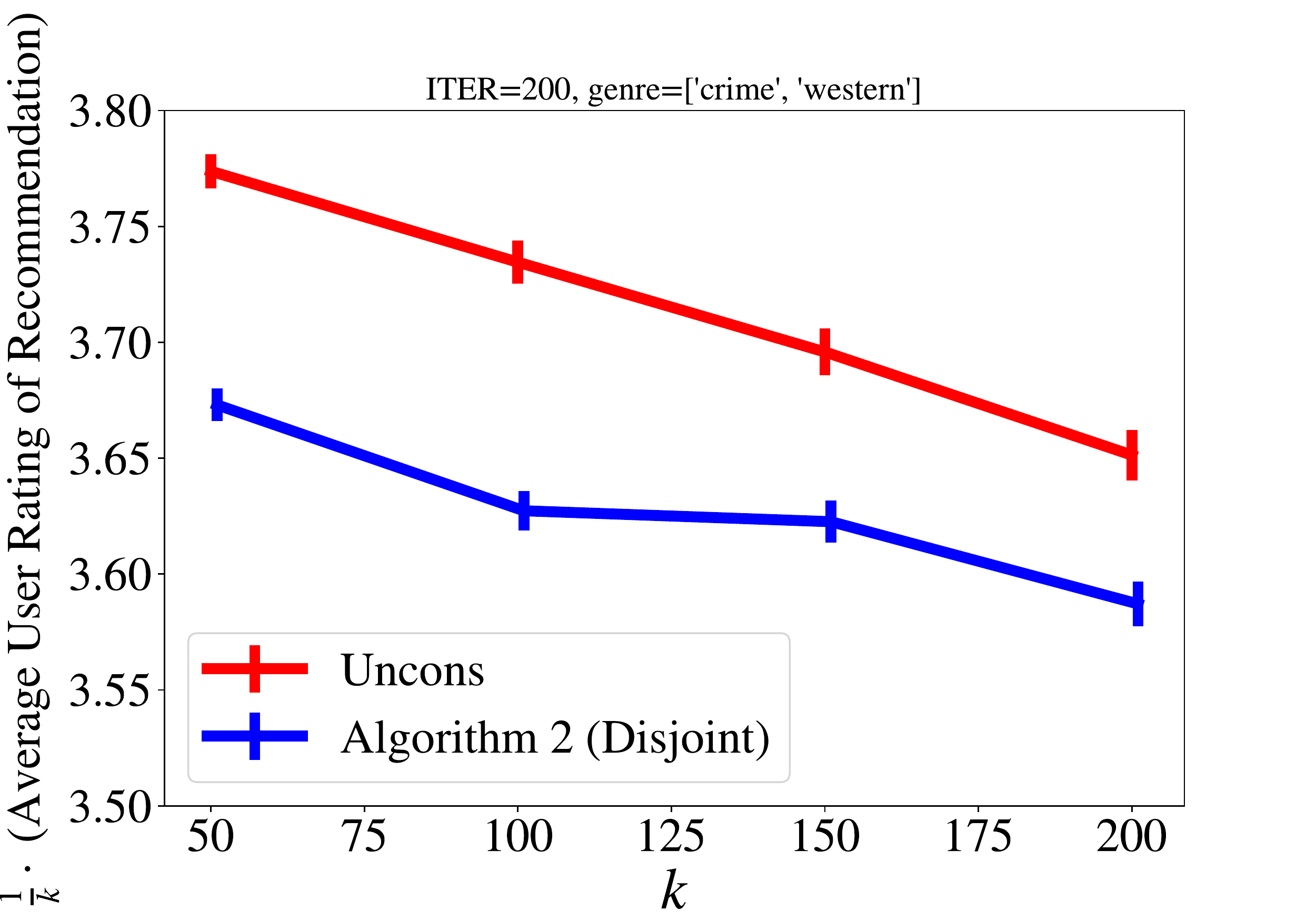}};
        \end{tikzpicture}
        }
        \ifconf\else\vspace{-4mm}\fi
        \par
        \subfigure[\white{.} $T=\inbrace{\texttt{War}, \texttt{Western}}$]{
            \begin{tikzpicture}
          \node (image) at (0,-0.17) {\includegraphics[width=\ifconf0.25\linewidth\else0.29\linewidth\fi, trim={0.2cm 0cm 1.5cm 2cm},clip]{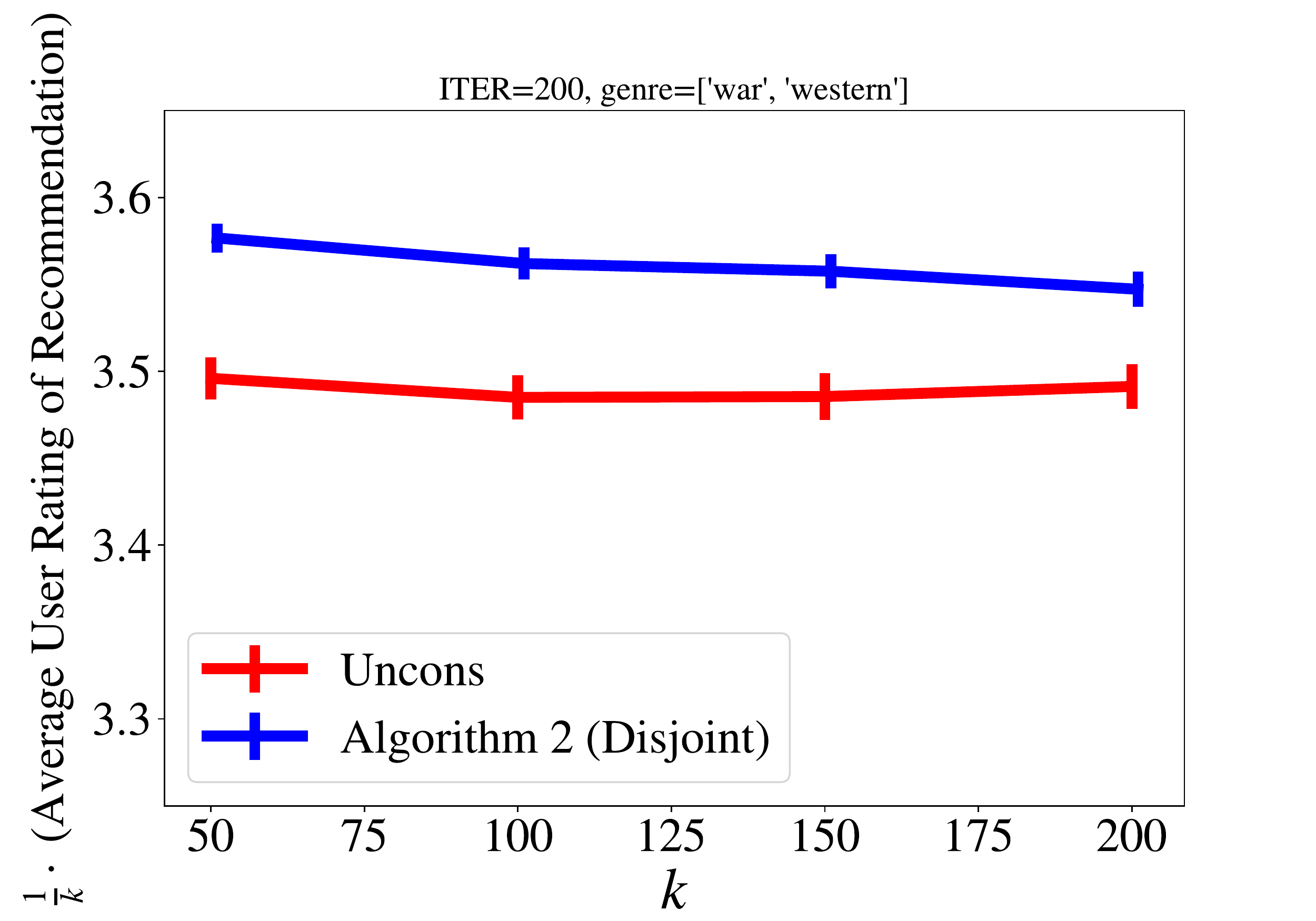}};
        \end{tikzpicture}
        }
        \par
        \caption{
        {\em Simulation with MovieLens data with movie recommendations from two men-stereotypical genres:}
        We observe that the relevance scores in the MovieLens data are disproportionately higher (by up to 3 times) for movies led by male actors compared to movies led by non-male actors in genres stereotypically associated with men.
        In contrast, user ratings for these sets of movies are within 6\% of each other in all genres.
        We chose genres where the ratio of average relevance scores of men-led movies is at least twice that of non-men-led movies, they are:
        $B=\{$\texttt{action}, \texttt{adventure}, \texttt{crime}, \texttt{western}, and \texttt{war}$\}$.
        We use relevance scores to recommend $k\in \inbrace{50, 100, 150, 200}$ movies from different subsets $T$ of $B$.
        This figure presents the results for all subsets $T\subseteq B$ of size 2.
        We observe that in 8 out of 10 subfigures \cref{alg:disj} has a similar or higher normalized latent utility than \uncons{} for all $k$ (by up to 5\%). 
        {\em Note that the range of $y$-axis varies across sugfigures and is either $[3.25, 3.65]$ or $[3.5,3.8]$.}
        }
        \label{fig:movie_lens_pairs_adventure}
    \end{figure}

    \newpage
    \subsubsection{Plots with three genres ($\abs{T}=3$)} 
    \renewcommand{\folder}{./figures/real-world-data/3}
    \begin{figure}[b!]
        \centering \vspace{-5mm}
        \subfigure[$T$=$\inbrace{\texttt{Action}, \texttt{Advent.}, \texttt{Crime}}$\small \hspace{-4mm}]{
            \begin{tikzpicture}
          \node (image) at (0,-0.17) {\includegraphics[width=\ifconf0.25\linewidth\else0.29\linewidth\fi, trim={0.2cm 0cm 1.5cm 2cm},clip]{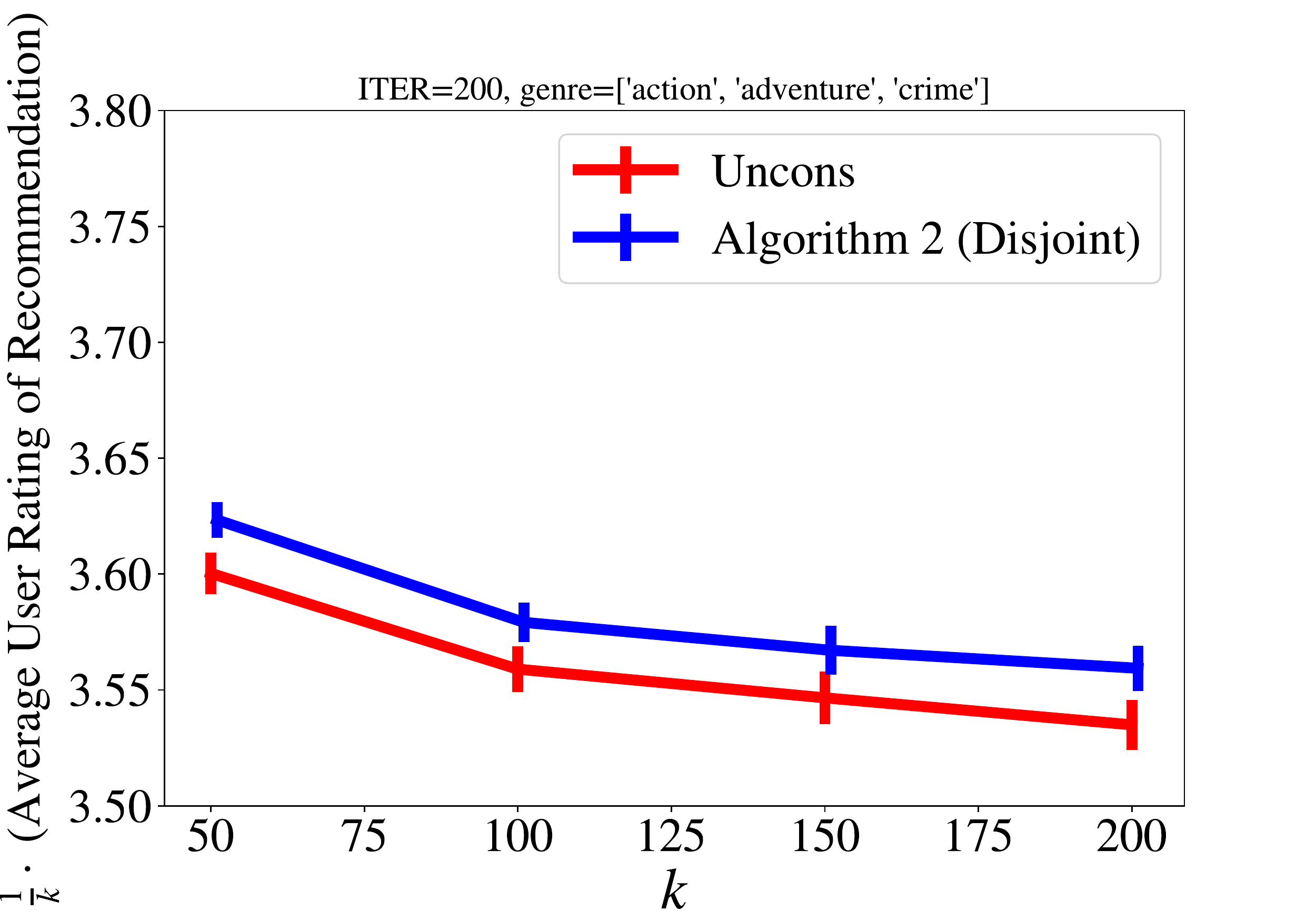}};
        \end{tikzpicture}
        }
        \subfigure[$T$=$\inbrace{\texttt{Action}, \texttt{Adventure}, \texttt{War}}$\small]{
            \begin{tikzpicture}
          \node (image) at (0,-0.17) {\includegraphics[width=\ifconf0.25\linewidth\else0.29\linewidth\fi, trim={0.2cm 0cm 1.5cm 2cm},clip]{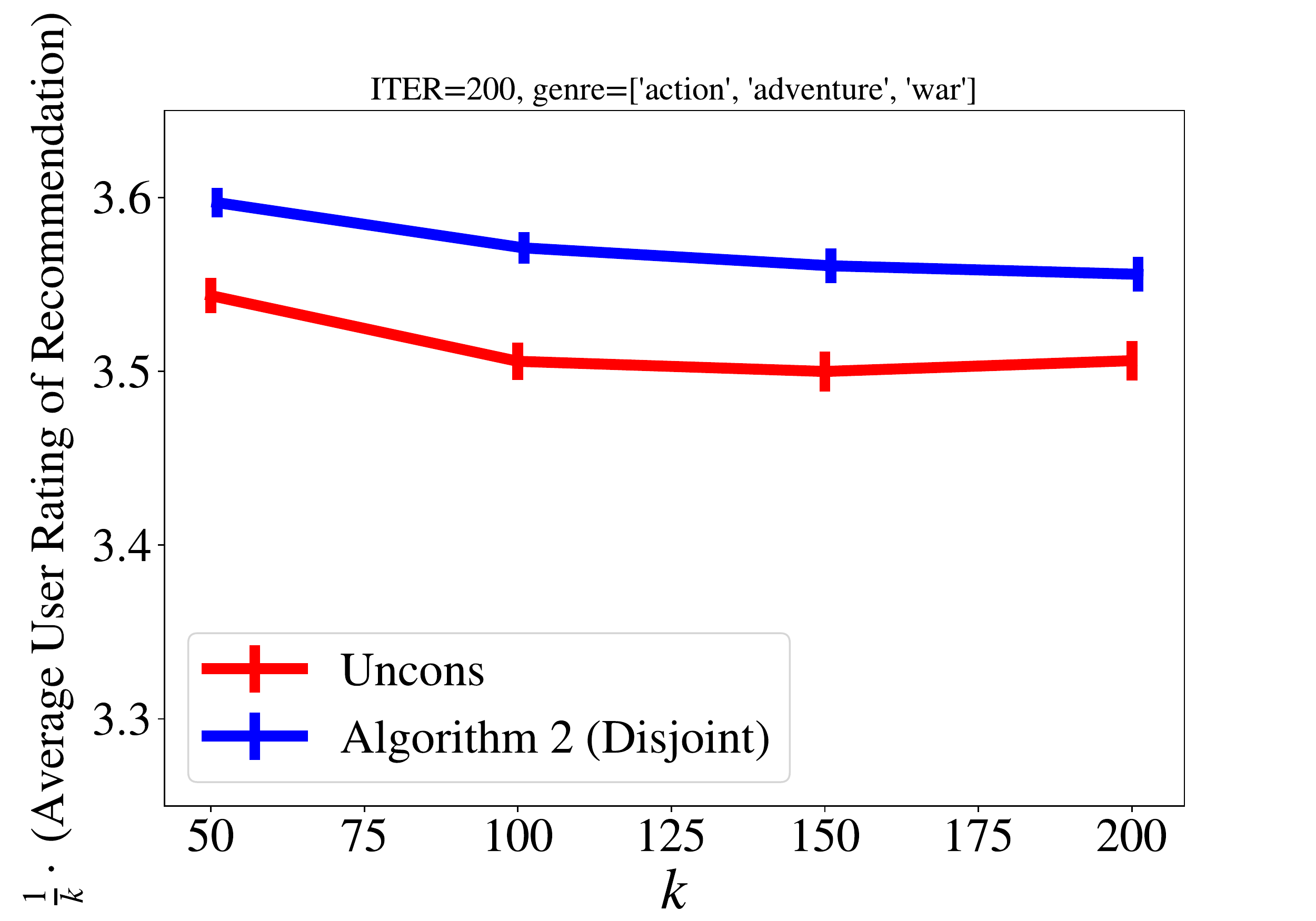}};
        \end{tikzpicture}
        }
        \subfigure[$T$=$\inbrace{\texttt{Action}, \texttt{Advent.}, \texttt{Western}}$\small\hspace{-4mm} ]{
            \begin{tikzpicture}
          \node (image) at (0,-0.17) {\includegraphics[width=\ifconf0.25\linewidth\else0.29\linewidth\fi, trim={0.2cm 0cm 1.5cm 2cm},clip]{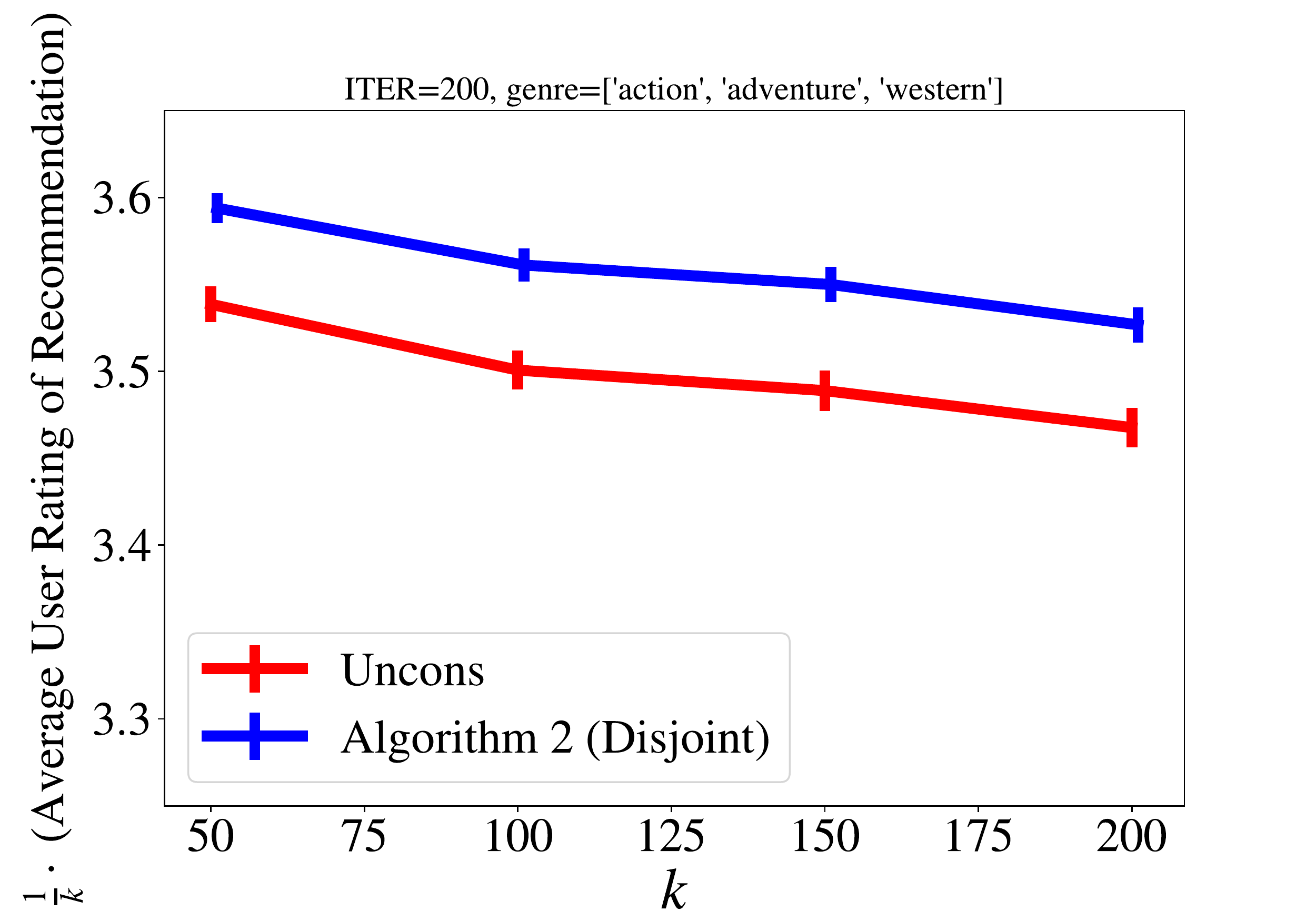}};
        \end{tikzpicture}
        }
        \par
        \vspace{-3mm}
        \subfigure[\small $T=\inbrace{\texttt{Action}, \texttt{Crime}, \texttt{War}}$]{
            \begin{tikzpicture}
          \node (image) at (0,-0.17) {\includegraphics[width=\ifconf0.25\linewidth\else0.29\linewidth\fi, trim={0.2cm 0cm 1.5cm 2cm},clip]{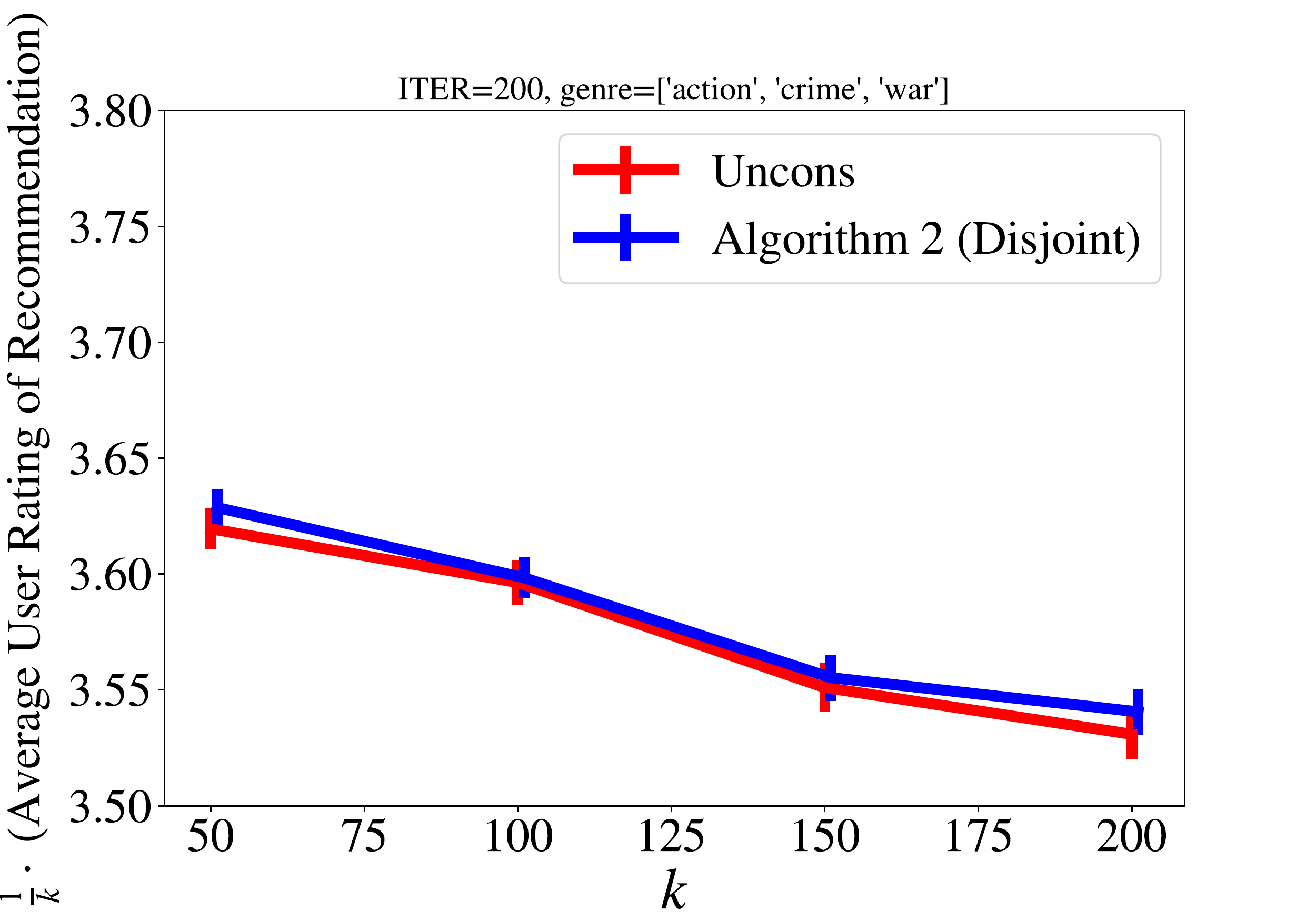}};
        \end{tikzpicture}
        }
        \subfigure[$T$=$\inbrace{\texttt{Action}, \texttt{Crime}, \texttt{West.}}$\small \hspace{-4mm}]{
            \begin{tikzpicture}
          \node (image) at (0,-0.17) {\includegraphics[width=\ifconf0.25\linewidth\else0.29\linewidth\fi, trim={0.2cm 0cm 1.5cm 2cm},clip]{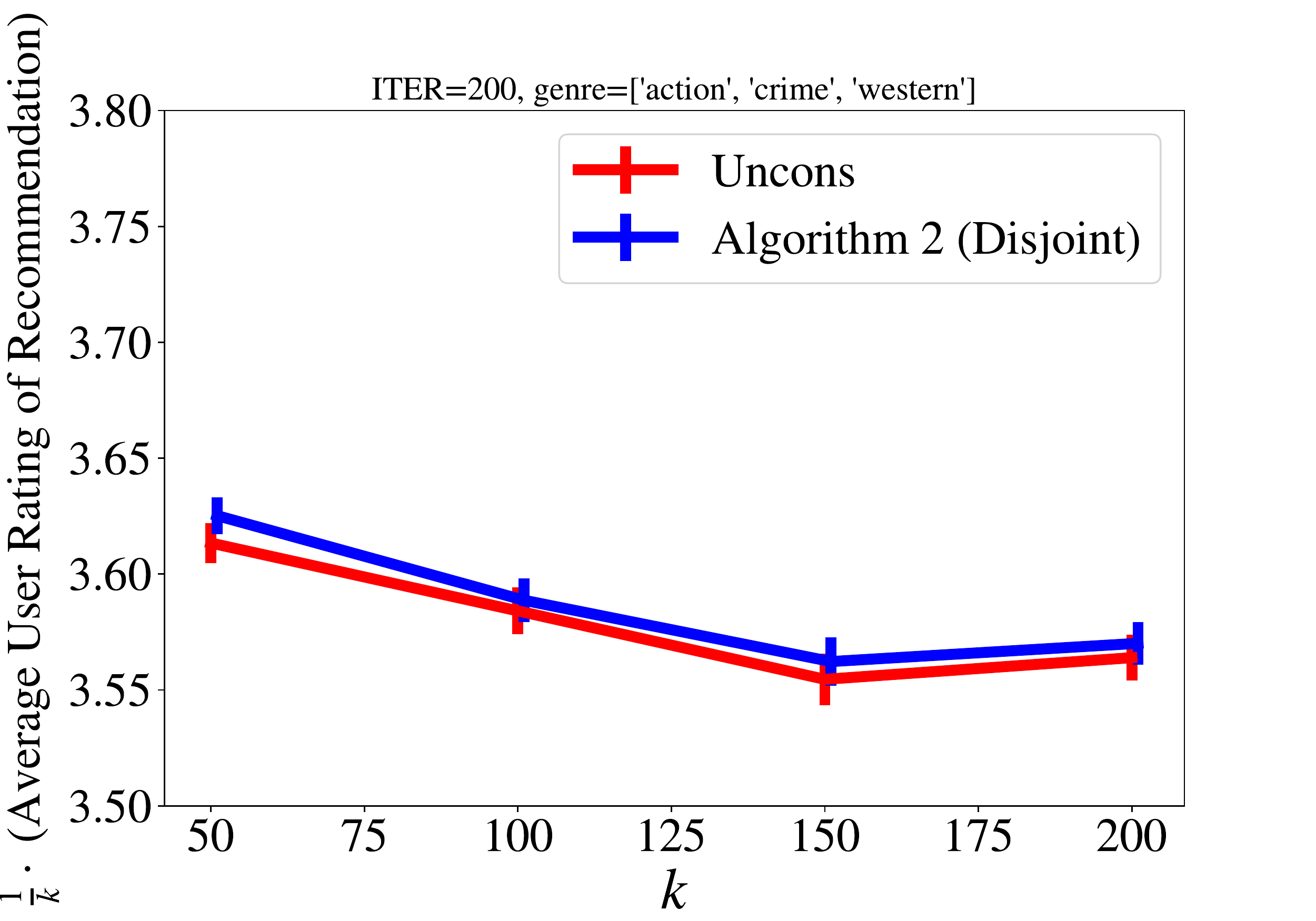}};
        \end{tikzpicture}
        }
        \subfigure[\small $T=\inbrace{\texttt{Adventure}, \texttt{Crime}, \texttt{War}}$  ]{
            \begin{tikzpicture}
          \node (image) at (0,-0.17) {\includegraphics[width=\ifconf0.25\linewidth\else0.29\linewidth\fi, trim={0.2cm 0cm 1.5cm 2cm},clip]{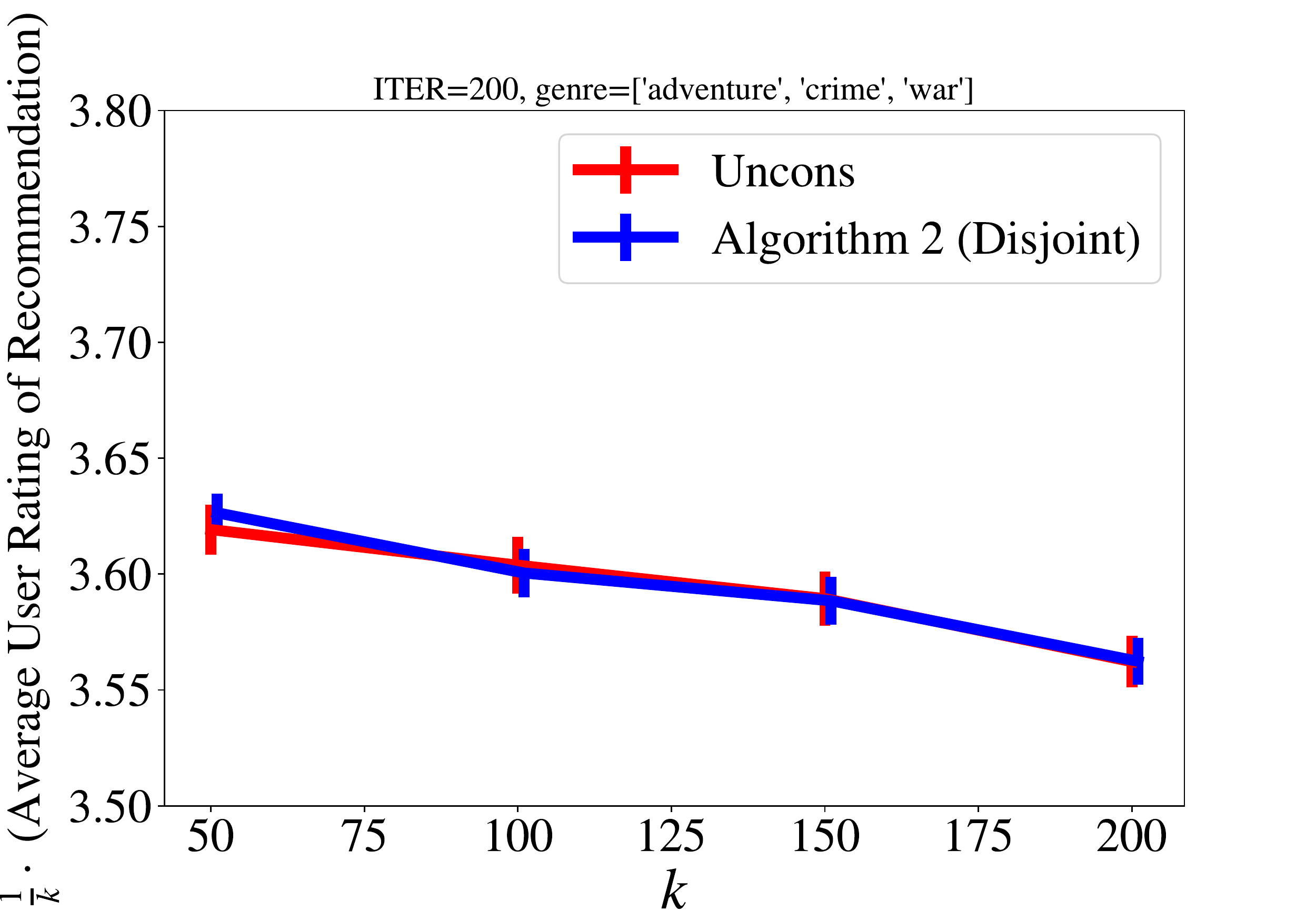}};
        \end{tikzpicture}
        }
        \par
        \vspace{-3mm}
        \subfigure[$T$=$\inbrace{\texttt{Advent.}, \texttt{Crime}, \texttt{West.}}$\small ]{
            \begin{tikzpicture}
          \node (image) at (0,-0.17) {\includegraphics[width=\ifconf0.25\linewidth\else0.29\linewidth\fi, trim={0.2cm 0cm 1.5cm 2cm},clip]{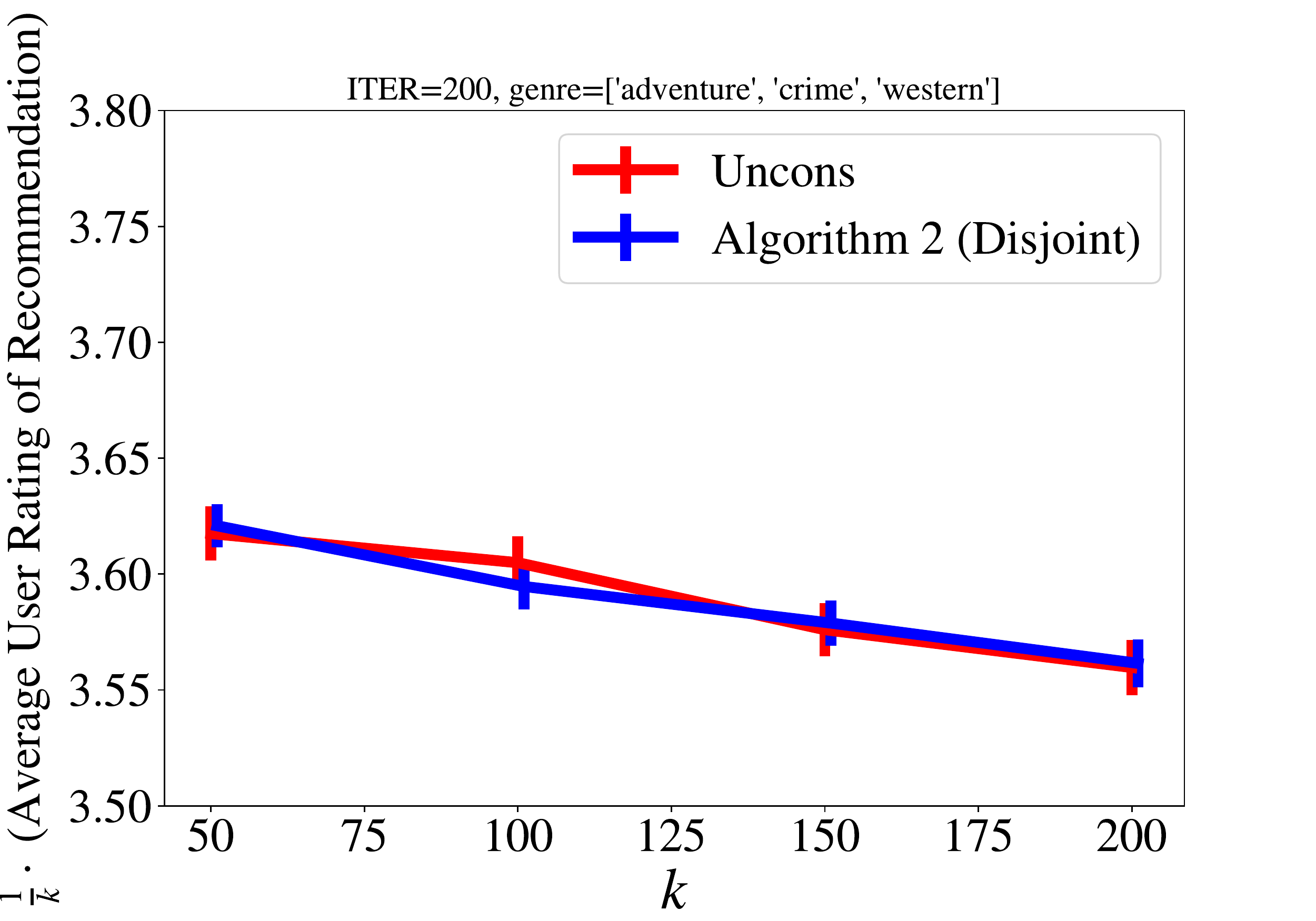}};
        \end{tikzpicture}
        }
        \subfigure[$T$=$\inbrace{\texttt{Adventure}, \texttt{War}, \texttt{West.}}$]{
            \begin{tikzpicture}
          \node (image) at (0,-0.17) {\includegraphics[width=\ifconf0.25\linewidth\else0.29\linewidth\fi, trim={0.2cm 0cm 1.5cm 2cm},clip]{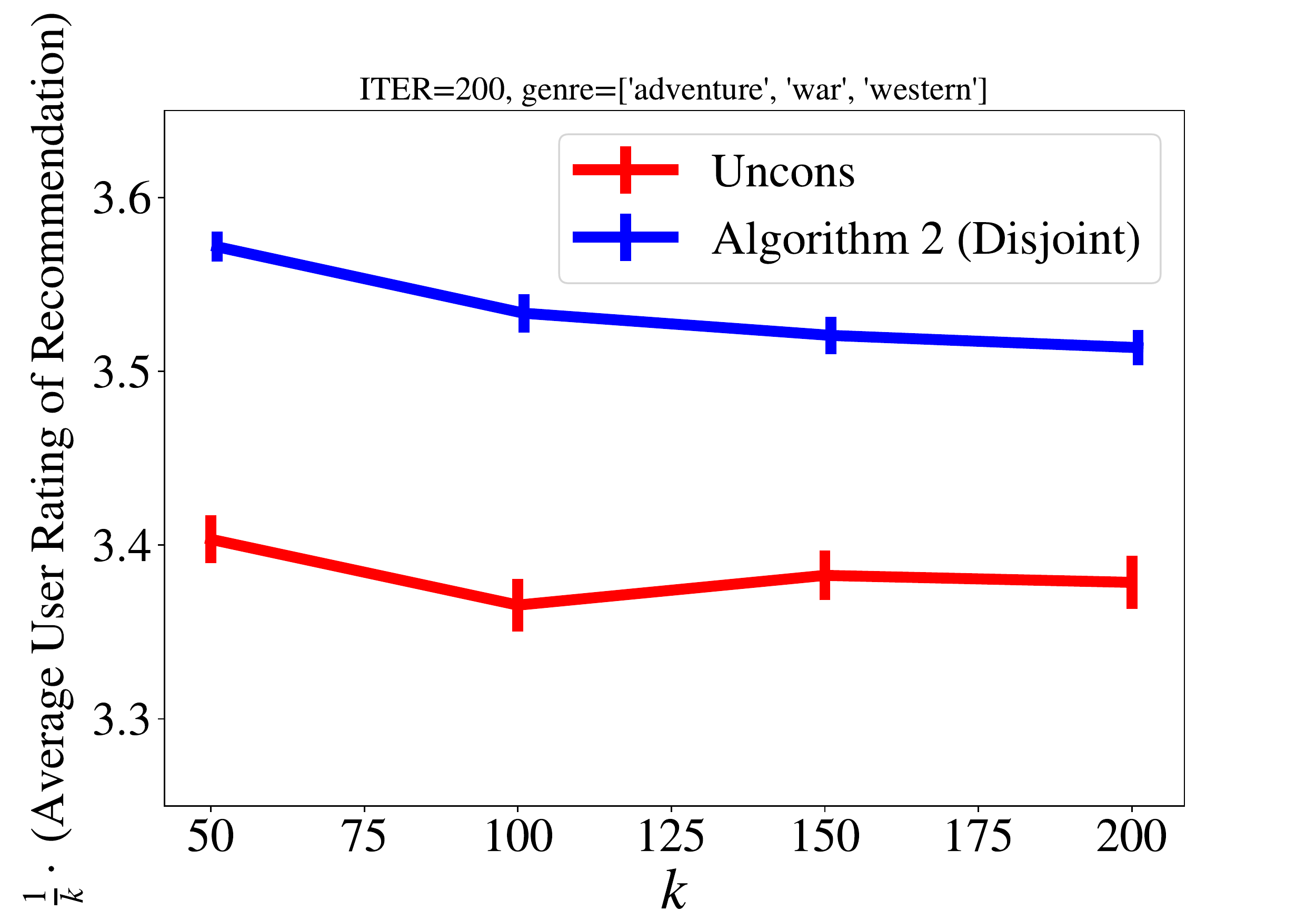}};
        \end{tikzpicture}
        }
        \subfigure[\small $T=\inbrace{\texttt{Action}, \texttt{War}, \texttt{Western}}$]{
            \begin{tikzpicture}
          \node (image) at (0,-0.17) {\includegraphics[width=\ifconf0.25\linewidth\else0.29\linewidth\fi, trim={0.2cm 0cm 1.5cm 2cm},clip]{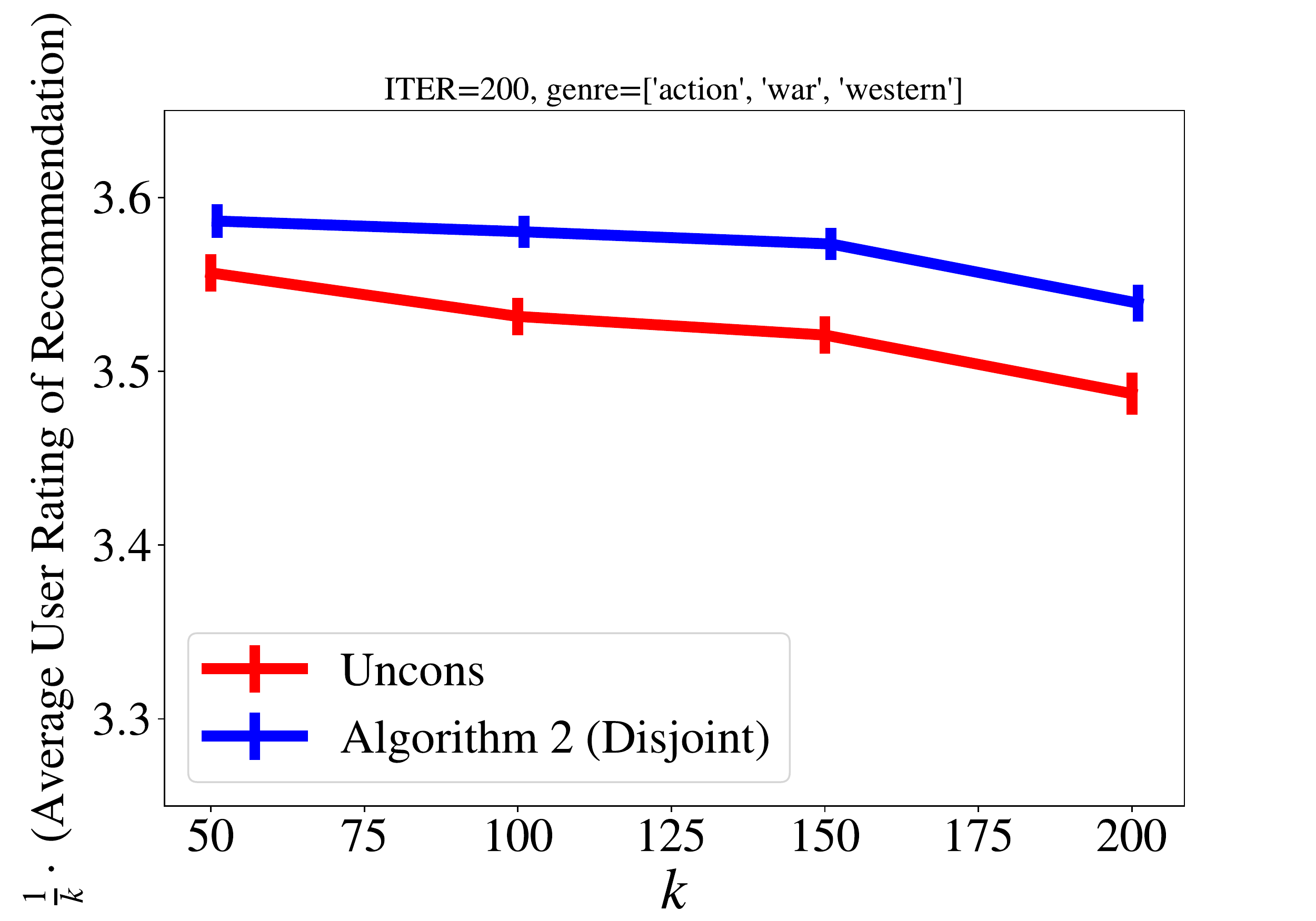}};
        \end{tikzpicture}
        }
        \par
        \vspace{-3mm}
        \subfigure[\small $T=\inbrace{\texttt{Crime}, \texttt{War}, \texttt{Western}}$]{
            \begin{tikzpicture}
          \node (image) at (0,-0.17) {\includegraphics[width=\ifconf0.25\linewidth\else0.29\linewidth\fi, trim={0.2cm 0cm 1.5cm 2cm},clip]{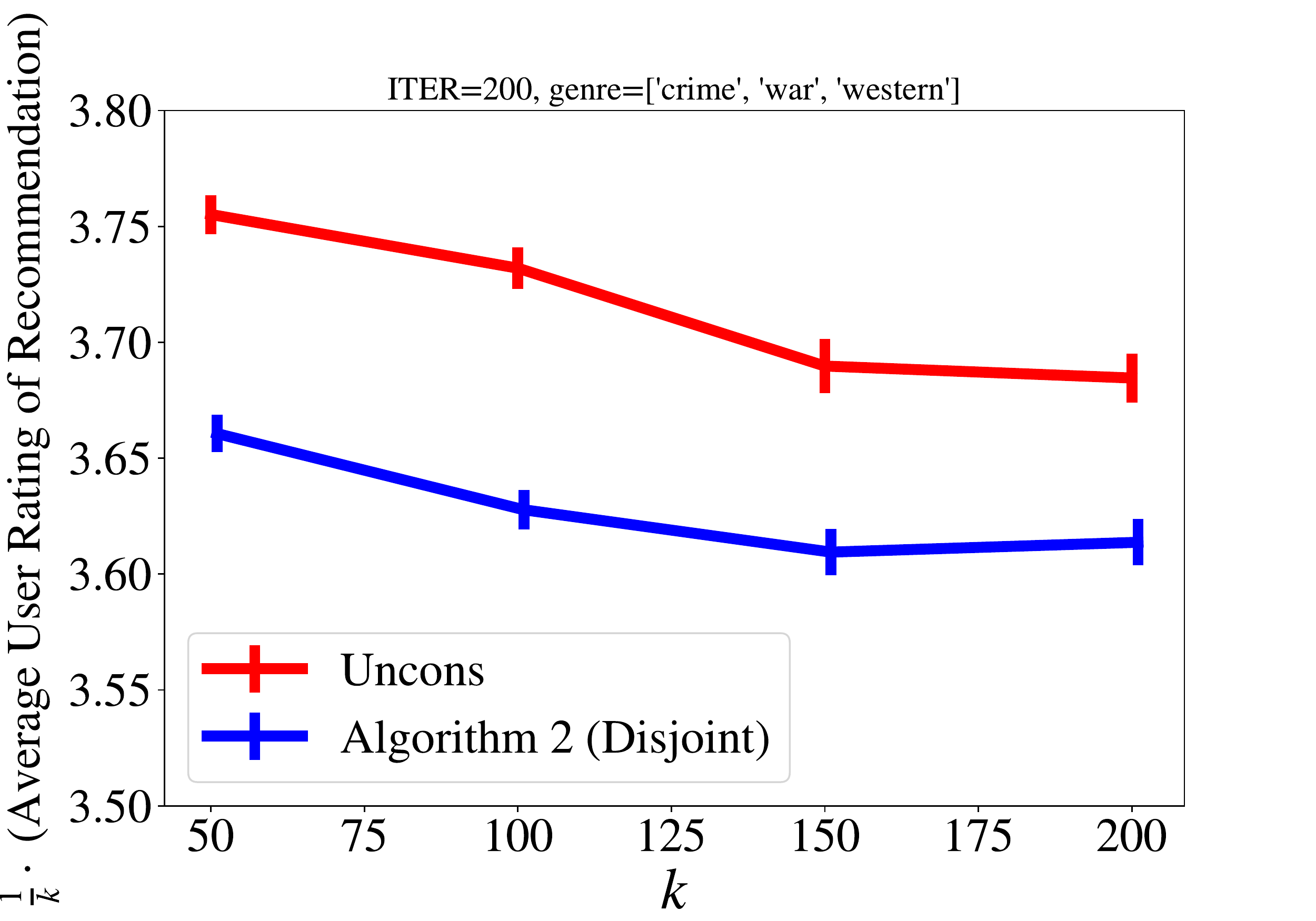}};
        \end{tikzpicture}
        }
        \par
        \caption{
        {\em Simulation with MovieLens data with movie recommendations from three men-stereotypical genres:}
        We observe that the relevance scores in the MovieLens data are disproportionately higher (by up to 3 times) for movies led by male actors compared to movies led by non-male actors in genres stereotypically associated with men.
        In contrast, user ratings for these sets of movies are within 6\% of each other in all genres.
        We chose genres where the ratio of average relevance scores of men-led movies is at least twice that of non-men-led movies, they are:
        $B=\{$\texttt{action}, \texttt{adventure}, \texttt{crime}, \texttt{western}, and \texttt{war}$\}$.
        We use relevance scores to recommend $k\in \inbrace{50, 100, 150, 200}$ movies from different subsets $T$ of $B$.
        This figure presents the results for all subsets $T\subseteq B$ of size 3.
        We observe that in 9 out of 10 subfigures \cref{alg:disj} has a similar or higher normalized latent utility than \uncons{} for all $k$ (by up to 5\%). 
        {\em Note that the range of $y$-axis varies across sugfigures and is either $[3.25, 3.65]$ or $[3.5,3.8]$.}
        }
        \vspace*{-5mm}
        \label{fig:movie_lens_pairs_crime}
    \end{figure}

    \newpage
    \subsubsection{Plots with four genres ($\abs{T}=4$)} \white{...A}

    \renewcommand{\folder}{./figures/real-world-data/4} 
    \begin{figure}[h!]
        \centering 
        \subfigure[$T=\inbrace{\texttt{Adventure}, \texttt{Crime}, \texttt{War},\texttt{Western}}$\small ]{
            \begin{tikzpicture}
          \node (image) at (0,-0.17) {\includegraphics[width=\ifconf0.25\linewidth\else0.29\linewidth\fi, trim={0.2cm 0cm 1.5cm 2cm},clip]{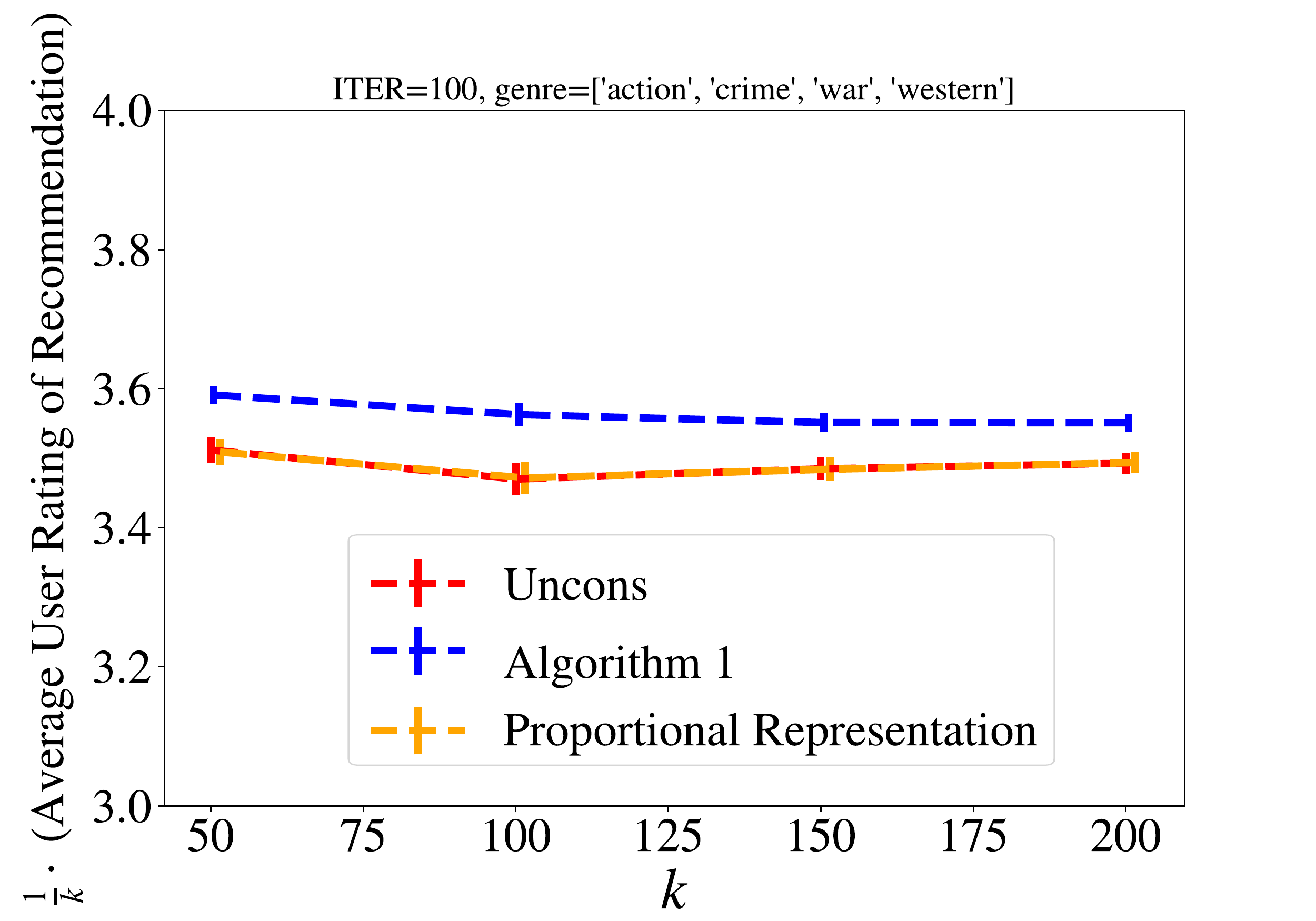}};
          \node[rotate=0] at (-0.1,1.4) {\white{......................................................................}};
        \end{tikzpicture}
        }
        \subfigure[ $T=\inbrace{\texttt{Action}, \texttt{Crime}, \texttt{War}, \texttt{Western}}$\small ]{
            \begin{tikzpicture}
          \node (image) at (0,-0.17) {\includegraphics[width=\ifconf0.25\linewidth\else0.29\linewidth\fi, trim={0.2cm 0cm 1.5cm 2cm},clip]{\folder/File3.pdf}};
                    \node[rotate=0] at (-0.1,1.4) {\white{......................................................................}};
        \end{tikzpicture}
        }
        \par
        \subfigure[ $T=\inbrace{\texttt{Action}, \texttt{Adventure}, \texttt{War}, \texttt{Western}}$\small ]{
            \begin{tikzpicture}
          \node (image) at (0,-0.17) {\includegraphics[width=\ifconf0.25\linewidth\else0.29\linewidth\fi, trim={0.2cm 0cm 1.5cm 2cm},clip]{\folder/File1.pdf}};
            \node[rotate=0] at (-0.1,1.4) {\white{......................................................................}};
        \end{tikzpicture}
        }
        \subfigure[ $T=\inbrace{\texttt{Action}, \texttt{Adventure}, \texttt{Crime}, \texttt{Western}}$\small ]{
            \begin{tikzpicture}
          \node (image) at (0,-0.17) {\includegraphics[width=\ifconf0.25\linewidth\else0.29\linewidth\fi, trim={0.2cm 0cm 1.5cm 2cm},clip]{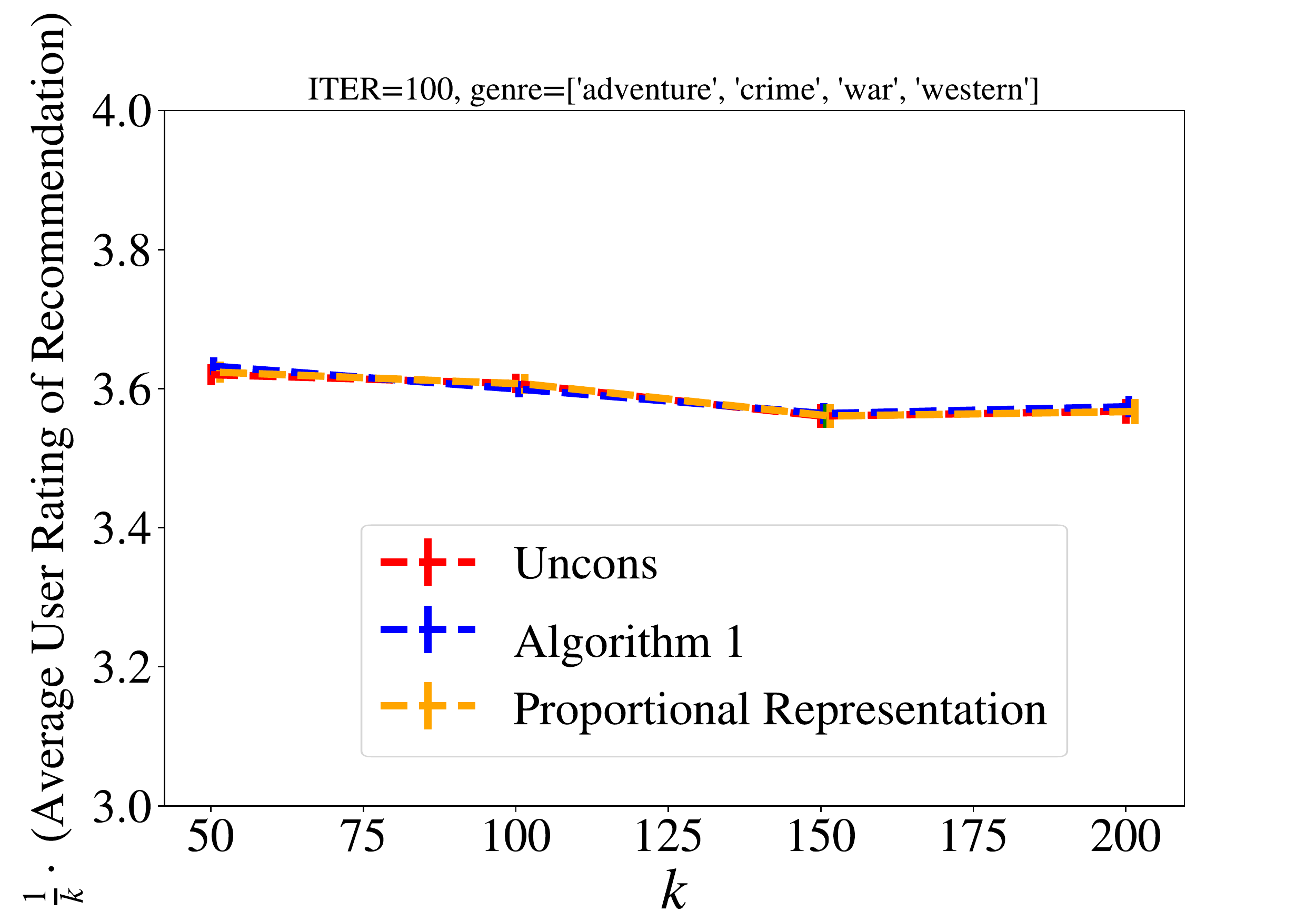}};
          \node[rotate=0] at (-0.1,1.4) {\white{......................................................................}};
        \end{tikzpicture}
        }
        \subfigure[ $T=\inbrace{\texttt{Action}, \texttt{Adventure}, \texttt{Crime}, \texttt{War}}$\small ]{
            \begin{tikzpicture}
          \node (image) at (0,-0.17) {\includegraphics[width=\ifconf0.25\linewidth\else0.29\linewidth\fi, trim={0.2cm 0cm 1.5cm 2cm},clip]{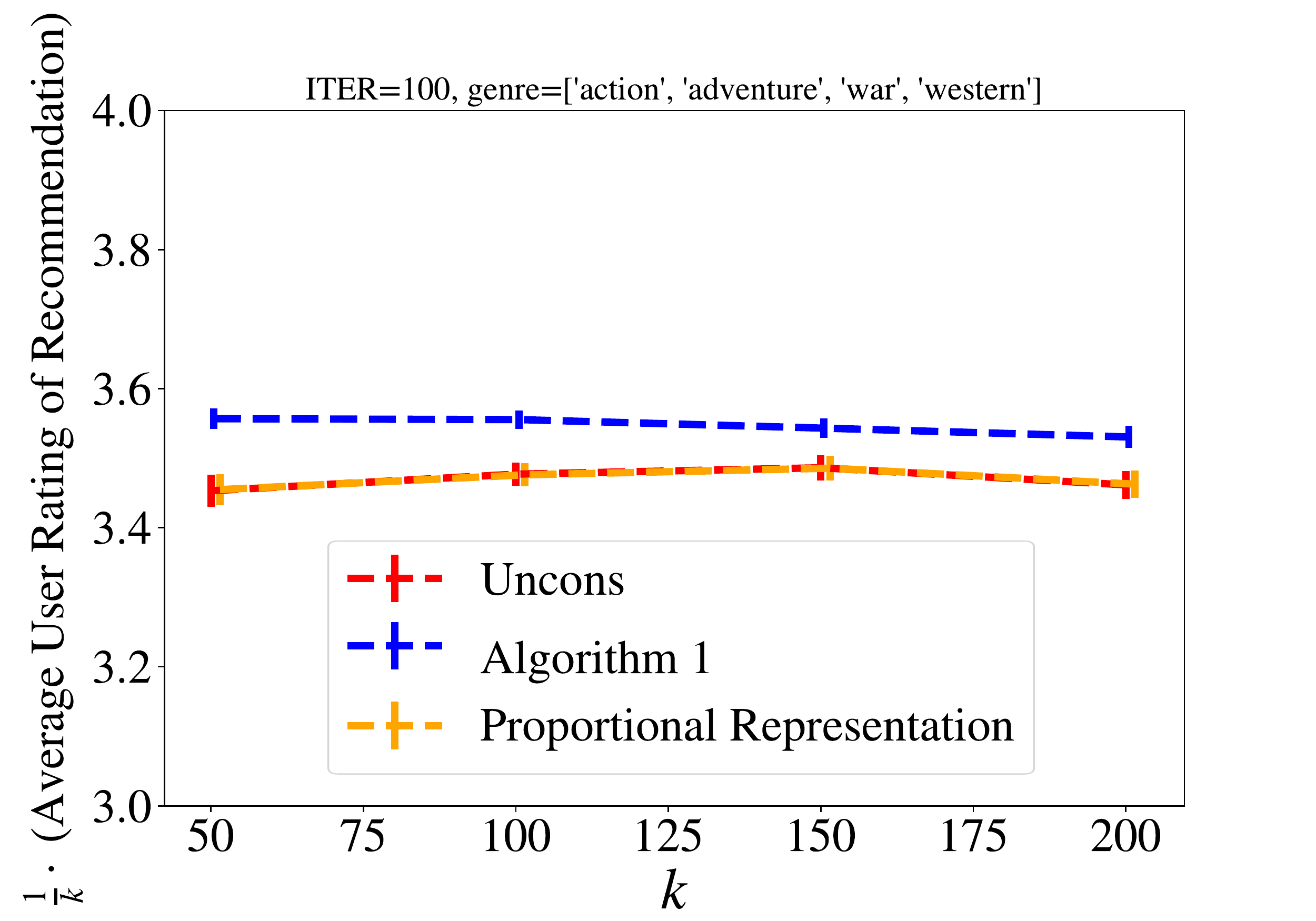}};
          \node[rotate=0] at (-0.1,1.4) {\white{......................................................................}};
        \end{tikzpicture}
        }
        \caption{
        {\em Simulation with MovieLens data with movie recommendations from four men-stereotypical genres:}
        We observe that the relevance scores in the MovieLens data are disproportionately higher (by up to 3 times) for movies led by male actors compared to movies led by non-male actors in genres stereotypically associated with men.
        In contrast, user ratings for these sets of movies are within 6\% of each other in all genres.
        We chose genres where the ratio of average relevance scores of men-led movies is at least twice that of non-men-led movies, they are:
        $B=\{$\texttt{action}, \texttt{adventure}, \texttt{crime}, \texttt{western}, and \texttt{war}$\}$.
        We use relevance scores to recommend $k\in \inbrace{50, 100, 150, 200}$ movies from different subsets $T$ of $B$.
        This figure presents the results for all subsets $T\subseteq B$ of size 4.
        We observe that in 5 out of 5 subfigures \cref{alg:disj} has a higher normalized latent utility than \uncons{} for all $k$ (by up to 5\%).}
        \label{fig:real_world_3}
    \end{figure} 
    \newpage   
 
    \subsubsection{Plots with five genres ($\abs{T}=5$)} \white{...A}
 
    \renewcommand{\folder}{./figures/real-world-data/5}
    \begin{figure}[h!]
        \centering
        \ifconf\vspace{-4mm}\fi
            \begin{tikzpicture}
          \node (image) at (0,-0.17) {\includegraphics[width=\ifconf0.25\linewidth\else0.29\linewidth\fi, trim={0.2cm 0cm 1.5cm  2cm},clip]{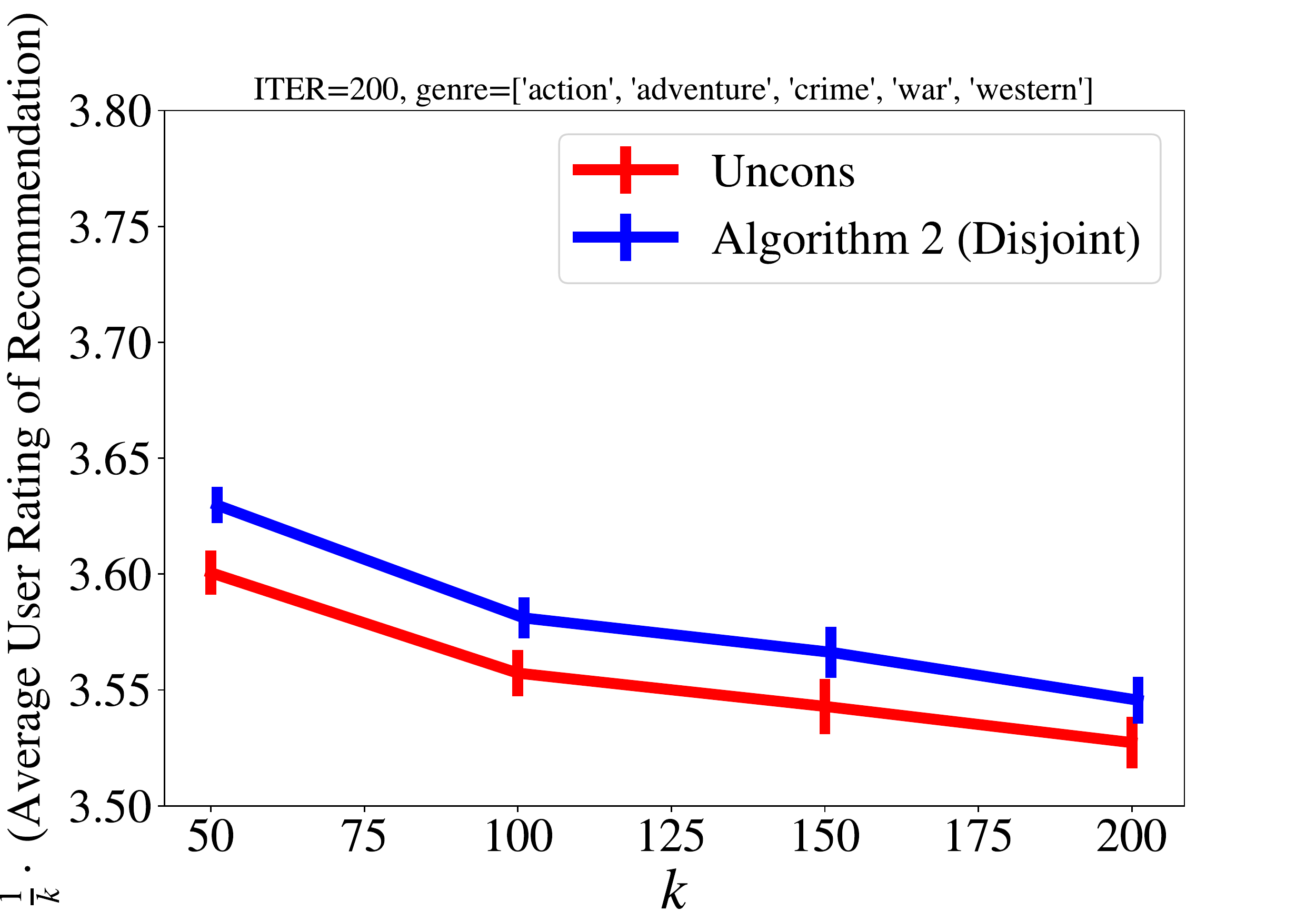}};
                    \node[rotate=0] at (-0.1,1.7) {\white{......................................................................}};
        \end{tikzpicture}
        \caption{
        {\em Simulation with MovieLens data with movie recommendations from five men-stereotypical genres:}
        We observe that the relevance scores in the MovieLens data are disproportionately higher (by up to 3 times) for movies led by male actors compared to movies led by non-male actors in genres stereotypically associated with men.
        In contrast, user ratings for these sets of movies are within 6\% of each other in all genres.
        We chose genres where the ratio of average relevance scores of men-led movies is at least twice that of non-men-led movies, they are:
        $B=\{$\texttt{action}, \texttt{adventure}, \texttt{crime}, \texttt{western}, and \texttt{war}$\}$.
        We use relevance scores to recommend $k\negsp{}\in\negsp{} \inbrace{50,\negsp{} 100,\negsp{} 150,\negsp{} 200}$ movies from different subsets $T$ of $B$.
        This figure presents the results for all subsets $T\subseteq B$ of size 5.
        We observe that the figure \cref{alg:disj} has a similar normalized latent utility to \uncons{} for all $k$.}
        \label{fig:real_world_5}
    \end{figure}

\end{document}